\documentclass{article}

\usepackage{microtype}
\usepackage{graphicx}
\usepackage{subcaption}
\usepackage{booktabs} 

\usepackage{hyperref}

\usepackage{amsmath}
\usepackage{amsthm}
\usepackage{amsfonts}
\usepackage{amssymb}
\usepackage{natbib}
\usepackage{nicefrac}
\usepackage{xcolor, pict2e}
\usepackage{tikz}
\usepackage[abs]{overpic}
\usepackage{chngcntr}
\usepackage{apptools}
\AtAppendix{\counterwithin{proposition}{section}}

\newtheorem{proposition}{Proposition}
\DeclareMathOperator*{\argmin}{\arg\!\min}
\DeclareMathOperator*{\argmax}{\arg\!\max}
\newcommand*{\R}{\mathbb{R}}
\newcommand*{\E}{\mathbb{E}}
\newcommand*{\Z}{\mathcal{Z}}
\newcommand*{\X}{\mathcal{X}}
\newcommand*{\bm}[1]{\mathbf{#1}}
\newcommand*{\N}{\mathcal{N}}
\newcommand*{\HH}{\mathcal{H}}
\newcommand*{\I}{\mathbb{I}}
\newcommand*{\C}{\mathcal{C}}
\newcommand*{\M}{\mathcal{M}}
\newcommand{\inner}[2]{\langle#1,#2\rangle}
\newcommand{\tangent}[2]{\mathcal{T}_{#1}{#2}}
\newcommand{\T}{\intercal}
\newcommand{\logmap}[2]{\text{Log}_{#1}({#2})}
\newcommand{\expmap}[2]{\text{Exp}_{#1}({#2})}
\newcommand{\vectorize}[1]{\text{vec}[#1]}
\newcommand{\parder}[2]{\frac{\partial #1}{\partial #2}}

\usepackage[accepted]{icml2021}

\icmltitlerunning{A prior-based approximate latent Riemannian metric}

\begin{document}

\twocolumn[
\icmltitle{A prior-based approximate latent Riemannian metric}



\icmlsetsymbol{equal}{*}

\begin{icmlauthorlist}
\icmlauthor{Georgios Arvanitidis}{mpi}
\icmlauthor{Bogdan Georgiev}{fraun}
\icmlauthor{Bernhard Sch{\"o}lkopf}{mpi}
\end{icmlauthorlist}

\icmlaffiliation{mpi}{Max Planck Institute for Intelligent Systems, T{\"u}bingen, Germany}
\icmlaffiliation{fraun}{Fraunhofer IAIS, ML2R, Sankt Augustin, Germany}

\icmlcorrespondingauthor{Georgios Arvanitidis}{gear@tuebingen.mpg.de}

\icmlkeywords{Riemannian manifolds, latent space, generative models, prior learning}

\vskip 0.3in
]



\printAffiliationsAndNotice{}  

\begin{abstract}
    Stochastic generative models enable us to capture the geometric structure of a data manifold lying in a high dimensional space through a Riemannian metric in the latent space. However, its practical use is rather limited mainly due to inevitable complexity. In this work we propose a surrogate conformal Riemannian metric in the latent space of a generative model that is simple, efficient and robust. This metric is based on a learnable prior that we propose to learn using a basic energy-based model. We theoretically analyze the behavior of the proposed metric and show that it is sensible to use in practice. We demonstrate experimentally the efficiency and robustness, as well as the behavior of the new approximate metric. Also, we show the applicability of the proposed methodology for data analysis in the life sciences.

\end{abstract}

\section{Introduction}
\label{sec:intro}

The \emph{manifold hypothesis} states that in a high dimensional space the data has a low dimensional nonlinear geometric structure. One way to compute distances that respect this structure is by using discrete shortest paths on neighborhood graphs \citep{tenenbaum:science:2000}. Nevertheless, this strategy does not allow to perform continuous analysis, as for example Riemannian statistics \citep{pennec:jmiv:2006}. Hence, methods based on latent variable models have been developed that enables us to compute continuous shortest paths.

Generative models provide a way to estimate the probability density of the given data lying in an ambient space $\X$. While most of the models utilize a latent space $\Z$, the Variational Auto-Encoder (VAE) also learns a low dimensional representation of the data \citep{rezende:icml:2014, kingma:iclr:2014}. Unfortunately, using straight lines to compute distances in the latent space is misleading, and in addition, is not identifiable \citep{hauberg:only:2018}.


\begin{figure}[t]
	\centering
	\includegraphics[width=0.8\columnwidth]{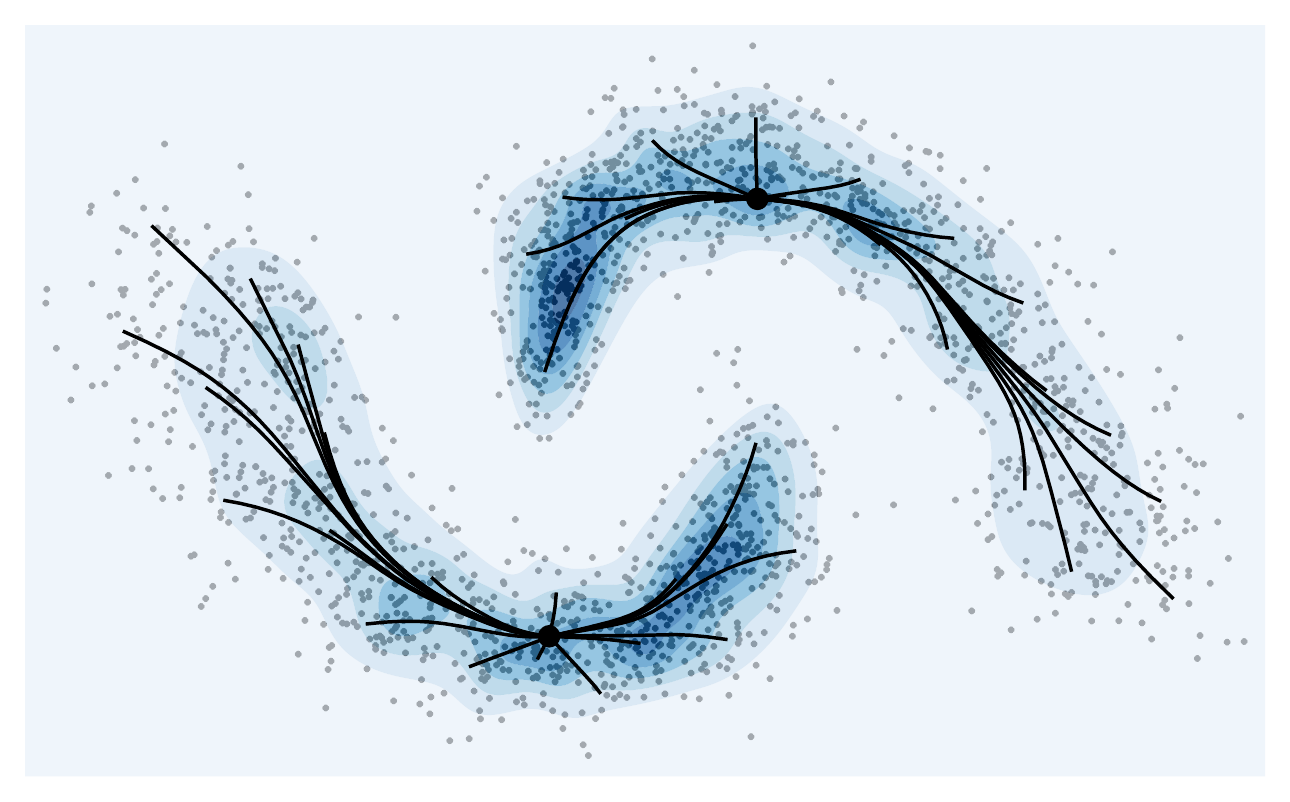}
	\caption{We propose a conformal Riemannian metric that is related to the learnable prior of a latent variable model. Intuitively, shortest paths ({\protect\raisebox{2pt}{\tikz{\draw[black,solid,line width=1pt](0,0) -- (2mm,0);}}}) prefer regions in latent space with high density.}
\end{figure}

One solution is to compute shortest paths in $\Z$ using a Riemannian metric that is induced by the generator \citep{tosi:uai:2014, arvanitidis:iclr:2018}. This gives a natural and identifiable distance measure, since it is actually computed directly on the data manifold in $\X$. However, we need to estimate meaningfully the generator's uncertainty to properly capture the geometry in $\Z$. While this approach allows us to compute continuous and principled distances respecting the data manifold, it is not particularly efficient and robust. Specifically, one issue is that the metric is based on the generator's Jacobian and its derivative, which is typically expensive and complex. Also, the uncertainty of the generative process is modeled meaningfully using kernel methods \citep{arvanitidis:iclr:2018} which limits further the robustness of the metric. Thus, for practical purposes it is sensible to search for a useful approximate Riemannian metric.


As regards VAEs, several improvements have been proposed and we are interested in learnable priors. Usually, in a VAE a simple prior over $\Z$ is chosen, as the unit Gaussian. This is not flexible and expressive enough in order to capture the structure of the data representations, which potentially might be complex and multimodal. Therefore, learnable priors have been proposed that adapt to the distribution of the latent representations \citep{tomczak:aistats:2018}.

{In this work} we propose a methodology to approximate the induced Riemannian metric in $\Z$ with a locally conformally flat surrogate metric that is based on a learnable prior. In particular, we first propose to utilize a basic energy-based model as a learnable prior in a VAE. Then, we define a conformal Riemannian metric in $\Z$ that is inverse proportional to the prior. This constitutes a robust metric that is also highly efficient both in computational speed, as well as in modeling capabilities. Furthermore, we study theoretically when the proposed metric is a sensible approximation to the Riemannian metric that is induced by the generator. In the experiments, we compare the behavior of the two metrics, and we also show potential applications in life sciences.


\section{Some basics of Riemannian geometry}
\label{sec:geometry}

We consider Riemannian manifolds \citep{lee:2018, docarmo:1992}, which are smooth spaces where one can compute lengths between points. An intuitive way to think of a $d$-dimensional smooth manifold $\M$ is as an embedded smooth $d$-dimensional hypersurface in a higher dimensional Euclidean ambient space $\X=\R^D$, which locally is homeomorphic to a $d$-dimensional Euclidean space. In this perspective, one considers the tangent space $\tangent{\bm{x}}{\M}$ at a point $\bm{x}\in\M$ as a $d$-dimensional vector space in $\X$ that is tangential to the hypersurface at the point $\bm{x}$. Technically, a manifold is covered by a collection of \emph{charts} (local parametrizations), and for simplicity we assume that a ``sufficiently large'' \emph{global chart} exists (a global parametrization of the hypersurface). We denote this global chart by the mapping $h:{\HH}\subseteq\R^d \rightarrow \M\subset\X$. By definition ${h}(\cdot)$ is a diffeomorphism onto its image and we say that ${\HH}$ are the intrinsic coordinates of the manifold.

A \textit{Riemannian metric} is a positive definite matrix that changes smoothly throughout the space. A smooth manifold $\M$ together with a Riemannian metric constitutes a Riemannian manifold. Let an embedded $\M\subset\X$ and a Riemannian metric ${\bm{M}_\X}:\X\rightarrow \R^{D\times D}_{+}$, which defines a local inner product at each tangent space $\tangent{\bm{x}}{\M}$ between two tangent vectors $\bm{u},\bm{v}\in\R^D$ at $\bm{x}\in\M$ as $\inner{\bm{u}}{\bm{v}}_{\bm{x}} = \inner{\bm{u}}{{\bm{M}_{\X}}(\bm{x})\bm{v}}$. The global parametrization ${h}(\cdot)$ allows to map a vector $\overline{\bm{v}}\in{\HH}$ to a unique tangent vector $\bm{v}\in\tangent{\bm{x}}{\M}$ using the Jacobian matrix $\bm{J}_{h}:{\HH} \rightarrow \R^{D\times d}$ as $\bm{v} = \bm{J}_{h}({\bm{z}}) \overline{\bm{v}}$. Since ${h}(\cdot)$ is smooth, a Riemannian metric ${\bm{M}_{\HH}}:{\HH} \rightarrow \R^{d\times d}_+$ is induced in ${\HH}$ as ${\bm{M}_{\HH}}({\bm{z}}) = \bm{J}_{{h}}({\bm{z}})^{\T} {\bm{M}_\X}({h}({\bm{z}}))  \bm{J}_{{h}}(\overline{\bm{z}})$. Commonly, we consider ${\bm{M}_\X}(\cdot) = \I_D$, which implies that the metric $\bm{M}_{\HH}(\cdot)$ in $\HH$ is induced by the embedding.

Essentially, the Riemannian metric shows how the distances change in an infinitesimal region. This enables us to compute the length of a curve $\gamma:[0,1]\rightarrow \M\subset\X$ as
\begin{equation}
	\label{eq:curve_length}
	\int_0^1 \sqrt{\inner{\dot{\gamma}(t)}{\dot{\gamma}(t)}_{\gamma(t)}} dt = \int_0^1 \sqrt{\inner{\dot{c}(t)}{\dot{c}(t)}_{c(t)}} dt,
\end{equation}
where $\gamma(t) = h(c(t))$ and $\dot{\gamma}(t) = \partial_t \gamma(t)= \bm{J}_h( c(t))\dot{c}(t) \in \tangent{\gamma(t)}{\M}$ represents the velocity of the curve. Hence, we can compute the length of $\gamma(t)$ in intrinsic coordinates $c(t)\in\HH$. In addition, we can find the shortest path between two points in $\HH$ by applying the Euler-Lagrange equations at the curve energy. This gives a system of second order non-linear ordinary differential equations (ODEs) $\ddot{c}(t) = F(\dot{c}(t), c(t), t)$ and for the system see App.~\ref{app:riemannian_geometry}. Intuitively, the shortest paths are pulled towards areas of $\HH$ where $\bm{M}_{\HH}(\cdot)$ is small. So, having $\HH$ together with a Riemannian metric $\bm{M}_{\HH}(\cdot)$ enables us to compute shortest paths between the corresponding points on $\M$. These curves are known as \emph{geodesics}.



\begin{figure}[t]
	\centering
	\includegraphics[width=0.57\columnwidth]{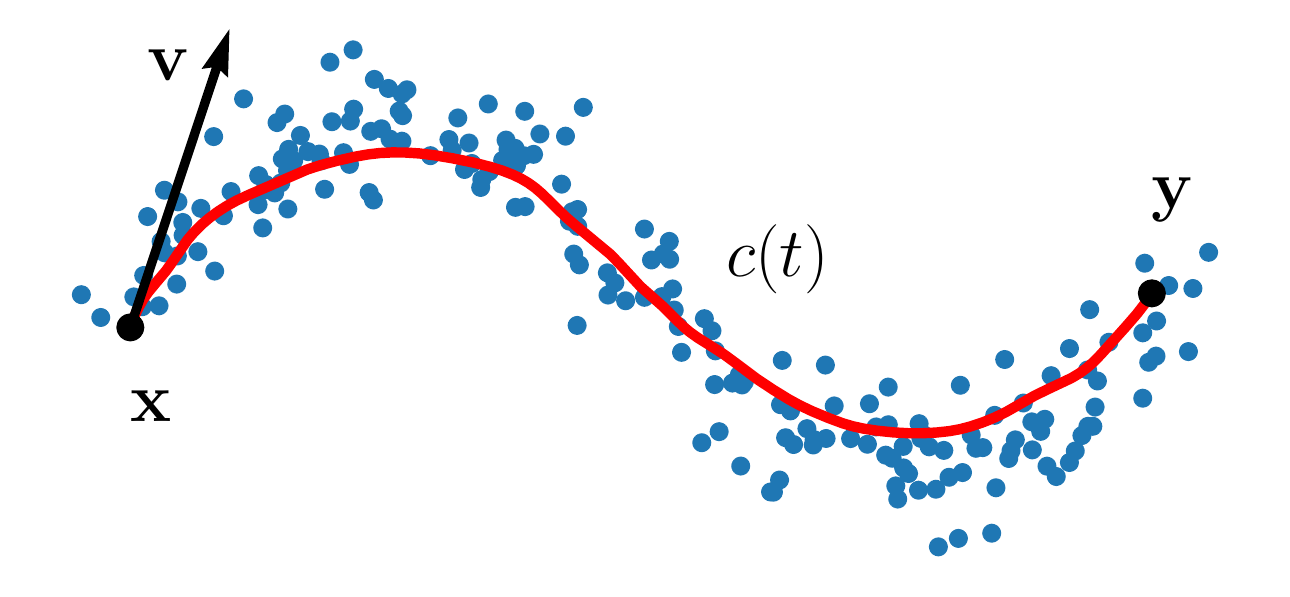}
	\caption{A shortest path $c(t)$ with a tangent vector $\bm{v}$.}
	\label{fig:geodesic_example}
\end{figure}

We say that a Riemannian metric $\bm{M}(\cdot)$ is conformal to another metric $\widetilde{\bm{M}}(\cdot)$ when for a positive smooth function $m:\HH \rightarrow \R_{>0}$ we have  $\widetilde{\bm{M}}(\bm{z}) = m(\bm{z})\cdot \bm{M}(\bm{z})$. Intuitively, the conformal metric is simply a scaling of $\bm{M}(\cdot)$, while the simplest example is to consider $\widetilde{\bm{M}}(\bm{z}) = m(\bm{z})\cdot \I_d$. Such a metric has some appealing properties such as interpretability and efficiency. More specifically, the corresponding ODE system for computing a shortest path simplifies to
\begin{equation}
	\label{eq:conformal_metric_ode}
	\ddot{c}(t) = \frac{\nabla m(c(t)) \dot{c}(t)^\intercal \dot{c}(t) - 2 \dot{c}(t)\nabla m(c(t))^\intercal \dot{c}(t)}{2~m(c(t))},
\end{equation}
where the $\nabla m:\HH\rightarrow \R^d$ is the gradient of $m(\cdot)$.  The interpretability implies that we can easily control the shortest paths behavior in $\HH$ by designing $m(\cdot)$ accordingly. 



We can do computations on $\M$ or equivalently in $\HH$ with the following operators. The \emph{logarithmic map} takes two points $\bm{x},~\bm{y}\in\M$ and returns a tangent vector $\bm{v}=\logmap{\bm{x}}{\bm{y}} \in\tangent{\bm{x}}{\M}$. The inverse operator is the \emph{exponential map} that takes $\bm{x}\in\M$ and $\bm{v}\in\tangent{\bm{x}}{\M}$ and returns a geodesic $\gamma(t) = \expmap{\bm{x}}{t\cdot\bm{v}}$ with $\gamma(1)=\bm{y}$. For an example see Fig.~\ref{fig:geodesic_example}, while for additional details on geometry see App.~\ref{app:riemannian_geometry}.



The manifold hypothesis assumes that the data lie uniformly near an embedded $\M\subset\X$. However, a global parametrization $h(\cdot)$ rarely exists and $d$ is unknown, especially, for a given set of finite and noisy observations. So, a practical method to capture the geometry of $\M$ is to learn a function $g:\Z\subseteq\R^{d'} \rightarrow\X$ to approximate the data, which is smooth but not constrained to be a diffeomorphism as $h(\cdot)$, and also, in general $\Z \neq \HH$. As before, using $\bm{J}_g(\cdot)$ we are able to compute a $\bm{M}_{\Z}(\cdot)$ with the desired meaningful behavior to be small in parts of $\Z$ that correspond to regions of $\M$ with non-zero data density. This is known as the \emph{pull-back} metric, and in practice, we use generative models to learn $g(\cdot)$. Also, as we discuss in Sec.~\ref{sec:learned_latent_riemannian_manifolds}, we need $g(\cdot)$ to be a stochastic function with meaningful uncertainty estimates in order to properly capture the geometry of the data manifold in $\Z$.

\section{Riemannian Metric Learning}
\label{sec:riemannian_metric_learning}

Apart from the pull-back metric, there is also another way to learn a Riemannian metric directly from the given data. Let assume $N$ data points $\bm{x}_{1:N}\in\X$. We can construct a parametric Riemannian metric $\bm{M}_\lambda:\X \rightarrow \R_+^{D \times D}$ with parameters $\lambda$ directly from the observations, which enables us to compute shortest paths that respect their underlying geometric structure in $\X$. Essentially, we want to pull these paths towards the regions of $\X$ with non-zero data density. This implies that the metric should be small near the data and to increase as moving further from them. In practice, this simply changes the way we measure distances in $\X$ and a conceptual example can be seen in Fig.~\ref{fig:geodesic_example}.

One approach is to consider a predefined set of metric tensors centered at some base points in $\X$, and then using a kernel to compute the Riemannain metric as a weighted sum of the predefined tensors \citep{hauberg:nips:2012}. In a similar spirit, \citet{arvanitidis:neurips:2016} used a kernel to compute the Riemannian metric as the inverse local diagonal covariance matrix of the data. Also, \citet{arvanitidis:aistats:2019} proposed a simple method to construct conformal Riemannian metrics directly from the data by multiplying a positive function in $\X$ with the Euclidean metric. For details about these metrics see App.~\ref{app:riemannian_metrics}. However, in these approaches the parameters $\lambda$ are typically fixed, and in general, it is a challenging task to find the best $\lambda$ \citep{arvanitidis:gsi:2017}.


In contrast, \citet{lebanon:uai:2003} proposed a simple methodology in order to estimate the parameters $\lambda$ of a predefined parametric Riemannian metric $\bm{M}_\lambda(\cdot)$ directly from the data. First the density function $p_\lambda(\bm{x}) \propto (\sqrt{|\bm{M}_\lambda (\bm{x})|})^{-1}$ is defined and assuming that the data are independent and identically distributed we get the likelihood
\begin{equation}
\label{eq:metric_density}
	\lambda^* = \argmax_\lambda \prod_{n=1}^N \frac{(\sqrt{|\bm{M}_\lambda (\bm{x}_n)|})^{-1}}{\int_{\X} (\sqrt{|\bm{M}_\lambda (\bm{x}')|})^{-1} d\bm{x}'},
\end{equation}
which we can optimize using maximum likelihood estimation. Intuitively, the density should be high near the given data, which directly means that the metric should be small, while the regularizer does not allow the metric to become zero. The quantity $ \sqrt{|\bm{M}_\lambda (\bm{x})|}$ is the magnification factor, which is a scaling factor for the Lebesgue $d\bm{x}$ and represents the local distrortion of the distance.


Obviously, this approach for Riemannian metric learning has some disadvantages. Most importantly, we have to define explicitly the parametric form of the metric before the training, which potentially limits its flexibility. Also, in higher dimensions it is hard to guarantee the actual behavior and usability of such an ad-hoc metric. In addition, the optimization is challenging especially in high dimensions due to the normalization constant. Therefore, this methodology is mostly limited to low dimensional spaces, where it is easy to define a metric and the data manifold is simple. However, the actual formulation motivates us to relate a density function with a conformal metric (see Sec.~\ref{sec:our_proposed_metric}).

\section{Generative Models}
\label{sec:generative_models}

An efficient way to approximate the underlying probability density function of the observations $\bm{x}_{1:N} \in \X$ is to learn a generative model. Recent advances in deep generative modeling showed a great performance in this task. In particular, there are several types of generative models such as Variational Auto-Encoders (VAEs) \citep{kingma:iclr:2014, rezende:icml:2014}, Generative Adversarial Networks (GANs) \citep{goodfellow:neurips:2014} and flow based models \citep{dinh:2016:arxiv}. In this work, we are interested in the VAE model, where a low dimensional latent space $\Z$ is utilized in order to construct an explicit density model in the ambient space $\X$. Additionally, we can get in the latent space a low dimensional representation of the data.

Specifically, we use a likelihood function $p_\theta(\bm{x} | \bm{z})$ that is typically chosen to be a Gaussian $\N(\bm{x} ~|~ \mu_\theta(\bm{z}),~\I_D \cdot \sigma^2_\theta(\bm{z}))$ or a $\text{Bernoulli}(\bm{x} ~|~ \mu_\theta(\bm{z}))$ in case the data being binary, and a prior distribution over the latent variables $p(\bm{z})$. Commonly, the prior is chosen to be a simple distribution as $\N(0, \I_d)$. The functions that parametrize the likelihood $\mu_\theta:\Z\rightarrow \X$ and $\sigma^2_\theta:\Z \rightarrow \R_{>0}^D$ are usually deep neural networks. One way to learn the parameters of these functions is by using an approximate posterior $q_\phi(\bm{z}|\bm{x}) = \N(\bm{z} ~|~ \mu_\phi(\bm{x}),~\I_d\cdot \sigma^2_\phi(\bm{x}) )$, where again $\mu_\phi:\X\rightarrow \Z$ and $\sigma^2_\phi:\X \rightarrow \R_{>0}^d$ are deep neural networks. Then, we can derive using Jensen's inequality the evidence lower bound as 
\begin{equation}
	\label{eq:vae_elbo}
	\E_{q_\phi(\bm{z} | \bm{x})}[\log p_\theta(\bm{x} | \bm{z})] - \text{KL}[q_\phi(\bm{z}| \bm{x}) || p(\bm{z})],
\end{equation}
which is a lower bound to the log-likelihood for a point $\bm{x}$. Now, we are able to optimize this objective function using the reparametrization trick $\bm{z} = \mu_\phi(\bm{x}) + \text{diag}(\varepsilon)\cdot \sigma_\phi(\bm{x})$, where $\varepsilon\sim \N(0,\I_d)$, which allows to compute stochastic gradients with low variance \citep{mohamed:jmlr:2020}.

Even if the standard VAE provides a successful way to approximate the data density, many variants have been proposed that improve the basic model in several aspects. One line of work proposes to use more flexible approximate posteriors, which potentially improve the lower bound \citep{rezende:icml:2015, titsias:aistats:2019}. Another line of work provides tighter lower bounds to the log-likelihood using importance sampling techniques \citep{burda:arxiv:2016}. Finally, some approaches suggest instead of using a simple prior for the latent variables to learn a flexible prior, which desirably adapts better to the latent representations as in an empirical Bayes setting \citep{tomczak:aistats:2018, bauer:aistats:2019}. Intuitively, the behavior of the learnable prior in $\Z$ is similar to a meaningful Riemannian metric, and thus, we focus in this type of models.

\subsection{Prior learning in Variational Auto-Encoders}

One of the first successful methodologies to learn the prior in a VAE is the VampPrior \citep{tomczak:aistats:2018}. In this approach the learnable prior is chosen to be the aggregated posterior $p(\bm{z}) \triangleq q(\bm{z}) = \int_\X q(\bm{z} | \bm{x}) p(\bm{x}) d\bm{x}$, where $p(\bm{x})$ the true density. In a standard VAE this is essentially a huge Gaussian mixture model since typically we approximate this integral using the training data $p(\bm{z}) \approx \frac{1}{N}\sum_{n=1}^N q(\bm{z} | \bm{x}_n)$. Of course, such a prior can easily overfit, so the authors proposed to use instead only $K$ learnable inducing points $\bm{x}_{1:K}$. This simple prior is empirically shown to be very effective, however, when the data dimension is high, training the inducing points is computationally expensive. Also, is hard to chose the number $K$ of the inducing points.

One variant is to learn implicitly the VampPrior using a discriminator, which does not need to set $K$ \citep{takahashi:aaai:2019}. However, with this method we can only get samples from the prior while an analytic formula does not exist. Similarly, \citet{klushyn:neurips:2019} proposed a hierarhical prior as $p(\bm{z}) = \int p(\bm{z}|\zeta)\bm(\zeta) d\zeta$, which in practice is approximated by $K$ samples from the hyper-prior $\zeta_k \sim p(\zeta)$, together with a complicated constrained optimization strategy specifically designed for this problem. Again here, an analytic formula for the prior is not easy to be derived, but only samples. 


Another set of approaches is related to the energy-based models. \citet{bauer:aistats:2019} proposed a prior where an acceptance function is trained to accept or reject samples from a base prior as the unit Gaussian. \citet{pang:arxiv:2020a} proposed to learn an energy-based model prior directly by optimizing the log-likelihood of the data requiring iterative expensive Markov Chain Monte Carlo sampling in the latent space for the prior and the true posterior. Finally, \citet{aneja:arxiv:2020} proposed an energy-based model prior trained by contrasting samples from the aggregated posterior to samples coming from a base prior, but this prior is trained post-hoc. These approaches motivate our proposed prior.


\subsection{A learnable prior for Variational Auto-Encoders}
\label{sec:our_proposed_prior}

Let a function $f_\psi:\Z \rightarrow \R$ parametrized as a deep neural network and a base distribution $p(\bm{z}) = \N(0, \I_d)$. We use as learnable prior the energy-based model \citep{lecun:energy:2006}
\begin{equation}
\label{eq:our_prior}
	\nu_\psi(\bm{z}) = \frac{\exp(f_\psi(\bm{z})) p(\bm{z})}{\C},
\end{equation}
where $\C = \int_\Z \exp(f_\psi(\bm{z})) p(\bm{z}) d\bm{z}$ is the normalization constant. Then, we plug this prior in the evidence lower bound of the VAE, so Eq.~\ref{eq:vae_elbo} now becomes
\begin{align}
\label{eq:our_new_elbo}
	\E_{q_\phi(\bm{z} | \bm{x})}&[\log p_\theta(\bm{x} | \bm{z})] - \text{KL}[q_\phi(\bm{z}| \bm{x}) || p(\bm{z})]\nonumber\\
	+ &\E_{q_\phi(\bm{z} | \bm{x})}[f_\psi(\bm{z})] - \log(\C),
\end{align}
which can be optimized using stochastic gradients as well. Obviously, the challenge in this bound is to estimate the normalization constant. However, since the dimensionality of $\Z$ is usually low this allows us to estimate the normalization constant $\C$ relying on basic Monte Carlo as $\C = \int_\Z \exp(f_\psi(\bm{z})) p(\bm{z}) d\bm{z}\approx \frac{1}{S}\sum_{s=1}^S \exp(f_\psi(\bm{z}_s))$ where $\bm{z}_s \sim p(\bm{z})$. Nevertheless, more sophisticated techniques for estimating the constant can be used.

Even if this is a rather simple prior, it comes with some desirable properties. First, the behavior is easy to interpret, as the prior increases near the latent codes of the data, while in contrast, the normalization constant tries to reduce its value in regions of $\Z$ with no codes. This implicit regularization does not allow the model to overfit the latent codes, which is directly related to the effectiveness of the integration. Also, contrastive techniques can be used in order to control even further the prior fitting i.e. far from latent codes to push the prior towards zero. Also, the KL divergence of the standard VAE still appears in the objective. This is beneficial because the encoder is still encouraged to provide a meaningful structure for the latent codes, while in a different case the representations could be placed sparsely without any structure in $\Z$ depending on the flexibility of $f_\psi(\cdot)$. 

Clearly, our proposed energy-based model prior is a rather simple choice, while being closely related to more advanced models which aim to improve generative modeling \citep{pang:arxiv:2020a, aneja:arxiv:2020}. However, to the best of our knowledge, such a prior has not be used in the standard VAE setting. Also, our main motivation for proposing this prior is not to improve the generative modeling performance, but instead to have a flexible prior that adapts to the data, which is efficient to evaluate and derivate. As we show in Sec.~\ref{sec:our_proposed_metric} this prior is the base to define a conformal metric in $\Z$, which approximates the geometry of the data manifold, while being on the same time efficient and robust.

\section{Riemannian metric via generative modeling}
\label{sec:learned_latent_riemannian_manifolds}


Instead of learning a Riemannian metric in $\X$ from data (see Sec~\ref{sec:riemannian_metric_learning}), we discuss how to learn one in the latent space $\Z$ of a generative model. Briefly, a generator $g:\Z\rightarrow \X$ induces a pull-back metric in $\Z$ (see Sec.~\ref{sec:geometry}) that essentially informs us about the local distortions of $\Z$ when mapping through $g(\cdot)$. In principle, this metric captures the geometry of the data manifold lying in $\X$. However, as we discuss in this section even if this is a theoretically rigorous approach, it comes with some practical disadvantages.

\citet{tosi:uai:2014} first proposed to capture the geometry of a data manifold by modeling the generator $g(\cdot)$ using a Gaussian Process Latent Variable Model (GP-LVM) \citep{lawrence:2005:jmlr}. In particular, the generator is taken to be a Gaussian process $g \sim \text{GP}(0, k(\bm{z}, \bm{z}'))$ and the latent codes of the data $\bm{z}_{1:N}$ are trained. Since GPs are closed under differentiation the Jacobian $\bm{J}_g(\cdot)$ is a random process, and consequently, a stochastic Riemannian metric is induced in $\Z$. This metric comes with a meaningful behavior, since it is small near the latent codes and increases when the uncertainty of $g(\cdot)$ increases. Obviously, this properly captures in $\Z$ the geometry of the data manifold. However, apart from this desired behavior this metric is not very useful due to the practical constraints that are induced from the GP.

In a similar spirit, \citet{arvanitidis:iclr:2018} derived a Riemannian metric using deep generative models. In particular, for a standard VAE the generator can be written as a stochastic function $\bm{x} = g_\theta(\bm{z}) = \mu_\theta(\bm{z}) + \text{diag}(\varepsilon)\cdot \sigma_\theta(\bm{z})$ where $\varepsilon\sim\N(0,\I_D)$. This induces a random Riemannian metric in the latent space $\Z$ for which the expectation is
\begin{equation}
	\label{eq:pullback_vae_metric}
	\bm{M}_\theta(\bm{z}) = \bm{J}_{\mu_\theta}(\bm{z})^{\intercal}\bm{J}_{\mu_\theta}(\bm{z}) + \bm{J}_{\sigma_\theta}(\bm{z})^{\intercal} \bm{J}_{\sigma_\theta}(\bm{z}).
\end{equation}
This metric can be interpreted, since when the uncertainty of $g(\cdot)$ increases, the second term of the metric becomes large, which constitutes a meaningful behavior. However, $\mu_\theta(\cdot)$ and $\sigma_\theta(\cdot)$ are usually parametrized as deep neural networks that are known to extrapolate arbitrarily. A solution proposed by \citet{arvanitidis:iclr:2018} is to use a positive Radial Basis Function (RBF) network to model the precision $\xi_\theta(\bm{z}) = (\sigma^2_\theta(\bm{z}))^{-1}$. Hence, moving further from the latent codes decreases the precision, which directly makes the second term of the expected metric Eq.~\ref{eq:pullback_vae_metric} to increase. Therefore, a stochastic generator together with meaningful uncertainty estimates enables us to properly capture the geometry of the data manifold in $\Z$ \citep{hauberg:only:2018}. Moreover, it has been theoretically shown in \citet{eklund:arxiv:2019} that this expected metric is sensible to use.

Nevertheless, even if this approach allows us to compute shortest paths in $\Z$ that respect the latent codes structure, it comes with some practical drawbacks. In particular, modeling the precision with an RBF is a reasonable choice to estimate meaningfully the uncertainty, but it based on a kernel as the Gaussian. So we need to select the number of components $K$, as well as their parameters as the bandwidth, which is in general a challenging problem. Also, in a high dimensional space $\Z$ the metric is not robust, because due to the curse of dimensionality it is hard to control the support of the kernel which causes an unstable metric (Eq.~\ref{eq:pullback_vae_metric}).

Additionally, to compute one shortest path we evaluate the corresponding ODE system several times, which involves the metric and its derivative that are based on the Jacobian of $g(\cdot)$. Clearly, for complex generators this is computationally very expensive. Also, by definition $g(\cdot)$ has to be twice differentiable, which makes hard the use of complicated architectures. Finally, the ODE system becomes highly unstable which affects negatively the performance of the solvers \citep{arvanitidis:aistats:2019}. Even if we can use solvers that are based on automatic differentiation and optimize a parametric curve by directly minimizing the energy \citep{yang:arxiv:2018}, for complex models this is still slow. Also, under this approach we cannot compute the exponential map that is typically necessary for Riemannian statistics.

Clearly, stochastic generators provide a theoretically solid methodology to properly capture in $\Z$ the geometry of the data manifold lying in $\X$. Moreover, this approach enables us to derive more informative metrics by considering the space $\X$ as a Riemannian manifold \citep{arvanitidis:arxiv:2020}. However, due to their mainly practical disadvantages, we are interested to approximate the geometry in $\Z$ using a simple, efficient and robust surrogate Riemannian metric.

\subsection{A prior-based conformal metric}
\label{sec:our_proposed_metric}

We propose a new Riemannian metric in $\Z$ that approximates the behavior of the true pull-back metric $\bm{M}_\theta(\cdot)$ (see Eq.~\ref{eq:pullback_vae_metric}), while having several advantages as regards its practicality. Let a VAE with a trainable smooth prior $\nu_\psi(\bm{z})$ for which we can evaluate easily the density function, as well as its derivative. We are motivated by \citet{lebanon:uai:2003} where a probability density is defined to be inverse proportional to the magnification factor (see Eq.~\ref{eq:metric_density}). In a similar spirit, we propose an approximation to the true $\bm{M}_\theta(\cdot)$ in $\Z$ using the following locally conformally flat Riemannian metric
\begin{equation}
	\label{eq:new_conformal_metric}
	\bm{M}_\psi(\bm{z}) = m(\bm{z}) \cdot \I_d = {(\alpha \cdot \nu_\psi(\bm{z}) + \beta)^{-\nicefrac{2}{d}}} \cdot \I_d,
\end{equation}
where $\alpha, ~\beta>0$ are scaling constants that allow to lower and upper bound the metric, respectively. This metric by definition is conformal to the Euclidean metric $\I_d$ in $\Z$, and also, the quantity $\sqrt{|\bm{M}_\psi(\bm{z})|} = (\alpha \cdot \nu_\psi(\bm{z}) + \beta)^{-1}$ is inverse proportional to the learnable prior.

Clearly, $\bm{M}_\psi(\cdot)$ has an interepretable and meaningful behavior, as in regions of $\Z$ where the density is high the metric is small, and thus, the shortest paths are pulled towards the latent codes. Intuitively, this properly captures the geometry of the high dimensional data manifold, at least in the sense that paths tend to avoid regions of $\Z$ with no latent codes. Additionally, the metric is directly learned from the data, while depending of the flexibility of $\nu_\psi(\cdot)$ it can be highly adaptive. This further implies that the metric is more robust in higher dimensional latent spaces, as in principle, does not depend on a predefined parametric form and/or a kernel. Also, as a conformal metric the corresponding ODE system simplifies (see Eq.~\ref{eq:conformal_metric_ode}). Hence, the proposed prior  Eq.~\ref{eq:our_prior} seems to be a perfect choice, since it is flexible, adaptive and efficient to evaluate, as well as to derivate.  

Of course, we can control the capacity of $f_\psi(\cdot)$ so that the prior does not overfit the latent codes. Similarly, advanced training techniques can be used to improve the fitting of the prior e.g. contrastive learning. Moreover, the prior can be easily replaced by a more sophisticated model that performs better, as long as the functional form of the density and its derivative are easy to compute.

\subsection{Theoretical analysis of the proposed metric}
\label{sec:new_metric_discussion}


Even if our proposed metric seems to be a good surrogate for the pull-back metric of \citet{tosi:uai:2014} and \citet{arvanitidis:iclr:2018}, here we analyze and compare its behavior in detail, and essentially, we are interested in the following problem. Let the smooth manifold $\Z=\I_d$ and two Riemannian metrics, the pull-back $\bm{M}_\theta(\cdot)$ and the proposed conformal $\bm{M}_\psi(\cdot)$ metric. We are interested if these two metrics result in shortest paths on the data manifold that are equivalent, which means that the corresponding curves in $\Z$ should be similar. Here, we analyze the behavior in three specific cases and we provide constructive demonstrations in the experiments. In addition, we show that under mild conditions the two metrics result in similar shortest paths.



We know that the pull-back metric  (Eq.~\ref{eq:pullback_vae_metric}) due to the second term increases in regions of the latent space where the uncertainty of $g(\cdot)$ increases. Of course, when the prior density is zero, the uncertainty is maximum, which implies that as the density decreases both metrics increase. Therefore, the behavior of the shortest paths is similar, since under both metrics they will be pulled towards the latent codes and avoid regions of $\Z$ with near zero density. Note that the two metrics are structurally different as $\bm{M}_\theta(\cdot)$ is a full matrix while $\bm{M}_\psi(\cdot)$ is simply a diagonal matrix, but practically their behavior is similar avoiding the same regions in $\Z$.

However, even if the paths follow regions with non-zero density in both cases, their specific behavior in there is hard to predict. The manifold hypothesis assumes that the data lie uniformly around $\M\subset\X$, and hence, we assume that the uncertainty of the generator is locally constant within regions of $\Z$ with latent codes. So the second term of $\bm{M}_\theta(\cdot)$ is near zero and only the first term captures the geometry. Of course, in this case the behavior of $\bm{M}_\theta(\cdot)$ is not necessarily similar to $\bm{M}_\psi(\cdot)$. For example, if the curvature of $\mu_\theta(\cdot)$ is high the pull-back increases, while the prior in the same region can be high as well, such that to ensure a uniform distribution of points around $\M$. A natural assumption thought, is that in regions of $\Z$ with uniform non-zero density the curvature of $g(\cdot)$ is small, so both metrics only locally result to approximately similar paths.

\begin{proposition}
\label{pror:straight_lines_prop}
    Let $\bm{M}_\theta(\cdot)$ and $\bm{M}_\psi(\cdot)$ the Riemannian metrics over the latent space $\Z$. We consider a neighborhood $\mathcal{U}$ of the data manifold $\M\subset\X$ and based on the manifold hypothesis, we assume that the data lie uniformly around $\mathcal{U}$. Let us suppose that in the corresponding region in $\Z$:
    \begin{enumerate}
        \item The density $\nu_\psi(\cdot)$ is approximately uniform.
        \item The generator's uncertainty $\sigma_\theta(\cdot)$ is approximately constant and in addition $\mu_\theta(\cdot)$ has low curvature.
    \end{enumerate}
Then for both the conformal $\bm{M}_\psi(\cdot)$ and the pull-back metric $\bm{M}_\theta(\cdot)$ the shortest paths are approximately straight lines.
\end{proposition}
\begin{proof}
    See App.~\ref{app:prop_straight_lines_prop}.
\end{proof}

However, in practice the assumptions of Prop.~\ref{pror:straight_lines_prop} does not hold always, especially due to the RBF (see App.~\ref{app:constructive_examples}). Also, there exist at least one case where the two metrics have exactly the opposite behavior. Consider a manifold where the data are uniformly distributed around it, except one part where there are more data with higher noise. This means that the corresponding region in $\Z$ will have higher density, since more latent codes will be encoded therein, which implies that $\bm{M}_\psi(\cdot)$ will be smaller. In contrast, $\bm{M}_\theta(\cdot)$ increases in the same region, since the uncertainty of the generator will also increase due to the actual data distribution. Therefore, the shortest paths will have the exact opposite behavior, and in particular, $\bm{M}_\psi(\cdot)$ will be misleading. 

The analysis in this section implies that the two metrics induce approximately the same topology in $\Z$, as in both cases shortest paths prefer regions with non-zero density. Also, if the curvature of $g(\cdot)$ is low, we showed that locally the two paths are similar. However, we note that potentially additional problematic cases might exist. Nevertheless, one important benefit of the proposed metric is that we can easily control it during learning through the prior. For example, a naive computational approach is to consider a pointwise regularizer of the form $||\bm{M}_\psi(\bm{z}) - \bm{M}_\theta(\bm{z})||^2_F$. Therefore, the proposed metric enables us to take into account the geometry during learning the model. In this way, we are able to influence the model by considering interpretable inductive biases throught geometric formulations.

\section{Experiments}
\label{sec:experiments}

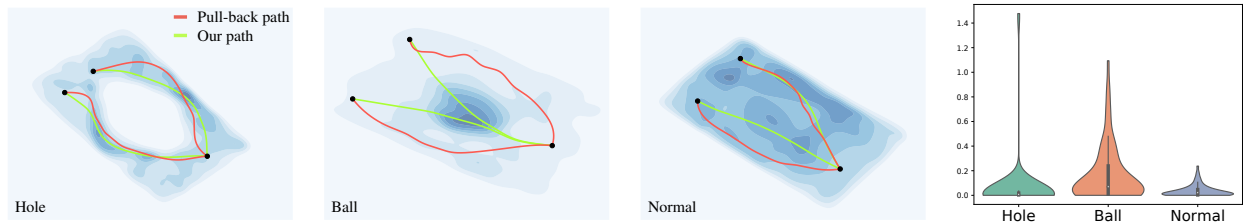
\begin{figure*}[h]
	\centering
	\begin{subfigure}{0.23\linewidth}
	    \begin{overpic}[width=1\linewidth]{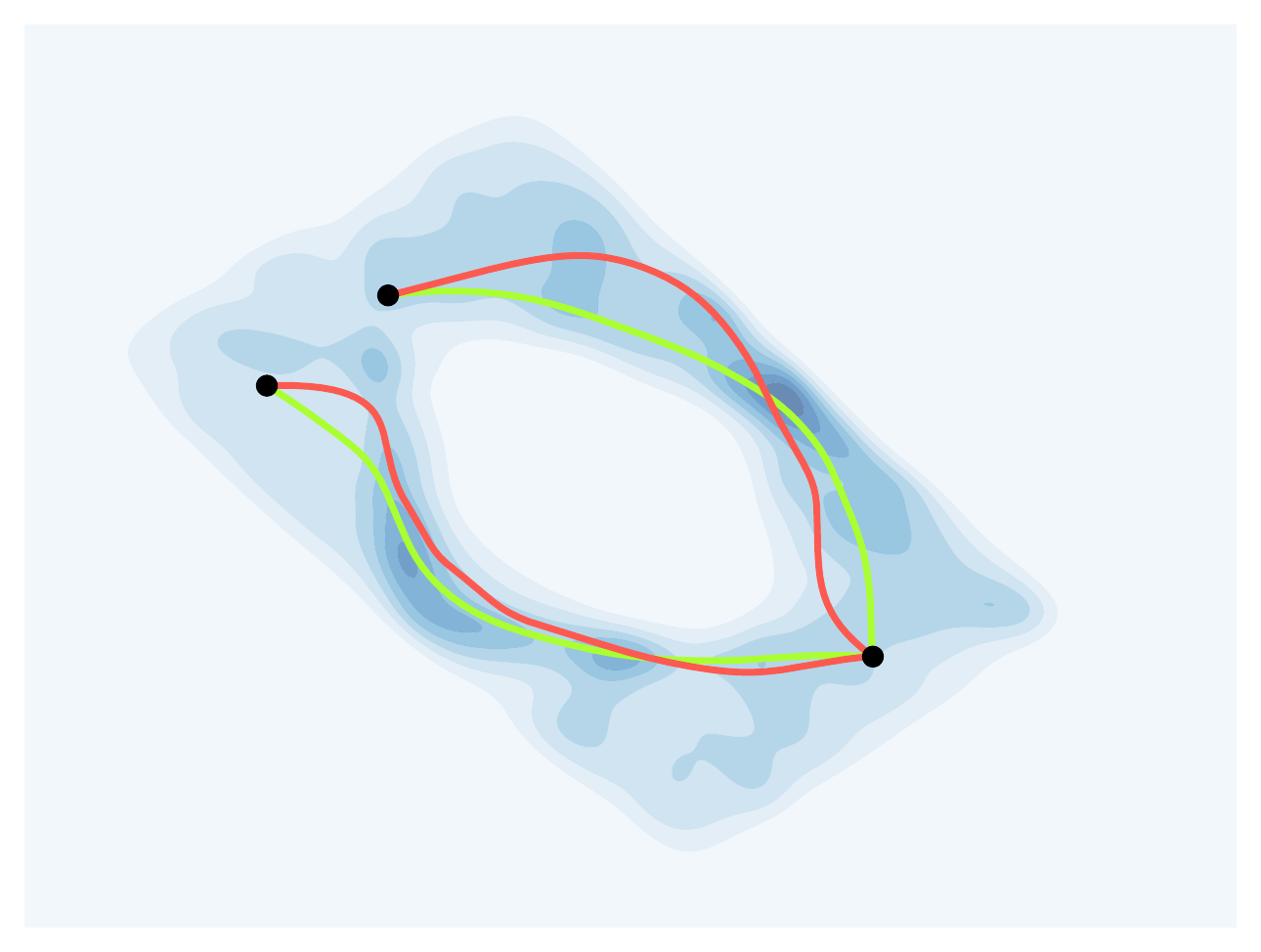}
        \put(5, 5){\tiny Hole}
        \put(74,77){{\tiny Pull-back path}}
	    \put(70,79){{\color[RGB]{234,101,88}	\linethickness{1.3pt}\line(-1,0){5}}}
	    \put(74,70){{\tiny Our path}}
	    \put(70,72){{\color[RGB]{189,252,83}	\linethickness{1.3pt}\line(-1,0){5}}}
    \end{overpic}
    \end{subfigure}
	~
	\begin{subfigure}{0.23\linewidth}
	    \begin{overpic}[width=1\linewidth]{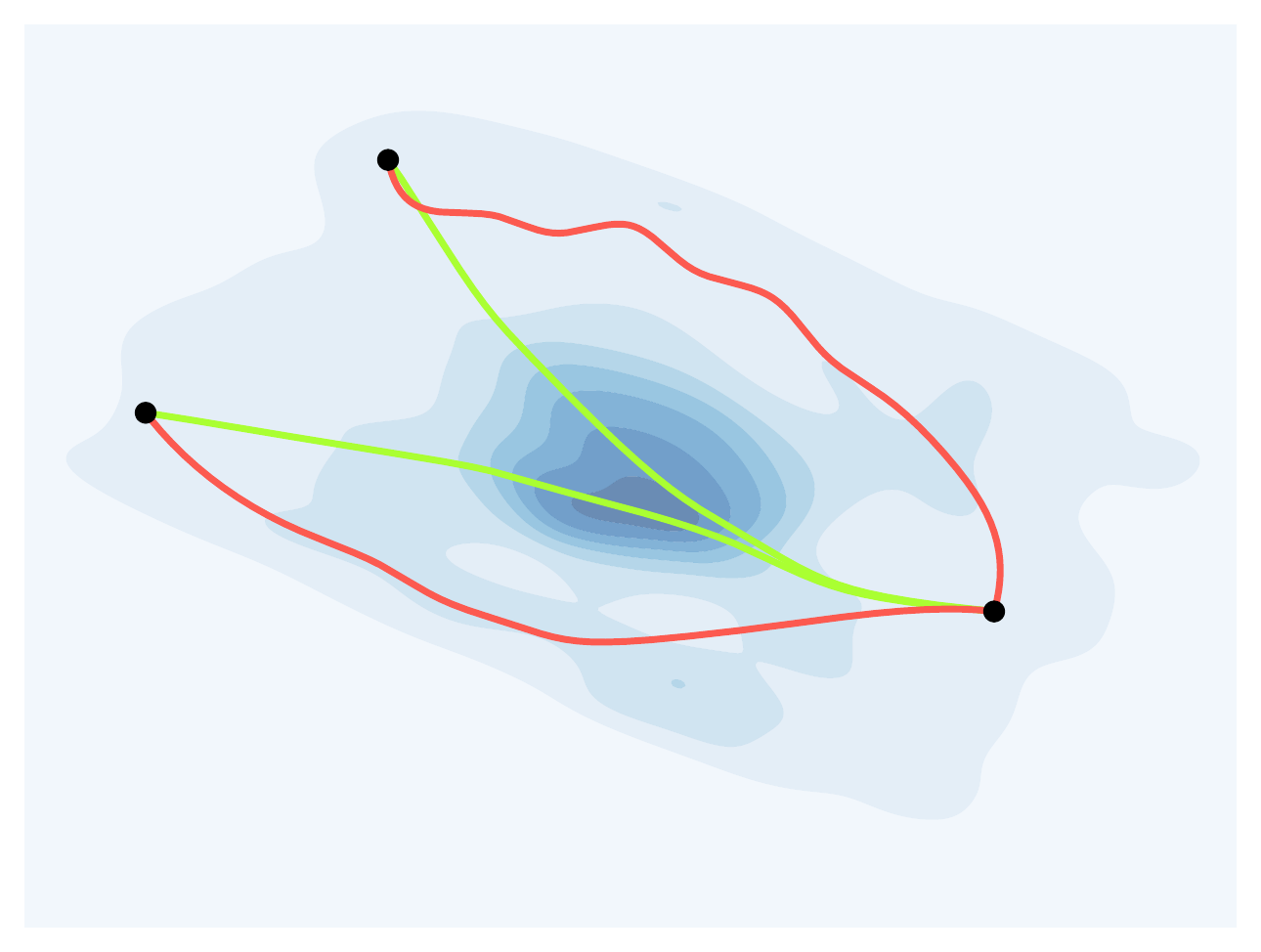}
        \put (5, 5) {\tiny Ball}
    \end{overpic}
    \end{subfigure}
	~ 
	\begin{subfigure}{0.23\linewidth}
	    \begin{overpic}[width=1\linewidth]{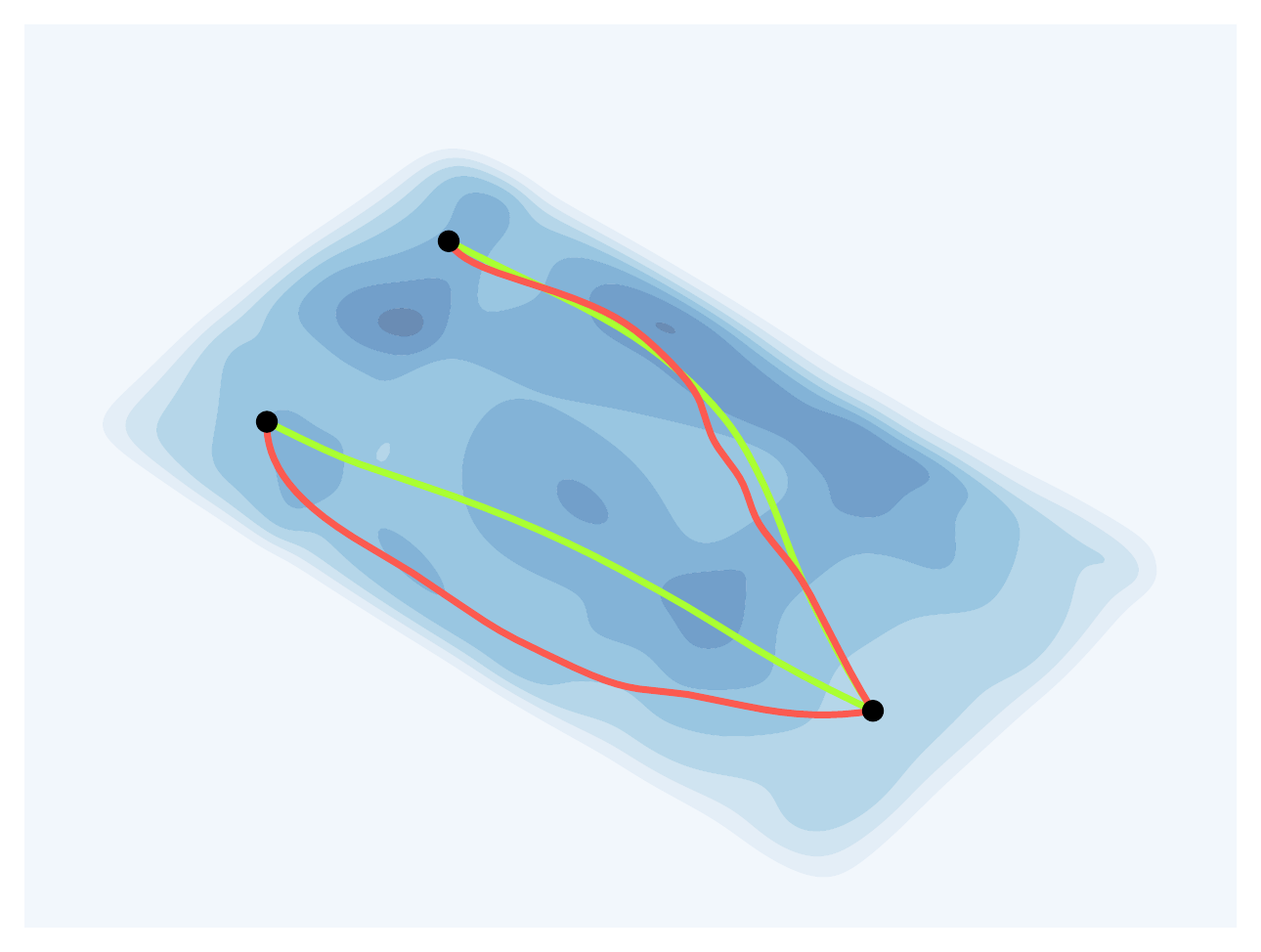}
        \put (5, 5) {\tiny Normal}
    \end{overpic}
    \end{subfigure}
	~
	\begin{subfigure}{0.23\linewidth}
	    \begin{overpic}[width=1\linewidth]{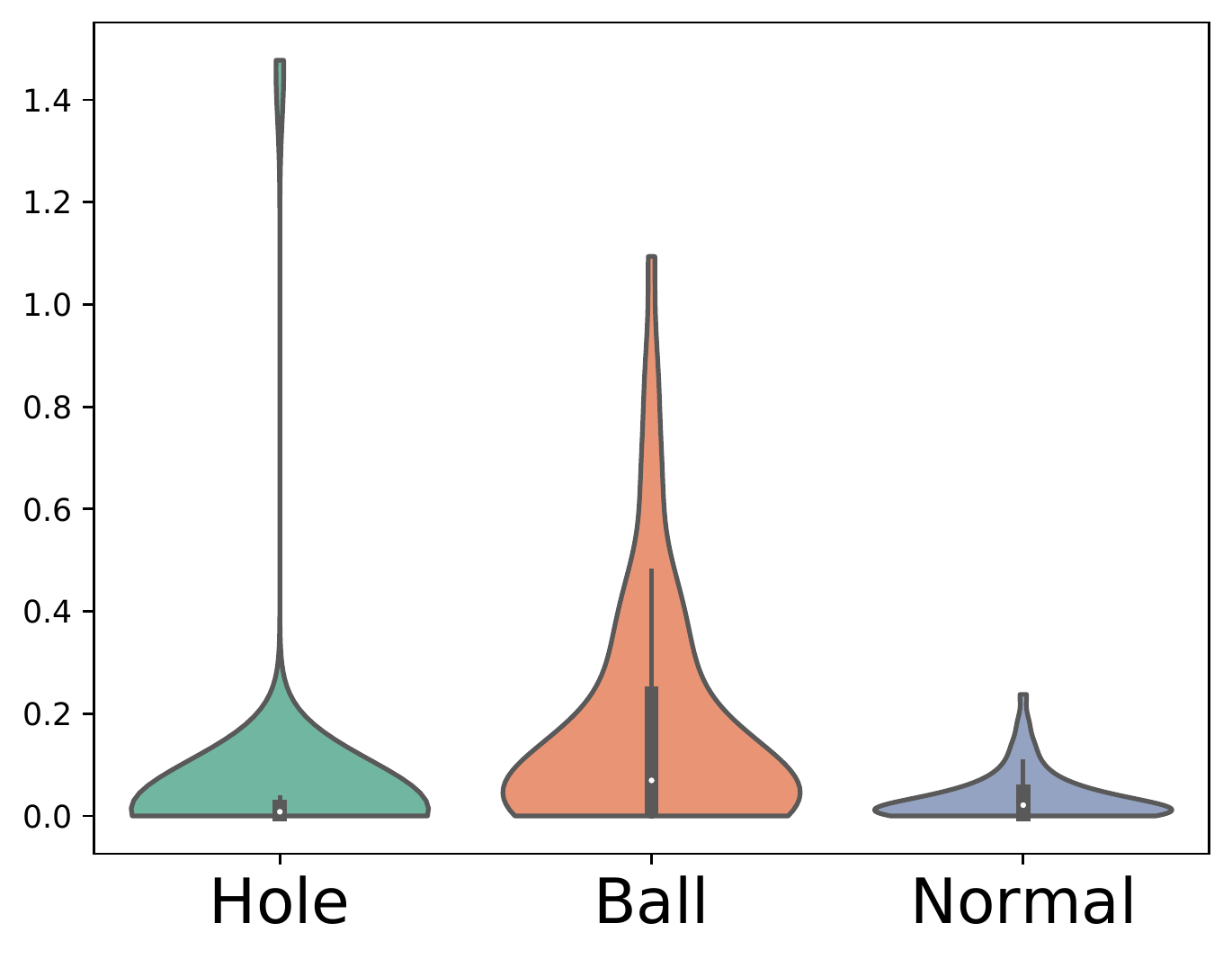}
    \end{overpic}
    \end{subfigure}
	\caption{Demonstrating the three cases analyzed in Sec.~\ref{sec:our_proposed_metric}. \emph{Left}: The two metrics behave similarly since the uncertainty estimation aligns well with the prior. \emph{Middle}: The behavior of the metrics is exactly the opposite, since the area with higher density has also higher uncertainty. \emph{Right}: Based on Prop.~\ref{pror:straight_lines_prop} we expect the shortest paths locally to be similar, as long as the prior is uniform and the curvature of the generator is small. Additionally, we show the distribution of distances between the curves for each case respectively. Note that in the case with the hole, the shortest paths tend to be similar except for a few outliers, while in the ball case many paths are not similar.}
	\label{fig:ex2:compare_cases}
\end{figure*}

Our experimental setting is two fold. First, we compare our prior to the state-of-the-art VampPrior \cite{tomczak:aistats:2018}. Note that our goal is not to improve generative modeling, but to show that $\nu_\psi(\cdot)$ adapts well to the latent codes, so it is a sensible choice for $\bm{M}_\psi(\cdot)$. Then, we compare the proposed metric with the pull-back $\bm{M}_\theta(\bm{z})$ of \citet{arvanitidis:iclr:2018} on several aspects as the robustness and the efficiency of shortest paths. Also, we provide a constructive example based on the analysis of Sec.~\ref{sec:our_proposed_metric}. Finally, we show applications of Riemannian statistics in life sciences. Details for the experiments and code can be found in App.~\ref{app:experiments}.

\subsection{Performance of the proposed prior}

We compare in terms of log-likelihood our learnable prior to the standard unit Gaussian and the VampPrior. We train 10 Convolutional-VAEs on MNIST and FashionMNIST datasets and we report the mean log-likelihood of test data in Table~\ref{tab:test_nll_prior_comparisson}, which we computed using importance sampling with 5000 samples as $p(\bm{x}) \approx \frac{1}{S}\sum_{s=1}^S \frac{p_\theta(\bm{x}|\bm{z}_s) p(\bm{z}_s)}{q_\phi(\bm{z}_s |\bm{x})}$ where $\bm{z}_s \sim q_\phi(\bm{z} | \bm{x})$. In addition, using PCA we projected the datasets in 100 dimensions and we fitted 10 VAEs with Gaussian decoders. This already captures $>90\%$ of the data variance, while enables us to use stochastic decoders such that to use the pull-back metric in the latent space. In both cases the dimension of the latent space is $d=10$.

\begin{table}[t]
	\centering
	\begin{tabular}{c c c c} 
		&	Standard	&	VampPrior	&	Ours	\\ \midrule
		MNIST	&	$85.38$	&	$83.28$ &	$83.56$	\\
		FMNIST	&	$227.12$	&	$224.15$    &	$224.53$	\\ \midrule
		MNIST $(100)$	&	$95.51$	&	$90.24$	&	$91.70$	\\
		FMNIST $(100)$	&	$87.17$	&	$81.83$	&	$84.06$	\\
	\end{tabular}
	\caption{The negative mean log-likelihood on test data.}
	\label{tab:test_nll_prior_comparisson}
\end{table}

For the VampPrior we use $K=500$ learnable inducing points and for our prior $f_\psi(\cdot)$ we use a fully connected 2-layer deep network with 128 units per layer and \texttt{tanh} activations. From the results in Table~\ref{tab:test_nll_prior_comparisson} we see that our proposed prior is comparable to the VampPrior, while being always better than the unit Gaussian prior. This shows that $\nu_\psi(\cdot)$ adapts well to the latent codes during training.


\subsection{Comparing the behavior of the metrics}
\label{sec:exp:constructive_examples}

Here we provide examples for the analysis in Sec.~\ref{sec:our_proposed_metric}. We construct a surface in $\X=\R^3$ as $[\bm{z}, 0.25 \cdot \sin(z_1)] + \varepsilon$ where ${z}_j\sim\mathcal{U}(0,2\pi),~j=1,2$ and $\varepsilon\sim \N(0, 0.1^2 \cdot \I_3)$ the same data with a hole, as well as including a uniform ball of points in the center. Then, we trained a VAE per dataset with our proposed prior and we fitted post-hoc an RBF network for each to induce the pull-back metric. In Fig.~\ref{fig:ex2:compare_cases} we show the $d=2$ latent spaces. Also, we define the distance between two curves as $\int_0^1||c_1(t) - c_2(t)||_2^2 dt$, where each curve is parametrized with unit speed under the Euclidean metric. This makes curves coming from different Riemannian metrics as comparable as possible. Thus, we select pairs of points and we compute the distance between the curves that correspond to the pull-back and our proposed conformal metric. For additional details see App.~\ref{app:constructive_examples}.


From the results we observe that the theoretical analysis in Sec.~\ref{sec:our_proposed_metric} is reasonable. In particular, for the \emph{hole} case we see that both metrics behave similarly, since the paths avoid crossing the regions in $\Z$ with zero density. This is useful in practice as the shortest paths are pulled towards the latent codes for both metrics. However, some outliers still exist, which means that the represented geometry is not exactly the same. This is apparent in the \emph{ball} case, where the two metrics have exactly the opposite behavior. The ball data increases the prior in $\Z$, while the uncertainty of $g(\cdot)$ increases in the same region as well. This causes the shortest paths to have a contrastive behavior. While in the \emph{normal} case the two metrics result to similar curves. However, for a pair of points the path of $\bm{M}_\psi(\cdot)$ crosses a region with higher density and the path of $\bm{M}_\theta(\cdot)$ not, but the two curves are still similar. Therefore, if the data lie uniformly near a manifold in $\X$ we expect the metrics to behave similarly, due to the relation of the prior to the uncertainty of $g(\cdot)$.

\subsection{Comparing efficiency and robustness of the metrics}
\label{sec:exp:efficiency_robustness}

We investigate the behavior of the metrics as dimension increases, as well as the influence this has on the computation of shortest paths. We use the MNIST digits 0,1,2 that we project with PCA to 100 dimensions and we train a VAE for each $d=[2,3,5,10]$ using our proposed prior, and also, we train post-hoc the RBF network to induce $\bm{M}_\theta(\cdot)$.  Moreover, to make the metrics comparable we rescale them so that the maximum magnification factor on the latent codes is 1.


\begin{figure}[t]
	\centering
    \begin{subfigure}{0.44\columnwidth}
	    \begin{overpic}[width=1\linewidth]{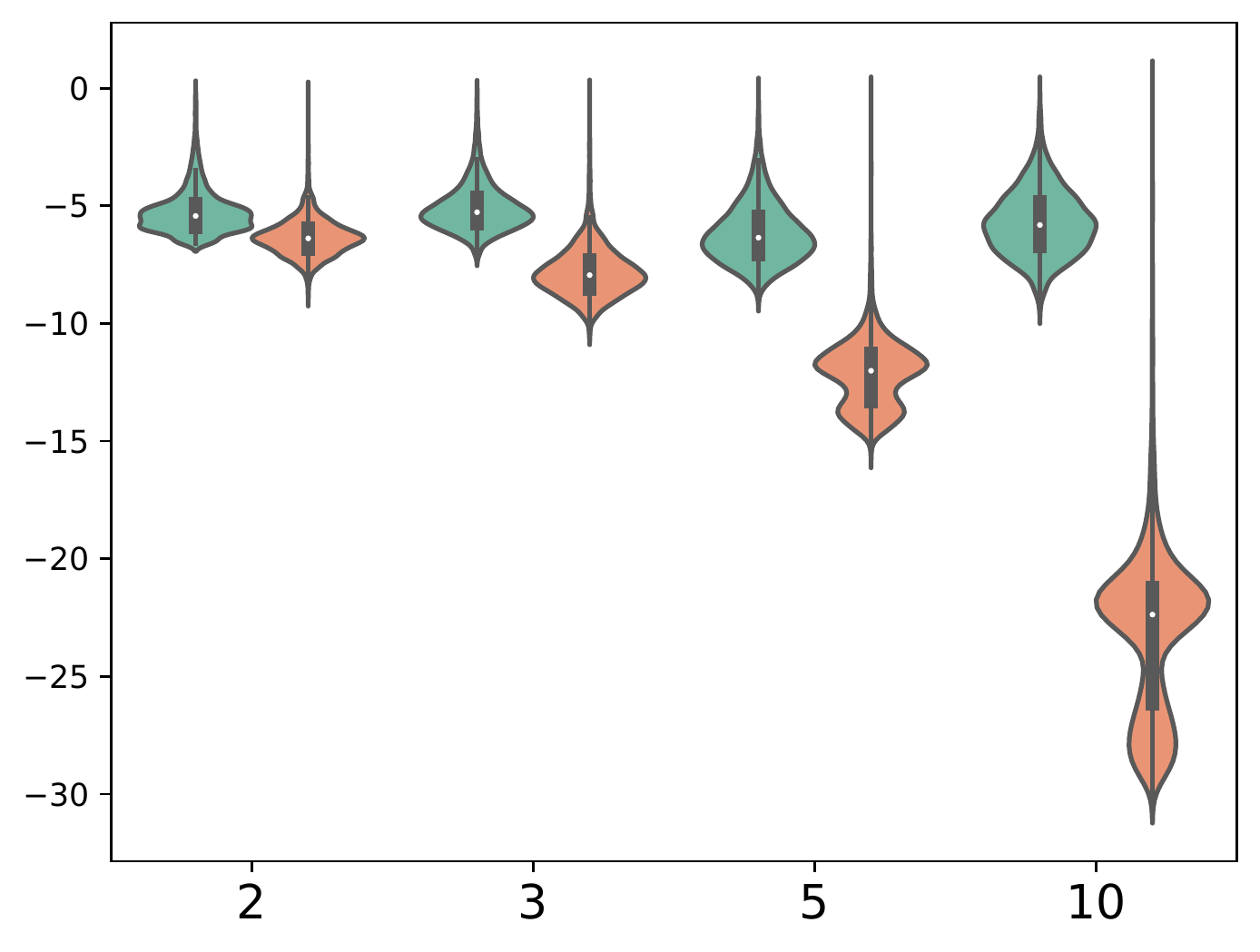}
	    \put(-7, 25){\rotatebox{90}{\tiny $\log[\sqrt{\bm{M}(\cdot)}]$}}
	    \put(50, -2.5){\tiny dim}
	    \put(19,10){{\tiny Our conformal metric}}
	    \put(15,12){{\color[RGB]{97,173,151}\circle*{5}}}
	    \put(19,18){{\tiny Pull-back metric}}
	    \put(15,20){{\color[RGB]{232,140,109}\circle*{5}}}
    \end{overpic}
    \end{subfigure}
    \quad
    \begin{subfigure}{0.44\columnwidth}
	    \begin{overpic}[width=1\linewidth]{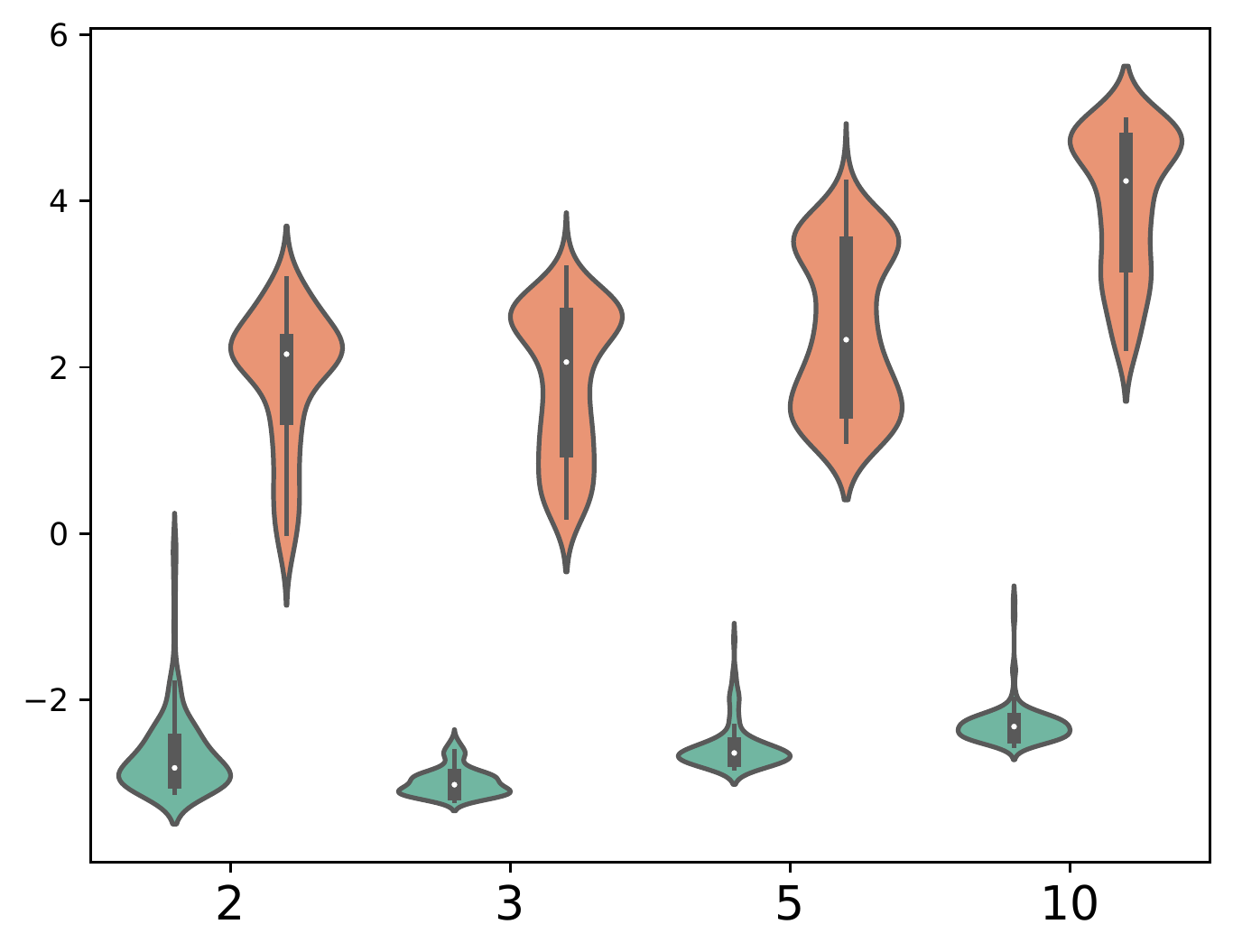}
        \put(-5, 30){\rotatebox{90}{\tiny $\log[$time(sec)$]$}}
        \put(50, -2.5){\tiny dim}
    \end{overpic}
    \end{subfigure}
	\caption{Metric robustness (\emph{left}) and shortest paths efficiency (\emph{right}) in higher dimensions. The magnification factor of $\bm{M}_\psi(\cdot)$ remains stable, while $\bm{M}_\theta(\cdot)$ due to the RBF is not robust. In fact, some latent codes fall far from the RBF centers, so the second term of $\bm{M}_\theta(\cdot)$ becomes large. Also, $\bm{M}_\theta(\cdot)$ results in a complex and unstable ODE system, so the efficiency of the solver is limited and many times fails, while for $\bm{M}_\psi(\cdot)$ the ODE is easier to solve.}
	\label{fig:robust_efficiency}
\end{figure}
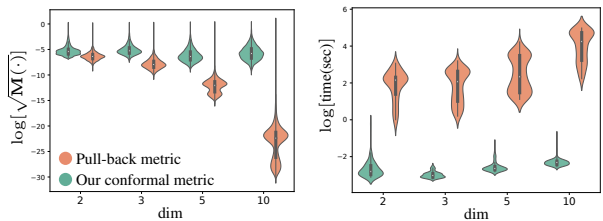

We show in Fig.~\ref{fig:robust_efficiency} the magnification factor computed on the latent codes and we see that $\bm{M}_\psi(\cdot)$ is robust as $d$ increases. This means that the prior behaves consistently i.e., the density on the representations is non-zero and is relatively similar across them. Also, we sample uniformly in the bounding box of the latent codes and the evaluation of the metric shows that indeed it is small only near the representations (see App.~\ref{app:efficiency_robustness}). In contrast, $\bm{M}_\theta(\cdot)$ is not robust because due to the curse of dimensionality the second term that is based on the RBF is inconsistent, which results in very high magnification factor on some of the latent codes.



Additionally, we selected 10 points per cluster and we compute the pairwise distances within each cluster, in order to investigate the influence of the metrics on the shortest paths. The results in Fig.~\ref{fig:robust_efficiency} shows that $\bm{M}_\psi(\cdot)$ is highly efficient when computing shortest paths, even when $d$ increases. The reason is that the corresponding ODE system is simpler, more stable and also easier to solve. While $\bm{M}_\theta(\cdot)$ mainly due to the RBF, results in an unstable ODE systems, as well as, only evaluating the metric and its derivative is significantly more expensive. Consequently, the computation of the paths is very slow, while many times the solver fails $(>25\%)$. Further details for this experiment in App.~\ref{app:efficiency_robustness}.


\begin{figure}[h]
	\centering
	\begin{subfigure}{0.44\columnwidth}
    	{\includegraphics[width=1\linewidth]{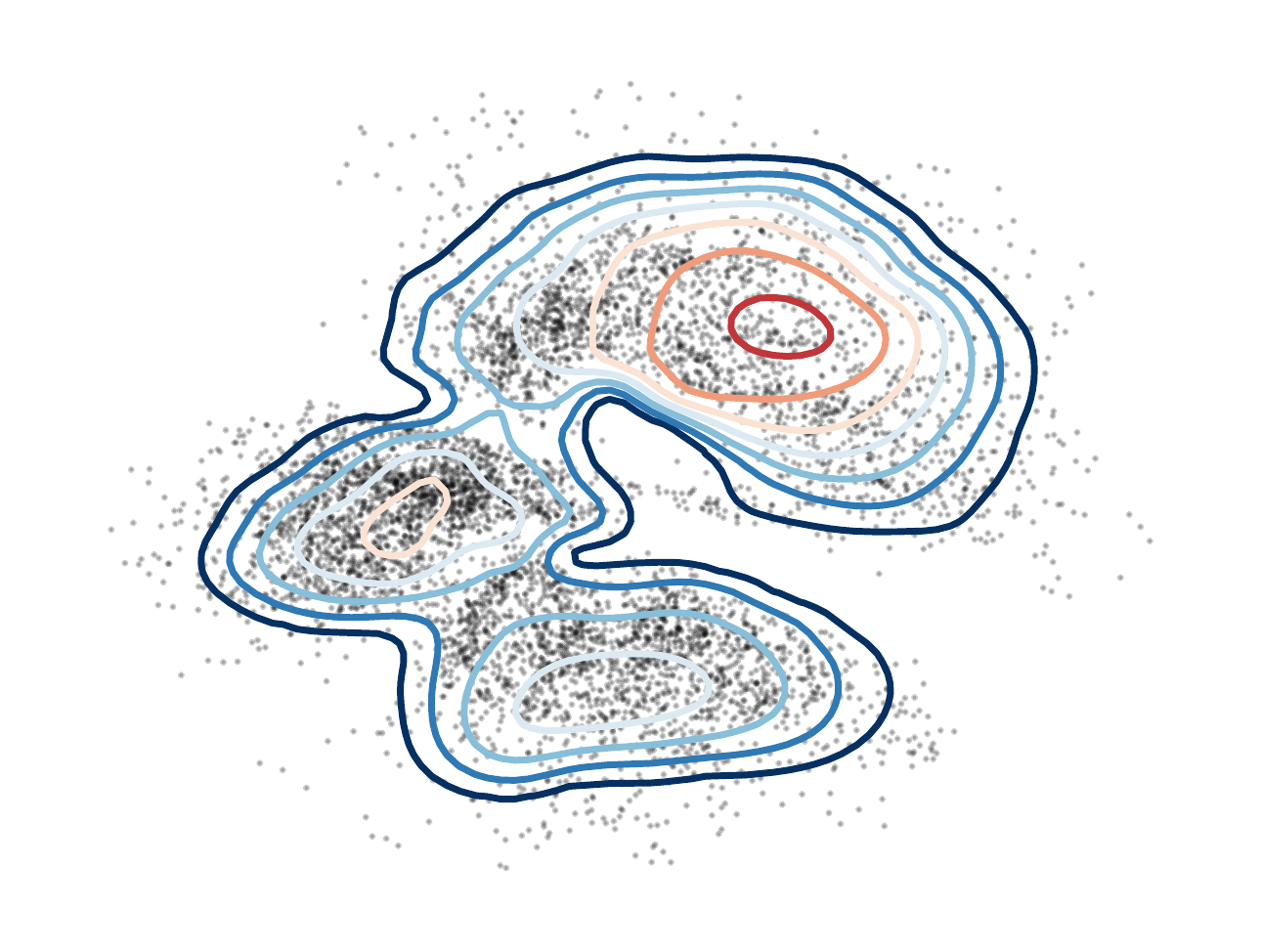}}
	\end{subfigure}
	~
	\begin{subfigure}{0.44\columnwidth}
	    {\includegraphics[width=1\linewidth]{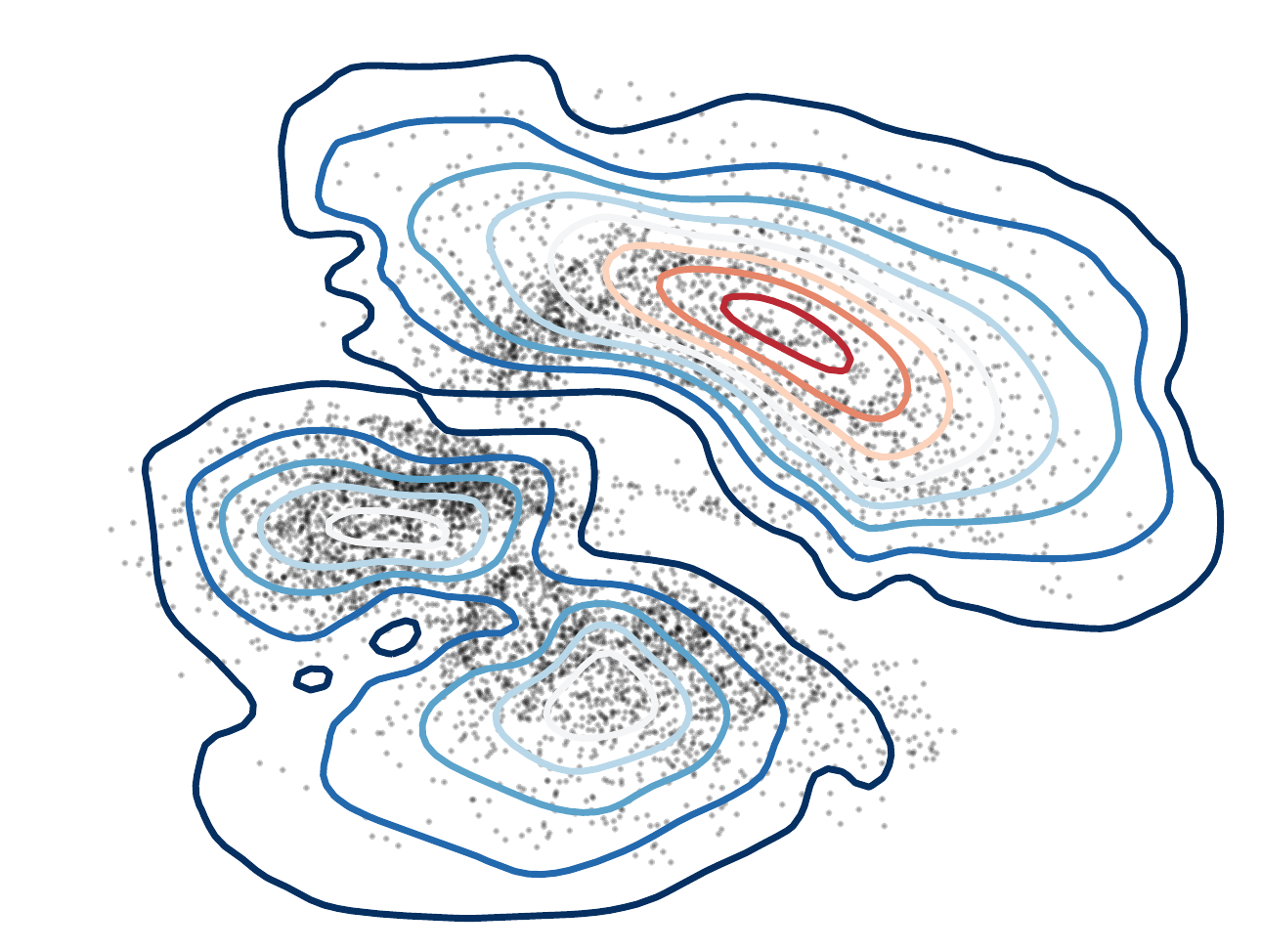}}
	\end{subfigure}
	\caption{LAND mixtures (Sec.~\ref{sec:ex:land_mnist}) in $\Z$ for our $\bm{M}_\psi(\cdot)$ (\emph{left}) and $\bm{M}_\theta(\cdot)$ (\emph{right}). Due to the robustness of our proposed metric the density adapts better to the representations, while the pull-back is high on some boundary points which affects negatively the fitting.}
	\label{fig:ex:land_mnist}
\end{figure}

\subsection{Statistical models on Riemannian manifolds}
\label{sec:ex:land_mnist}

We fit a mixture of locally adaptive normal distributions (LANDs) defined on Riemannian manifolds with  density $\rho(\bm{z}) = C(\mu,\bm{\Gamma}) \cdot \exp(-0.5 \cdot \inner{\logmap{\bm{\mu}}{\bm{z}}}{\bm{\Gamma}\cdot \logmap{\bm{\mu}}{\bm{z}}})$, mean $\mu\in\R^d$,  precision $\bm{\Gamma}\in\R^{d\times d}_+$ and normalization constant $C(\mu,\bm{\Gamma})$ \citep{arvanitidis:neurips:2016}. This is a flexible model but computationally expensive since it is fitted with gradient descent based on $\logmap{\mu}{\cdot}$ and $\expmap{\mu}{\cdot}$. In Fig.~\ref{fig:ex:land_mnist} we show the result on the latent codes of Sec.~\ref{sec:exp:efficiency_robustness}. Due to robustness of $\bm{M}_\psi(\cdot)$ the density adapts better. In contrast, outliers with high $\bm{M}_\theta(\cdot)$ cause underestimated precisions. Also, the running times are respectively 10 min and 2 hours, because the ODE (Eq.~\ref{eq:conformal_metric_ode}) for $\bm{M}_\psi(\cdot)$ is significantly more efficient.

\subsection{Applications in life sciences}
\label{sec:ex:life_data}


We show the usability of the proposed metric in real world problems. Note that our setting is simplified and specialized models for such data exist. For more details see App.~\ref{app:experiments_biology}..


\begin{figure}[h]
	\centering
	\begin{subfigure}{0.44\columnwidth}
	    {\includegraphics[width=1\linewidth]{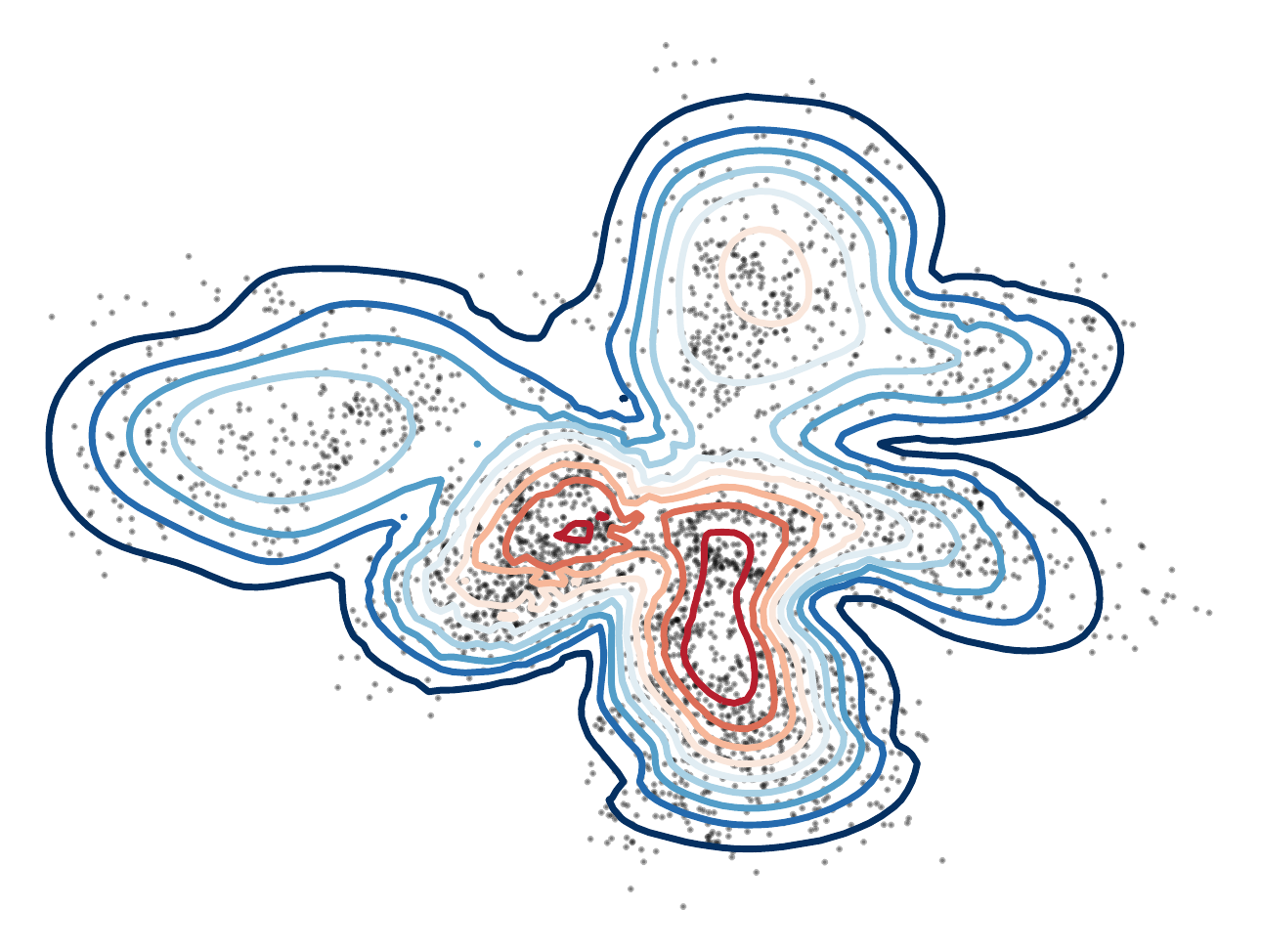}}
	\end{subfigure}
	~ 
	\begin{subfigure}{0.44\columnwidth}
	{\includegraphics[width=1\linewidth]{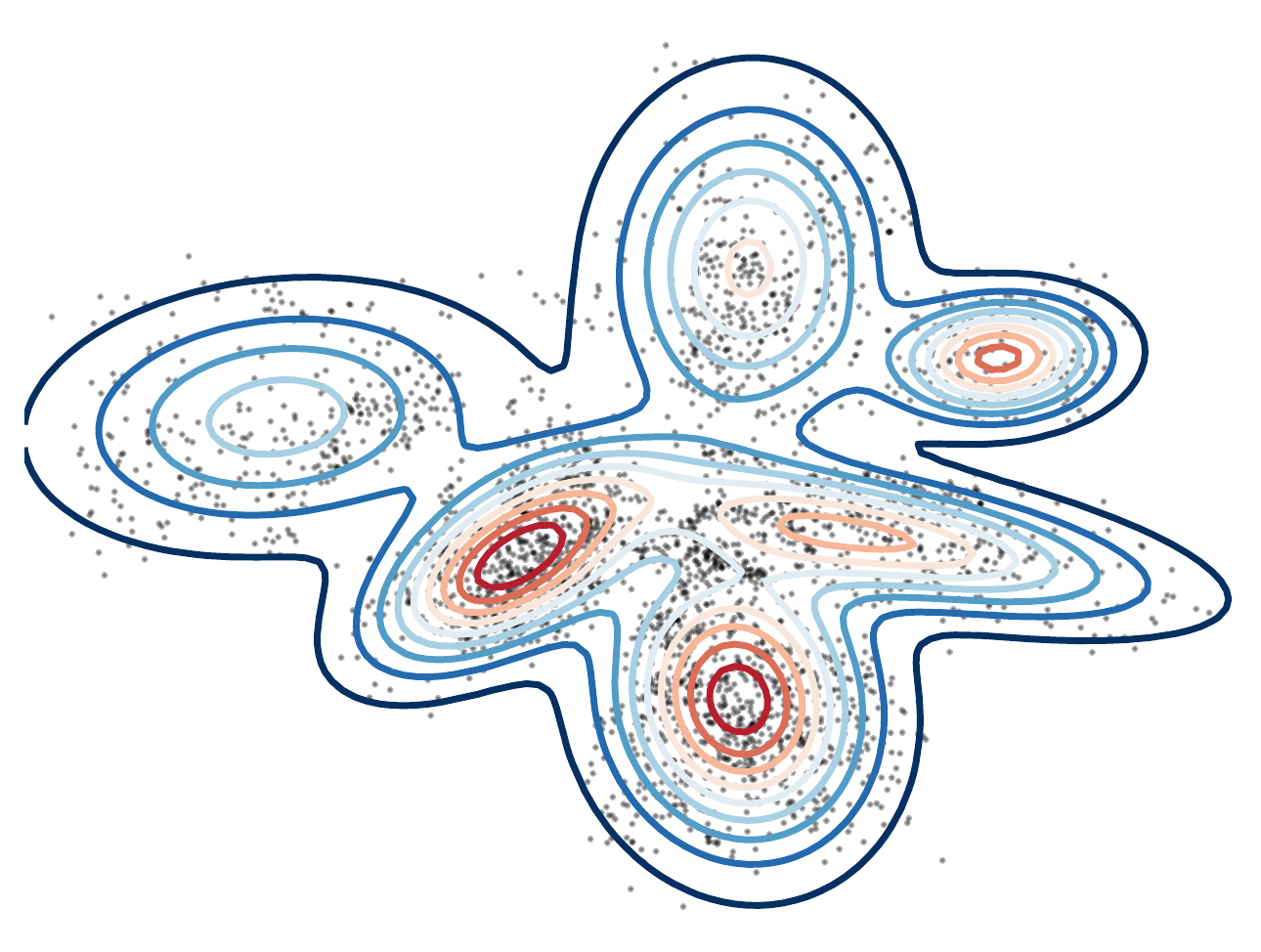}}
	\end{subfigure}
	\caption{A mixture of LANDs and a GMM on the Cortex data.}
	\label{fig:ex:land_cortex}
\end{figure}


We trained a VAE on mouse cortex cell data, which has a natural clustering \citep{zeisel:science:2015}. In Fig.~\ref{fig:ex:land_cortex} we compare a mixture of LANDs with a Gaussian mixture model (GMM), where we see that the LANDs adapts better to the representations, which can be useful for exploratory data analysis by experts (see App.~\ref{app:experiments_biology} for individual components). In addition, we can utilize the principal geodesics as a form of local \emph{disentanglement}, as they represent the directions with highest variance on the data manifold (see App.~\ref{app:experiments_biology}).


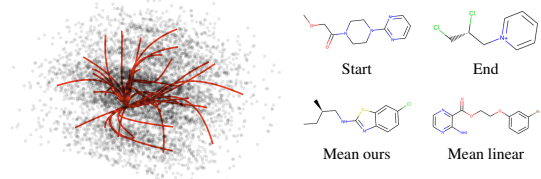
\begin{figure}[h]
\centering
	\begin{subfigure}{0.44\columnwidth}
	    {\includegraphics[width=1\linewidth]{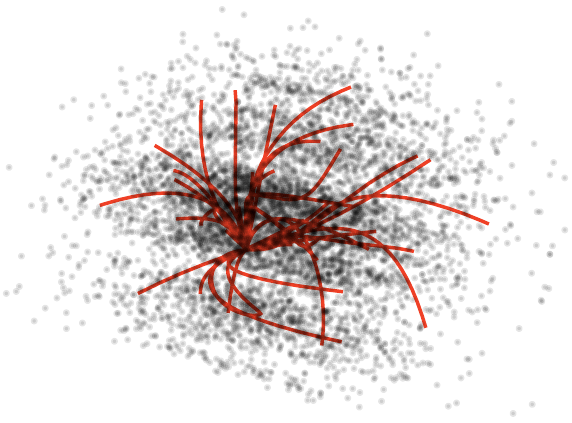}}
	\end{subfigure}
	~
	\begin{subfigure}{0.44\columnwidth}
	\lineskip=0pt
		\begin{tikzpicture}
		    \draw (0, 0) node[inner sep=0] {\includegraphics[width=0.4\columnwidth]{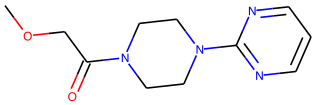}};
	    	\draw (0, -0.5) node {\tiny Start};
    	\end{tikzpicture}
	    ~
	    \begin{tikzpicture}
		    \draw (0, 0) node[inner sep=0] {\includegraphics[width=0.4\columnwidth]{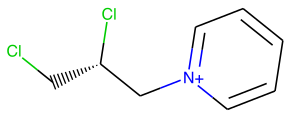}};
	    	\draw (0, -0.5) node {\tiny End};
    	\end{tikzpicture}
    	
    	\vspace{5pt}
    	
    	\begin{tikzpicture}
		    \draw (0, 0) node[inner sep=0] {\includegraphics[width=0.4\columnwidth]{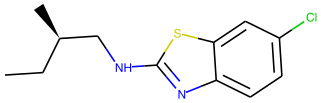}};
	    	\draw (0, -0.5) node {\tiny Mean ours};
    	\end{tikzpicture}
	    ~
	    \begin{tikzpicture}
		    \draw (0, 0) node[inner sep=0] {\includegraphics[width=0.4\columnwidth]{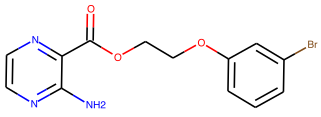}};
	    	\draw (0, -0.5) node {\tiny Mean linear};
    	\end{tikzpicture}
	\end{subfigure}
	\caption{Shortest paths in $\Z=\R^3$ and mean value comparison.}
	\label{fig:chemical_data}
\end{figure}


We used a subset of chemical compounds from the ZINC database \citep{zinc:2015} to train a recurrent VAE. Such data has an inherent natural structure that we capture in $\Z$. We see in Fig.~\ref{fig:chemical_data} that shortest paths respect the learned nonlinear structure (see also App.~\ref{app:experiments_biology}). This amounts to interpretable and meaningful interpolations, which can reveal biological information \citep{detlefsens:arxiv:2020}. As an example, we compare the linear and our shortest path mean.


\section{Conclusion}
\label{sec:conclusion}

We propose to capture the geometry of a data manifold  in the latent space of a generative model using a Riemannian metric that is inversely proportional to a learnable prior. In addition, we propose a suitable energy-based model for the learnable prior in a VAE context. Our analysis shows that the proposed metric is a sensible approximation of the true pull-back metric while being efficient and robust. Apart from its usefulness, our metric provides a way to implicitly take into account the data geometry during training a generative model, using interpretable regularizers in geometric forms.






\clearpage

%


\bibliography{paper_refs}
\bibliographystyle{icml2021}


\clearpage
\appendix

\twocolumn[
\icmltitle{Appendix: A prior-based approximate latent Riemannian metric}
]

\section{Riemannian geometry}
\label{app:riemannian_geometry}

A Riemannian manifold $\M$ is a smooth manifold together with a Riemannian metric that defines a smoothly changing local inner product that acts on the tangent space \citep{lee:2018, docarmo:1992}. The most intuitive way to conceptualize a Riemannian manifold is as a $d$-dimensional hypersurface embedded in a higher-dimensional ambient space $\X=\R^D$. The simplest Riemannian metric in this case is the restriction of the $\I_D$ on each tangent space $\tangent{\bm{x}}{\M}$. Essentially, the tangent space in this case is a $d$-dimensional vector space that touches $\M$ tangentially at $\bm{x}\in\M$. So a tangent vector ${\bm{v}}\in\tangent{\bm{x}}{\M}$ is actually a vector in $\R^D$. By definition, a smooth manifold can be covered by a collection of charts. A chart can be seen as a parametrization of a neighborhood on $\M$ formally written as $\phi_j:\mathcal{U}_j\subset\mathcal{M} \rightarrow \HH_j \subset\R^d$. In other words, a chart gives us $d$-dimensional coordinates that represent the points in a neighborhood $\mathcal{U}_j$. Moreover, on a smooth manifold the charts are diffeomorphisms by definition. 
However, for simplicity, we assume that a global chart $h(\cdot)$ exists, which gives a global paramtrization of the manifold, so we can write that $h:\HH \rightarrow\M$. The space $\HH$ is known as the \emph{intrinsic coordinates}.

Since $h(\cdot)$ is a differomorphism, we know that $\bm{J}_h:\HH \rightarrow \R^{D\times d}$ is full-rank, and hence, we can uniquely map a vector $\overline{\bm{v}} \in \HH$ from the intrinsic coordinates to a tangent vector ${\bm{v}}\in\tangent{\bm{x}}{\M}$ as ${\bm{v}} = \bm{J}_h(\bm{z}) \overline{\bm{v}}$. Therefore, assuming that the Riemannian metric in the ambient space $\X$ is the Euclidean ${\bm{M}}_\X(\cdot)=\I_D$, we can define the inner product in each tangent space as $\inner{{\bm{v}}}{{\bm{v}}}_\bm{x} = \inner{\overline{\bm{v}}}{{\bm{M}_\HH}(\bm{z}) \overline{\bm{v}}}$. Therefore, on the tangent space, which is actually a $d$-dimensional vector space the metric now is ${\bm{M}_\HH}(\bm{z}) = \bm{J}_h(\bm{z})^\intercal \bm{J}_h(\bm{z})$ and also changes for each point $\bm{x}=h(\bm{z})$. Intuitively, on the tangent space we represent a ``linearized'' view of $\M$ with respect to a base point $\bm{x}\in\M$. When working directly in the embedding space, the linear representation $\bm{v}$ is scaled by the metric ${\bm{M}_\X}(\bm{x})$, and equivalently, the $\overline{\bm{v}}$ is scaled by $\bm{M}_{\HH}(\bm{z})$ when working in intrinsic coordinates.

In addition, the metric $\bm{M}_{\HH}(\cdot)$ appears in the intrinsic coordinates $\HH$ and represents the amount of distortion caused to the infinitesimal volume element $d\bm{z}$ when mapped through $h(\cdot)$ on $\M$. Also, due to the fact that the chart is a diffeormorphism, we get that the metric is smooth as it is based on the Jacobian of $h(\cdot)$. Therefore, the embedding of a smooth manifold $\M$ in a higher dimensional ambient space with $\bm{M}_\X(\cdot)$ directly induces a Riemannian metric in the intrinsic coordinates $\HH$. In this work we assume that $\bm{M}_\X(\cdot)=\I_D$. 

This Riemannian metric allows us to compute distances between points on $\M$. Intuitively, it represents the distortions of the infinitesimal distance and volume element. In particular, let a curve $\gamma:[0,1]\rightarrow\M\subset\X$ with $\gamma(0)=\bm{x}$ and $\gamma(1)=\bm{y}$. We can measure the curve length on $\M$ by considering the curve simply lying in $\X$, so we get
\begin{align}
\label{app:eq:curve_length}
    \ell[\gamma] &= \int_0^1 \sqrt{\inner{\dot{\gamma}(t)}{\dot{\gamma}(t)}_{\gamma(t)}} dt = \int_0^1 \sqrt{\inner{\dot{\gamma}(t)}{\I_D \dot{\gamma}(t)}} dt \nonumber \\
    &=\int_0^1 \sqrt{\inner{\dot{c}(t)}{{\bm{M}_\HH}(c(t)) \dot{c}(t)}}dt = \ell[c],
\end{align}
where $\dot{\gamma}(t)=\partial_t \gamma(t) \in \tangent{\gamma(t)}\M$ is the velocity of the curve and $\gamma(t) = h(c(t))$. Here, we assumed that the metric of $\X$ is the Euclidean, however, other meaningful Riemannian metrics could have been use \citep{arvanitidis:arxiv:2020}. This result shows that instead of computing the length of a curve on $\M\subset\X$ we can equivalently compute it in the intrinsic coordinates $\HH$.

Moreover, we can find the shortest path i.e. the curve with minimum length, by optimizing the functional Eq.~\ref{app:eq:curve_length}. However, it is known that the length is parametrization invariant. In other words, we can reparametrize $t$ and get still the same length. Instead, the energy is not invariant under reparametrizations of $t$, and thus, we can find the curve with minimum energy by optimizing the energy functional
\begin{equation}
    \gamma^* = \argmin_{\gamma} \int_0^1 \inner{\dot{\gamma}(t)}{\dot{\gamma}(t)}_{\gamma(t)} dt,
\end{equation}
or equivalently, we can optimize this quantity using the curve $c(t)$ in the intrinsic coordinates $\HH$ using ${\bm{M}_{\HH}}(\cdot)$. In $\HH$ we can apply the Euler-Lagrange equations which gives us a system of second order nonlinear ordinary differential equations \citep{arvanitidis:iclr:2018}
\begin{align}
\label{eq:general_ode}
    \ddot{c}(t) 
    =&-\frac{1}{2}\bm{M}_{\HH}^{-1}(c(t))\Big[2 (\dot{c}(t)^\T \otimes \I_d) \parder{\vectorize{\bm{M}_{\HH}(c(t))}}{c(t)}\dot{c}(t) \nonumber\\
    &- \parder{\vectorize{\bm{M}_{\HH}(c(t))}}{c(t)}^\T (\dot{c}(t) \otimes \dot{c}(t))\Big],
\end{align}
that we need to solve in order to find the curve that minimizes the energy. The resulting curve is a minimizer of the length as well. Here, $\otimes$ is the Kronocker product and $\text{vec}[\cdot]$ stacks the columns of a matrix.

We can find the shortest path by solving the ODE system above as a boundary value problem (BVP) with $c(0)=\bm{z}_\bm{x}$ and $c(1)=\bm{z}_\bm{y}$ the corresponding points in $\HH$ of $\gamma(0)=\bm{x}$ and $\gamma(1)=\bm{y}$. Unfortunately, for general Riemannian manifolds the analytic solution is intractable, and thus, we rely on approximate numerical solutions \citep{arvanitidis:aistats:2019, hennig:aistats:2014, yang:arxiv:2018}. 

In order to perform computations on $\M$ or equivalently in the intrinsic coordinates $\HH$ we use two operators. The logarithmic map $\logmap{\bm{x}}{\bm{y}}={\bm{v}}\in\tangent{\bm{x}}{\M}$ takes two points $\bm{x},\bm{y}\in\M$ and returns a tangent vector on the tangent space of $\bm{x}$. The vector $\bm{v}$ can be seen as the initial velocity of the curve that starts at $\bm{x}$ and on time $t=1$ reaches the point $\bm{y}$. Essentially, since $\tangent{\bm{x}}{\M}$ is a vector space, this operator provides a linear representation of (a neighborhood on) $\M$ with respect to the base point $\bm{x}\in\M$. In practice, we compute the logarithmic map in the intrinsic coordinates by solving the ODE system as a Boundary Value Problem (BVP). The inverse operator is the exponential map that takes a point $\bm{x}\in\M$ and a vector ${\bm{v}}\in\tangent{\bm{x}}{\M}$ and returns a geodesic $\expmap{\bm{x}}{t\cdot{\bm{v}}} = \gamma(t)$ with $\gamma(1)=\bm{y}$. Again, we implement this operator in the intrinsic coordinates $\HH$ by solving the ODE system as an Initial Value Problem (IVP). The length of a tangent vector, as it lies on a tangent space, it is computed under the Riemannian metric and it is by definition $\text{length}[\gamma]=\inner{{\bm{v}}}{{\bm{v}}}_{\bm{x}}=\inner{\overline{\bm{v}}}{{\bm{M}_{\HH}}(\bm{z})\overline{\bm{v}}}=\text{length}[c]$, where $\gamma(t)$ and $c(t)$ the geodesics on $\M$ and $\HH$ respectively. We can rescale or reparametrize the intrinsic vector $\overline{\bm{v}}$ to $\widetilde{\bm{v}}$ such that the metric locally to become $\widetilde{\bm{M}}_{\HH}=\I_d$ so the $\text{length}[c] = \inner{\widetilde{\bm{v}}}{\widetilde{\bm{v}}}$. The new representation $\widetilde{\bm{v}}$ is known as the \emph{normal coordinates}.

For clarification, the tangent vector $\bm{v}\in\tangent{\bm{x}}{\M}$ in the ambient space $\X$ is a vector in $\R^D$ that is tangential to a $d$-dimensional $\M$ at the point $\bm{x}\in\M$. So the tangent space is a hyperplane that touches tangentially $\M$ at the point $\bm{x}$. Hence, $\M$ can be represented linearly on each tangent space, which is a $d$-dimensional vector space. On the other hand, an example of intrinsic coordinates for $\M$ can be see in Fig.~\ref{fig:geodesic_example}. In the intrinsic coordinates $\HH \subseteq\R^d$ the tangent space at a point $\bm{z}\in\HH$ is simply the $\R^d$ centered at $\bm{z}$. So we can linearly represent the intrinsic coordinates with respect to a base point $\bm{z}$ as vectors $\overline{\bm{v}}\in\R^d$ centered at $\bm{z}$. A second interpretation for $\overline{\bm{v}}\in\R^d$ is to be considered as the intrinsic representation of the vector $\bm{v}\in\R^D$ on the $d$-dimensional vector space $\tangent{\bm{x}}{\M}$.

The analysis above shows that essentially the Riemannian metric ${\bm{M}_{\HH}}$ and the intrinsic coordinates $\HH$ are enough in order to compute distances on a manifold $\M$. This further implies that as long as these quantities are given, then $\M$ could even be an abstract manifold. Unfortunately, in the setting where the manifold is implied by data that lie in $\X$, the Riemannian metric is usually unknown. Moreover, a unique chart rarely exists. In this case, we use a trick to capture the geometry of the data manifold.

\medskip

More specifically, let $\Z \subseteq \R^{d'}$ and we learn a function $g:\Z \rightarrow \M\subset\X$ that should be at least twice differentiable and not necessarily a differomorphism. Then, following the previous analysis we can induce a Riemannian metric $\bm{M}_{\Z}:\Z \rightarrow \R^{d'\times d'}_+$. The high level idea is that if a global chart existed and $\Z = \HH$ with  $g(\cdot) \approx h(\cdot)$ then the $\bm{M}_{\Z}(\cdot) \approx {\bm{M}}_{\HH}(\cdot)$. Even if this is rarely the case, the $\bm{M}_{\Z}(\cdot)$ is still able to capture some geometric properties of some regions of $\M$. This is known as the pull-back metric.

Essentially, let $\Z = \R^{d'}$, which is a smooth manifold with a trivial tangent space, and consider a Riemannian metric $\bm{M}_\Z(\cdot)$ therein. Computing curve lengths under this metric transforms $\Z$ into a Riemannian manifold. In some sense, this Riemannian manifold ``imitates'' or ``captures approximately'' the geometry of $\M$. In practice, the metric scales the distances locally in $\Z$, so it changes the way we measure curve lengths therein. However, in the data manifold regime as it has been shown from previous works \citep{arvanitidis:iclr:2018, tosi:uai:2014, hauberg:only:2018, eklund:arxiv:2019} the $g(\cdot)$ should be a stochastic generator in order to capture properly the geometry of $\M$ in a latent space $\Z$.

The proposed conformal metric in this paper is one way to approximate the behavior of the computationally expensive $\bm{M}_{\Z}(\cdot)$, since evaluating and derivating this metric relies on expensive computations. As we showed in the main paper (see Sec.~\ref{sec:new_metric_discussion}), the new metric in many cases is a sensible approximation to the actual pull-back metric. We showed that under some conditions, both metrics locally result to linear shortest paths. Also, we analyzed theoretically the behavior of the two metrics, where we argued that due to their actual definition both metrics induce the same ``topological'' structure in $\Z$. In other words, in both cases the shortest paths are pulled towards the training latent codes. Of course, there are also cases where the two metrics have the exact opposite behavior. However, the formulation of the proposed metric enable us to take it into account during training. Therefore, we can add regularizers to make the two metrics more similar or even to include interepretable inductive biases in the form of geometric regularizers.

\textbf{Identifiability} in our context considers the preservation of the distance measure between points under diffeomorphic reparametrizations of the intrinsic coordinates. In particular, let two functions $g_1:\Z_1\subseteq\R^d\rightarrow\M\subset\X$ and $g_2:\Z_2\subseteq\R^d\rightarrow\M\subset\X$, where $g_2(\cdot) = T\circ g_1(\cdot)$ with $T(\cdot)$ a diffeomorphic transformation. The reparametrization directly implies that for any pair of points $\bm{x}_1, ~\bm{y}_1\in \Z_1$ and the corresponding points $\bm{x}_2,~\bm{y}_2\in \Z_2$ the Euclidean distance in general is $||\bm{x}_1 - \bm{y}_1||_2 \neq ||\bm{x}_2 - \bm{y}_2||_2$. However, the curve length on the manifold $\M$ between $\bm{x}=g_1(\bm{x}_1)=g_2(\bm{x}_2)$ and $\bm{y}=g_1(\bm{y}_1)=g_2(\bm{y}_2)$ does not change. Note that when we measure the length of a curve using the pull-back metric in $\Z_1$ or $\Z_2$, then we actually measure the length directly on $\M$. Therefore, if both functions $g_1(\cdot)$ and $g_2(\cdot)$ generate $\M$, then we know that the curve length is the same in both parametrizations when measured under each corresponding pull-back metric. In other words, if for any arbitrary learned parametrization $g_j(\cdot)$ the generated $\M$ remains the same, then the distance measured under the corresponding pull-back is invariant.

\section{Riemannian metrics from data}
\label{app:riemannian_metrics}

There are several ways to construct a Riemannian metric from a given set of observations. Here, we present some methods that have been proposed in the literature.

\citet{hauberg:nips:2012} proposed a Riemannian metric as a weighted sum of a predefined set of metric tensors. In particular, let $\bm{M}_{1:K} \in\R_+^{D\times D}$ a predefined set of positive definite metric tensors centered at points $\bm{x}_{1:K}\in\R^D$. Then, the metric at new points $\bm{x}$ is computed as
\begin{equation}
    \bm{M}(\bm{x}) = \sum_{k=1}^K \widetilde{\bm{w}}_k(\bm{x}) \bm{M}_k,
\end{equation}
where $\bm{w}_k(\bm{x}) = \exp\left(- \frac{|| \bm{x}_k -\bm{x}||_2^2}{2 \sigma^2}\right)$, $\sigma>0$ the bandwidth or support of the kernel and $\widetilde{\bm{w}}_k(\bm{x}) = \frac{\bm{w}_k(\bm{x})}{\sum_{l=1}^K \bm{w}_k(\bm{x})}$. In this case, the predefined metrics can be estimated using additional information e.g. labels. The bandwidth controls how large is the neighborhood from which we consider the predefined metrics. Clearly, it is hard to find the optimal parameter $\sigma$. Especially, when the dimension of the space is high, so the curse of dimensionality influences the kernel's behavior. Finally, one downside of this metric is that as we move away from the training data, the magnification factor does not necessarily increase. Because the normalized weights still select some of the predefined metrics.

In a similar spirit \citet{arvanitidis:neurips:2016} proposed an unsupervised approach to construct a Riemannian metric from data. In particular, the metric is defined as the inverse local diagonal covariance matrix, so the diagonal elements of the metric are computed as
\begin{equation}
    M_{jj}(\bm{x}) = \left[\sum_{n=1}^N \bm{w}_n(\bm{x})(x_{nj} - x_j) + \rho\right]^{-1},
\end{equation}
where $\bm{w}_n(\bm{x}) = \exp\left(- \frac{|| \bm{x}_n -\bm{x}||_2^2}{2 \sigma^2}\right)$, the parameter $\sigma>0$ is again the bandwidth and $\rho>0$ a parameter to upper bound the metric. The influence of $\sigma$ can be explained, in some sense it controls the curvature of the metric i.e., how fast the metric changes. However, again finding the optimal parameter is a challenging task. Regarding $\rho$, it is chosen as a small value such that to pull shortest paths near the data. 

A conformally flat Riemannian metric has been proposed by \citet{arvanitidis:arxiv:2020}. The metric is defined as
\begin{equation}
    \bm{M}(\bm{x}) = (\alpha \cdot r(\bm{x}) + \beta)^{-1} \cdot \I_D,
\end{equation}
where $\alpha, \beta>0$ are parameters to lower and upper bound the metric. Here, the function $r(\cdot)$ is modeled as a positive RBF $r(\bm{x}) = \bm{w}^\intercal \bm{\phi}(\bm{x})$, with $\bm{w}\in\R^{K}_{>0}$ and $\phi_k(\bm{x}) = \exp\left(- \frac{|| \bm{c}_k -\bm{x}||_2^2}{2 \sigma^2}\right)$ for some centers $\bm{c}_k$ near the training data. So the behavior is $r(\bm{x}) \rightarrow 1$ near the training data and $r(\bm{x})\rightarrow 0$ as we move away from them. Again here, the problem is how to find the optimal parameters $\sigma$ as well as the kernel behavior in higher dimensions.

Even if the kernel based Riemannian metrics above are simple and meaningful, their performance is rather limited. The main problem is the selection of the bandwidth $\sigma$, as well as the behavior of the kernel especially in higher dimensions. For this reason, another line of work proposed to learn Riemannian metrics in the latent space of a generative model. This approach allows to reduce the dimensionality of the problem. Even if theoretically the resulting  metrics capture precisely the data manifold's geometry, in practice, they rely on some form of a kernel as well. Moreover, their usability is hindered by inevitable computational complexity.

Here we focus on deep generative models and more specifically on Variational Auto-Encoders \citep{kingma:iclr:2014, rezende:icml:2014}. However, the same analysis can be done in the context of Gaussian Processes with GPLVMs \citep{tosi:uai:2014}. Let a stochastic generator $\bm{x}=g(\bm{z}) = \mu_\theta(\bm{z}) + \sigma_\theta(\bm{z}) \cdot \varepsilon$ with $\varepsilon \sim \mathcal{N}(0, \I_D)$. Obviously, this stochastic function is not differentiable as it is a non-smooth function due to $\varepsilon$. Instead, fixing $\varepsilon$ makes the function smooth. This can be seen as generating a whole surface $g(\Z)\subset\X$ for a fixed noise vector $\varepsilon$, which in expectation results in a distribution of points that converges to the actual generative process. \citet{eklund:arxiv:2019} viewed this step as a random projection of a smooth surface that lies in a higher dimensional space $[\mu_\theta(\bm{z}), \sigma_\theta(\bm{z})] \in\R^{D^2}$ using the projection matrix $\text{block}\_\text{diag}([\I_D, \I_D \cdot \varepsilon])\in\R^{D\times {D^2}}$.

Therefore, we can compute the Jacobian of $g(\cdot)$ and then, compute the expected Riemannian metric in $\Z$ as
\begin{align}
    \bm{M}_\theta(\bm{z}) &= \mathbb{E}_\varepsilon [\bm{J}^\intercal_g(\bm{z}) \bm{J}_g(\bm{z})]\nonumber\\
    &=\bm{J}^\intercal_{\mu_\theta}(\bm{z})\bm{J}_{\mu_\theta}(\bm{z}) + \bm{J}^\intercal_{\sigma_\theta}(\bm{z})\bm{J}_{\sigma_\theta}(\bm{z}),
\end{align}
which is known as the pull-back metric. This an interpretable and meaningful metric, since the second term that is based on the uncertainty makes sure that the shortest paths prefer to stay in regions of the latent space with low uncertainty. To achieve this we need the $\sigma_\theta(\cdot)$ to increase as we move further from the latent codes. The solution proposed by \citet{arvanitidis:iclr:2018} is to utilize a positive RBF to model the precision i.e. the inverse variance. So, as we move further from the latent codes due to the RBF behavior the uncertainty increases. Even if this Riemannian metric seems as a reasonable solution, as we discussed in the main paper (Sec.~\ref{sec:learned_latent_riemannian_manifolds}) it comes with some practical disadvantages.

For this metric, there is a hyperparameter that we need to set. In particular, the $\sigma^2_\theta(\bm{z}) = \xi_\theta(\bm{z})^{-1}$ where
\begin{equation}
\label{eq:prec_rbf_defintion}
    \xi_\theta(\bm{z}) = \bm{W} \cdot 
    \begin{bmatrix}
          \exp(-0.5 \cdot \lambda \cdot ||\bm{z} - \bm{z}_1||^2_2) \\
          \exp(-0.5 \cdot \lambda \cdot ||\bm{z} - \bm{z}_2||^2_2) \\
          \vdots \\
          \exp(-0.5 \cdot \lambda \cdot ||\bm{z} - \bm{z}_K||^2_2)
    \end{bmatrix} + \zeta,
\end{equation}
with $\bm{W}\in\R^{D\times K}_{>0}$ and $\zeta>0$ is the lower bound for the precision or equivalently an upper bound for the uncertainty. Hence, $\zeta$ implicitly influences the metric, as it controls the point where $\bm{J}_{\sigma_\theta}(\cdot)$ becomes nearly constant. Practically, it does not allow the precision to become nearly zero. Before $\bm{J}_{\sigma_\theta}(\cdot)$ becoming nearly constant, the corresponding part of the pull-back metric achieves its maximum value. This sets the boundaries around the latent codes, which represents in some sense the topology of the data manifold in $\Z$. Also, it affects the maximum magnification factor.

As we discussed in the main paper, usually the training of the VAE is done using a deep neural network $\overline{\sigma}^2_\theta(\cdot)$ to model the uncertainty of the generator, and then, we train post-hoc the RBF, as a regression problem or using again the ELBO while fixing the other functions. So a practical way to set $\zeta$ is after the first phase of the training to compute the mean variance $\overline{\sigma}^2_{\text{mean}} =\frac{1}{N \cdot D} \sum_{n=1}^N \sum_{j=1}^D [\overline{\sigma}^2_{\theta}(\bm{z}_n)]_j$ of the training latent codes. Then we can set $\zeta = (\alpha \cdot \overline{\sigma}^2_{\text{mean}})^{-1}$ where $\alpha>0$ is a multiplicative factor e.g. $\alpha = 1000$. This is heuristic way to fix the hyperparameter $\zeta$.

In order the magnification factor to be as comparable as possible across different Riemannian metrics, we propose to scale each metric. As regards the pull-back metric we compute the magnification factor on the training latent codes and we find the maximum $m_{\text{max}} = \argmax_{\bm{z}_n}\sqrt{|\bm{M}_\theta(\bm{z})|}$. Then, we rescale the metric as 
\begin{equation}
\label{app:eq:pull_back_metric_scaling}
    \bm{M}_\theta(\bm{z}) \triangleq \frac{1}{m^{\nicefrac{2}{d}}_{\text{max}}}\left[ \bm{J}^\intercal_{\mu_\theta}(\bm{z})\bm{J}_{\mu_\theta}(\bm{z}) + \bm{J}^\intercal_{\sigma_\theta}(\bm{z})\bm{J}_{\sigma_\theta}(\bm{z})\right],
\end{equation}
which ensures that the maximum magnification factor on the training latent codes is 1. Here, if $m_{\text{max}}$ is a huge value, then the metric is scaled dramatically especially on the rest of the latent codes. This is precisely what we see in Sec.~\ref{sec:exp:efficiency_robustness}.

Given a probability density $\nu_\psi(\cdot)$ function that represents the prior of the training latent codes, we proposed in this paper the conformally flat Riemannian metric
\begin{equation}
	\bm{M}_\psi(\bm{z}) = m(\bm{z}) \cdot \I_d = {(\alpha \cdot \nu_\psi(\bm{z}) + \beta)^{-\nicefrac{2}{d}}} \cdot \I_d,
\end{equation}
where $\alpha,~\beta>0$ two hyperparameter that lower and upper bound the metric respectively.
Interestingly, this metric does not depend in any way on a kernel. Also, the parameters $\psi$ of the metric can be learned during the VAE training. As regards the hyperparameters, we can set $\beta = \nicefrac{1}{m_\text{max}}$ where $m_\text{max}$ is the highest value that $\sqrt{\bm{M}_\psi(\cdot)}$ can get. Typically, the $m_{\text{max}}$ is set to a large value like 100. Then, we find the lowest prior value on the training latent codes ${\nu_{\text{min}}} = \argmin_{\bm{z}_n} \nu_\psi(\bm{z})$ and set the $\alpha = \nicefrac{(1 - \beta)}{{\nu_{\text{min}}}}$. This ensures that the maximum magnification factor on the training latent codes is 1. Since this number is $1\ll m_{\text{max}}$ the shortest paths will be pulled towards the training latent codes.

Clearly, the proposed metric does not depend on a kernel, which makes it robust. Also, does not rely on any Jacobian computation, so it is efficient. While its hyperparameters can be fixed relatively easier, which makes it simple.

\section{Training details for the proposed prior}

Our learnable prior is based on energy-based models \citep{lecun:energy:2006}, and is defined as
\begin{equation}
    \nu_\psi(\bm{z}) = \frac{\exp(f_\psi(\bm{z})) \cdot p(\bm{z})}{\C},
\end{equation}
where $f_\psi:\Z \rightarrow \R$ is a deep neural network and the base prior $p(\bm{z}) = \mathcal{N}(0, \I_d)$. Since $\Z$ is typically a low dimensional space, the normalization constant can be computed using naive Monte Carlo as
\begin{equation}
    \mathcal{C} = \int_\Z \exp(f_\psi(\bm{z})) \cdot p(\bm{z}) d\bm{z} \approx \frac{1}{S} \sum_{s=1}^S \exp(f_\psi(\bm{z}_s)),
\end{equation}
where $\bm{z}_s \sim p(\bm{z})$. In addition, when we train the VAE we include the latent codes of the training batch in the estimation of the normalization constant. This helps to prevent the function $f_\psi(\cdot)$ of getting extreme values. In theory, the samples from $p(\bm{z})$ should be enough such that to regularize the function $f_\psi(\cdot)$. However, especially in higher dimensions the number of samples $S$ might not be large enough, in order to successfully regularize $f_\psi(\cdot)$. For this reason, we included the latent codes of the training batch, which we empirically observed to work well i.e., extreme values of $f_\psi(\cdot)$ and $\mathcal{C}$ does not occur.

Of course, in the training objective of the VAE (see Eq.~\ref{eq:our_new_elbo}), due to the $\log(\cdot)$ that is applied on the constant, we used the log-sum-exp trick in order to stabilize the training. Also, we are able to regularize the prior using temperature. In fact, we can use a temperature parameter $T$ in the exponent $\exp(T\cdot f_\psi(\bm{z}))$, for which a large $T$ gives a more complex prior and a smaller $T$ gives a smoother prior \citep{lecun:energy:2006}. Similarly, we can regularize the prior implicitly by applying standard regularization techniques for the parameters of $f_\psi(\cdot)$ e.g. $L_2$ regularization for the weights.

The actual contribution of the normalization constant is to regularize implicitly the prior to be nearly zero in regions of $\Z$ with no latent codes. This is important, because it ensures that the magnification factor increases as we move further from the latent codes. In other words, this helps to approximate well the geometry (and topology) of the data manifold. However, after the training, the normalization constant does not affect anymore the metric, especially since we set the parameters $\alpha, \beta$ as explained above.

\section{Theoretical analysis of the proposed metric}

In this section we provide the theoretical results discussed in the main paper (see Sec.~\ref{sec:new_metric_discussion}). Apart from the demonstrations in the main paper (see Sec.~\ref{sec:exp:constructive_examples}), we provide additional empirical evaluations of these results in App.~\ref{app:constructive_examples}. 

\begin{proposition}
    \label{app:prop_straight_lines_prop}
    Let $\bm{M}_\theta(\cdot)$ and $\bm{M}_\psi(\cdot)$ the Riemannian metrics over the latent space $\Z$. We consider a neighborhood $\mathcal{U}$ of the data manifold $\M\subset\X$ and based on the manifold hypothesis, we assume that the data lie uniformly around $\mathcal{U}$. Let us suppose that in the corresponding region in $\Z$:
    \begin{enumerate}
        \item The density $\nu_\psi(\cdot)$ is approximately uniform.
        \item The generator's uncertainty $\sigma_\theta(\cdot)$ is approximately constant and in addition $\mu_\theta(\cdot)$ has low curvature.
    \end{enumerate}
Then for both the conformal $\bm{M}_\psi(\cdot)$ and the pull-back metric $\bm{M}_\theta(\cdot)$ the shortest paths are approximately straight lines.
\end{proposition}
\begin{proof}
    From condition 1 we know that the $m(\bm{z})$ should be approximately constant in the corresponding region of $\Z$. More precisely, there exists a sufficiently small number $\epsilon > 0 $ so the $\| \nabla m(\bm{z}) \| \leq \epsilon $ in this region. Therefore, the ODE system (see Eq.~\ref{eq:conformal_metric_ode}) becomes a small perturbation of $\ddot{c}(t)=0$ and by basic ODE theory the solutions of the two systems are close (e.g. in $C^2$-norm). Moreover, the solution to the latter differential equation is the straight line. Note that we solve this ODE as a BVP problem with $c(0)$ and $c(1)$ the given boundary conditions.
    
    
    From condition 2 we know that $\sigma_\theta(\cdot)$ is approximately constant, which implies that the $\bm{J}_{\sigma_\theta}(\cdot)$ goes to $\bm{0}_{D\times d}$. The low curvature of $\mu_\theta(\cdot)$ implies that the Jacobian $\bm{J}_{\mu_\theta}(\cdot)$ is approximately constant. For example, if the map locally is linear, then it has zero curvature and the Jacobian is constant. In the general ODE system (see Eq.~\ref{eq:general_ode}) we need the derivative of the metric. In this case, this quantity will be approximately zero, since the low curvature implies that the metric will not change locally. Therefore, as above the ODE system becomes $\ddot{c}(t)=0$, where again the solution is the straight line.
\end{proof}


We are able to provide a more general result that relates any pull-back metric $\bm{M}_\theta(\cdot)$ with a conformal metric $\bm{M}_\psi(\cdot)$. Briefly, in a local neighborhood let us consider a bounded pull-back metric with small (Riemannian) curvature, i.e. the metric is almost locally isometric to a flat Euclidean space, and the corresponding volume form is tightly controlled. Then, we can reparametrize this neighborhood so that $\bm{M}_\theta(\cdot)$ becomes conformally flat, and in addition, we can rescale it such that the magnification factor becomes equal to $\sqrt{|\bm{M}_\psi(\cdot)|}$.

\begin{proposition}
\label{app:prop:reparametrization_straight_lines}
    Let $\bm{M}_\theta(\cdot)$ be the Riemannian metric over the latent space $\Z$ as above and let us suppose that:
    \begin{enumerate}
		\item The curvature tensor $R(\cdot)$ associated to $\bm{M}_\theta(\cdot)$ satisfies $\|R(\cdot)\|_\infty \leq \kappa$ for a sufficiently small positive $\kappa$;
		\item There exist constants $m_1, m_2$ with $|m_2 - m_1|$ sufficiently small, so that the volume element $\sqrt{|\bm{M}_\theta(\cdot)|}$ satisfies $ m_1 \leq \sqrt{|\bm{M}_\theta(\bm{z})|} \leq m_2 $ for each $\bm{z}$ in the chart.
	\end{enumerate}
	Then there exists a reparametrization $\tilde{\bm{z}}$ equipped with the conformal metric $\bm{M}_\psi(\tilde{\bm{z}})$ (given by a pointwise-diagonal matrix with equal entries whose volume form agrees with $\bm{M}_\theta(\bm{z})$ at corresponding points as above), so that the geodesics are approximately given by straight lines and locally the volume of the geodesic balls $B_\rho(\bm{z})$ and $B_\rho(\tilde{\bm{z}})$ are the same.
\end{proposition}

\begin{proof}
    Let us suppose w.l.o.g. that the origin $0 \in \mathbb{R}^n$ is contained in the chart and let us consider a normal coordinate neighbourhood $(B_r(0), \overline{\bm{z}})$ around $0$. It is well-known that in these coordinates the metric tensor $\overline{\bm{M}}(\Bar{\bm{z}})$ at $0$ is Euclidean and the following expansion of the metric holds:
    \begin{equation}
		\overline{M}_{ij}(\overline{\bm{z}}) = \delta_{ij} - \frac{1}{3} \overline{R}_{ikjl}(0) \overline{z}_k \overline{z}_l + O(|\overline{\bm{z}}|^3),
	\end{equation}
    where $\overline{R}$ denotes the respective curvature tensor. In fact, using Jacobi fields one can obtain higher-order expansions whose coefficients are again given by expressions of the curvature tensor. Hence, for every positive number $\epsilon$, if $\kappa$ is small enough, then $|\overline{M}_{ij}(\overline{\bm{z}}) - \delta_{ij}| \leq \epsilon$ for each $\overline{\bm{z}} \in B_r(0)$. Moreover, a similar expansion holds for the Christoffel symbols of the Levi-Civita connection induced by $\overline{M}_{ij}$:
    \begin{equation}
        \Gamma^{i}_{jk}(\overline{\bm{z}}) = -\frac{1}{3} \left( \overline{R}_{ikjl}(0) - \overline{R}_{ijkl}(0)  \right) \overline{z}_l + O(|\overline{\bm{z}}|^2),
    \end{equation}
    hence, for eventually choosing a smaller $\kappa$ one has $|\Gamma^i_{jk}(\overline{\bm{z}})| \leq \epsilon$. Now considering the geodesic equations for $\overline{\bm{M}}(\overline{\bm{z}})$ and the flat Euclidean metric at $\bm{M}_e(\overline{\bm{z}})$, by basic perturbation theory for ODEs, one sees that the solutions (i.e. the geodesic curves) within $B_r(0)$ satisfy:
    \begin{equation}
        \| \gamma_{\bm{M}_e} - \gamma_{\overline{\bm{M}}} \|_{C^\infty} \leq \epsilon, 
    \end{equation}
    if $\kappa $ is chosen small enough (shrunk further). This implies that the geodesic distances induced by the Euclidean metric $\bm{M}_e(\overline{\bm{z}})$ and $\overline{\bm{M}}(\bm{z})$ are similar i.e., straight lines. 
    
    Finally, let us pointwise rescale the coordinates $\overline{\bm{z}}$ to $\tilde{\bm{z}}$ so that the metric $\bm{M}_e(\overline{\bm{z}})$ agrees with the conformal metric $\bm{M}_\psi(\tilde{\bm{z}})$ - in particular, the volume forms at $\tilde{\bm{z}}$ and $\bm{z}$ agree (since we assumed that $|m_2 - m_1|$ is sufficiently small, the rescale is essentially given by constant multiplication). Here $(\tilde{\bm{z}}, \bm{M}_\psi(\tilde{\bm{z}}))$ is the conformal metric as defined above. Moreover, by construction the volumes of the geodesic balls $B_\rho(\tilde{\bm{z}})$ and $B_\rho(\bm{z})$ agree.
\end{proof}

Obviously, the Prop.~\ref{app:prop:reparametrization_straight_lines} implies that if $\Z$ is the latent space where the pull-back metric is defined, if we reparametrize it to get an equivalent conformally flat metric we work over a new space $\Z'$. This is not easily applicable and useful in our setting, since we are interested to compute shortest paths directly in the latent space $\Z$ of the generative model. However, when Prop.~\ref{app:prop_straight_lines_prop} applies, then the reparametrization from $\Z$ to $\Z'$ is actually the indentity map. Also, we are able to rescale the magnification factors such that the maximum value on the training latent codes is the same (see App.~\ref{app:riemannian_metrics}). This result implies that shortest paths are approximately straight lines and with equal curve length under each corresponding metric. For a demonstration see App.~\ref{app:constructive_examples}.

\section{Experimental details}
\label{app:experiments}

In this section, we give additional details about our experimental setting and our implementations. The source code can be found \url{here}\footnote{Source code publicly available in Github upon acceptance}.

\subsection{Shortest path solver}
\label{app:shortest_path_solver}

One of the main tools that we use in our experiments is the solver for ODE system. Specialized approximate numerical solvers have been proposed \citep{hennig:aistats:2014, arvanitidis:aistats:2019, yang:arxiv:2018} mainly for efficiency. However, usually the off-the-self numerical solvers provide more accurate solutions, especially logarithmic maps, and for this reason our approach is based on the SciPy's BVP solver. Commonly, methods of this type implement a version of Newton's method, and thus, convergence heavily relies on the initial solution. For this reason, we use a heuristic graph based solver, in order to provide a curve to the BVP solver as an initial solution.

First, we construct a $k$-NN graph in the latent space $\Z$ by using the Euclidean distances to find the neighbors. Once we construct the graph, we assign as edge weights the length of the straight line computed under the Riemannian metric. Essentially, a large edge weight informs us that this is not a good connection. For two test points, we find their $k$-NN neighbors on the graph again using the Euclidean metric first, and then, we update the weights of the edges using the Riemannian metric. This step enables us to use as the starting and ending point, the nodes of the graph that are closer to the test points. Then, we can find the discrete shortest path on the weighted graph using Dijkstra's algorithm.

This returns a sequence of points, but for which the starting and ending points are not the test points. So once we have the sequence, we replace the two points on the edges with the test points. In order to smooth the final path we apply a filter. We update each point except the boundary ones as $\bm{p}_i = (\bm{p}_{i-1} + \bm{p}_{i} + \bm{p}_{i+1})/{3}$. Of course, we are able to apply more sophisticated filtering techniques. Finally, we use a \emph{cubic spline} to interpolate the filtered points, including the boundary points, which are essentially the test points.

Of course, this is a heuristic solution and does not satisfy the corresponding ODE system. However, in many cases it constitutes a sufficiently good initial solution, which helps the BVP solver to converge. So for our solver, if we do not have a previous solution computed, we initialize the solver with the heuristic graph based curve. If the BVP solver fails, then we return the graph based curve as our solution. However, in cases where the logarithmic map is necessary (as the LAND optimization) and the BVP solver fails, we exclude the point from the current step of the algorithm. The reason is that the logarithmic map of the graph based curve is rather arbitrary, and thus, we do not use it.

\subsection{Details for the comparison of the priors}
\label{app:prior_comparisson}

Here we present the details for our generative modeling experiment where we compare three priors in a VAE setting. In particular, we used two settings, a Convolutional VAE and a standard VAE. In Table~\ref{app:tab:cvae_prior} and Table~\ref{app:tab:vae_prior} we present the details of each setting. Note that for the standard VAE we projected the data using PCA in 100 dimensions with whittening, so the given data are in $\X=\R^{100}$. This step keeps $>90\%$ of the data variance and on the same time allows to train an RBF network post-hoc in order to induce the pull-back metric in $\Z$. We used the standard data splitting train/test both for MNIST and FashionMNIST\footnote{https://pytorch.org/docs/stable/torchvision/datasets.html}.

For the RBF network that we used to model $\sigma_\theta(\cdot)$ we use $K=100$ components. The RBF network is fitted post-hoc. In practice, in the first phase of the VAE training we use a deep neural network to model $\overline{\sigma}^2_\theta(\cdot)$ and in the second phase we train individually the RBF's weights i.e. as a regression problem or using the ELBO keeping the rest of the functions fixed. The RBF's centers are trained with $k$-means using the training latent codes. Then, using the points in each cluster we compute the corresponding covariance matrices, and then, for the bandwidth of each kernel we use the minimum variance on the diagonal of each covariance. This approach guarantees that the centers will be near the latent codes and that the bandwidths will be small enough such that to capture precisely the stucture of the latent codes. Note that training the RBF with the VAE implies that the centers are trainable so it could happen that some centers are moved far from the latent codes. Also, it is hard to pre-specify the bandwidth of the kernels, and in addition, if this parameter is trainable usually overestimated bandwidths occur. As we mention in the main paper for the VampPrior we used $K=500$ trainable inducing points.

\begin{figure}[t]
    \centering
    \begin{subfigure}{0.31\linewidth}
        \includegraphics[width=1\linewidth]{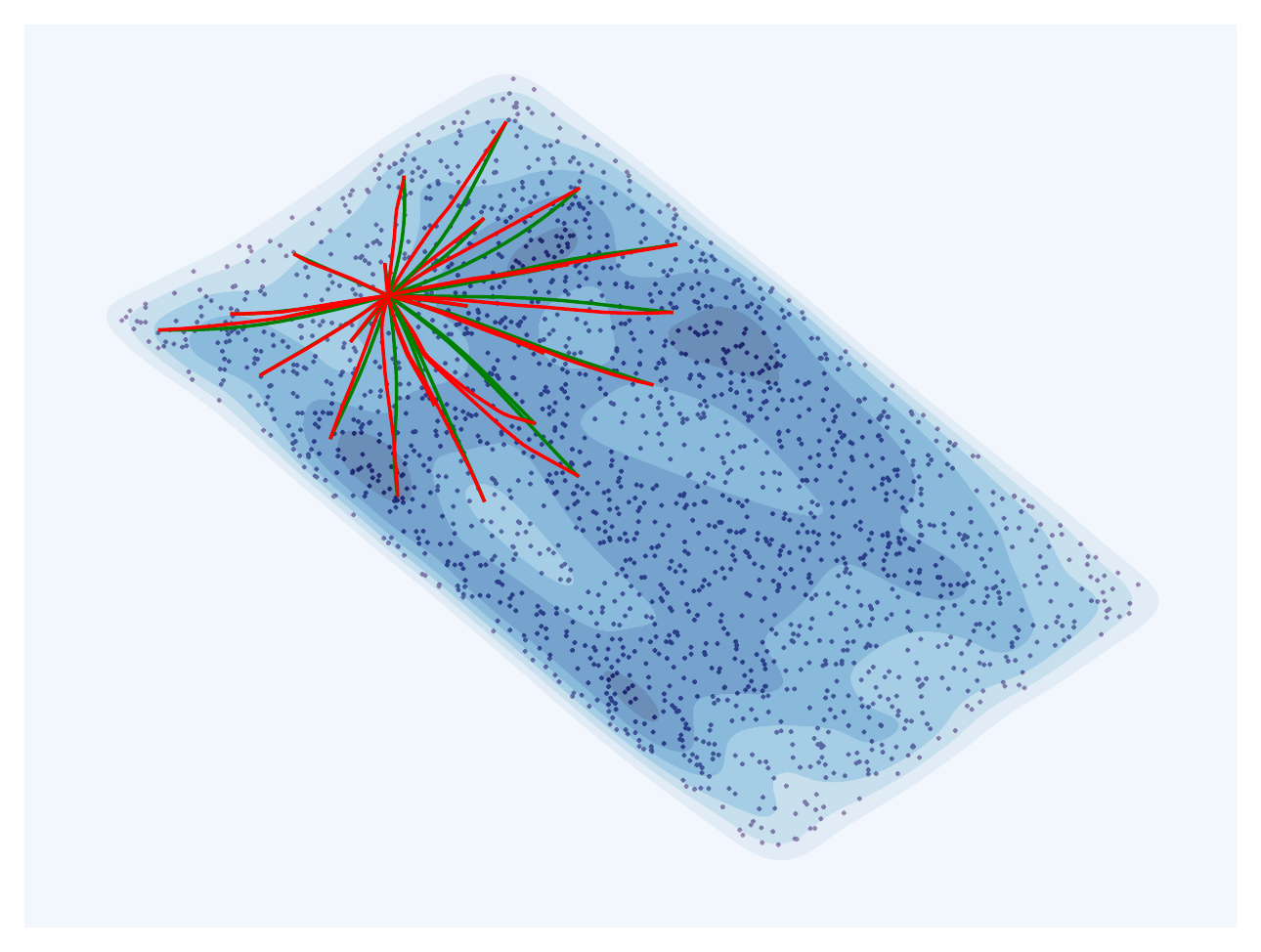}
    \end{subfigure}
    ~
    \begin{subfigure}{0.31\linewidth}
        \includegraphics[width=1\linewidth]{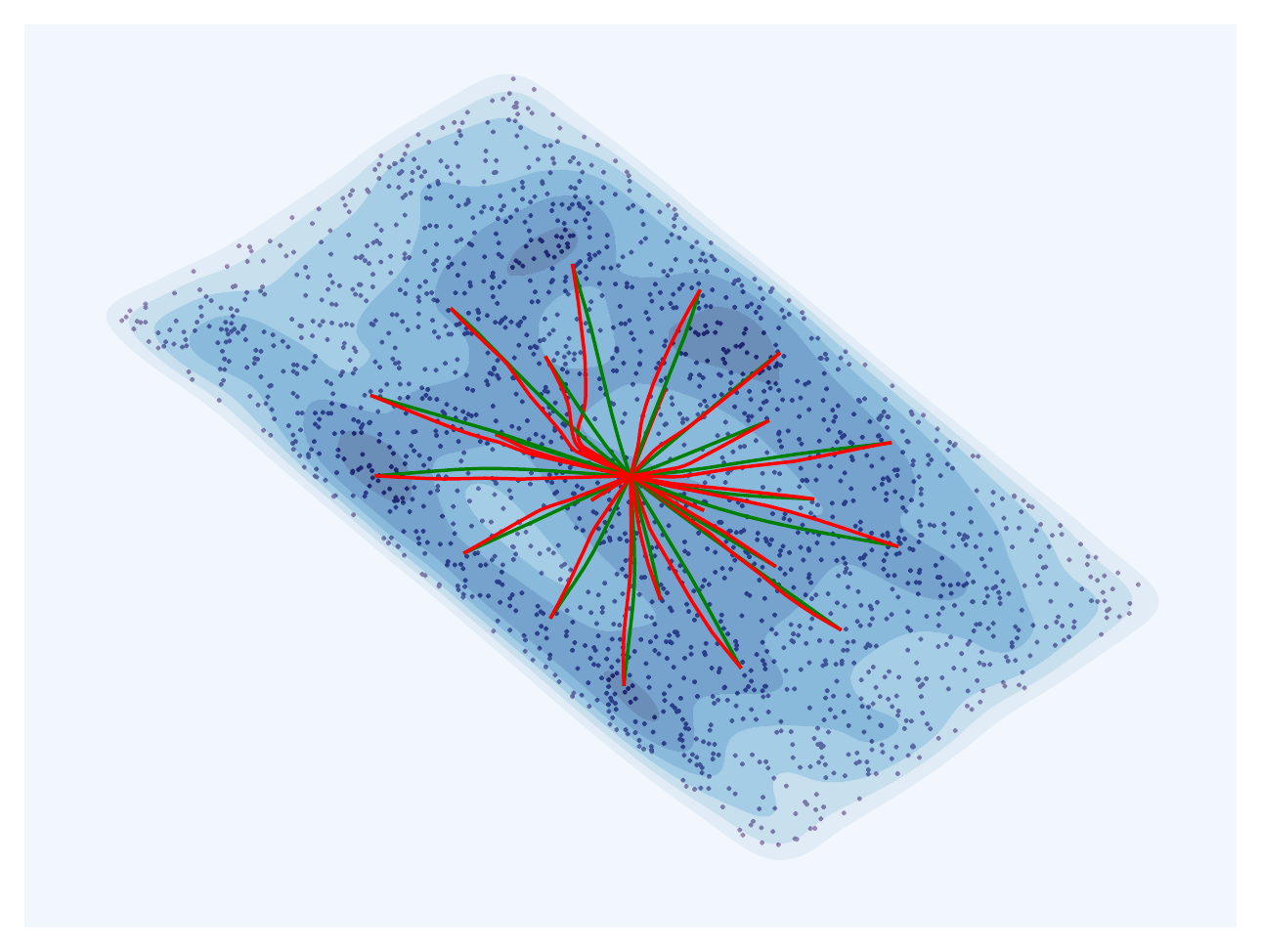}
    \end{subfigure}
    ~
    \begin{subfigure}{0.31\linewidth}
        \includegraphics[width=1\linewidth]{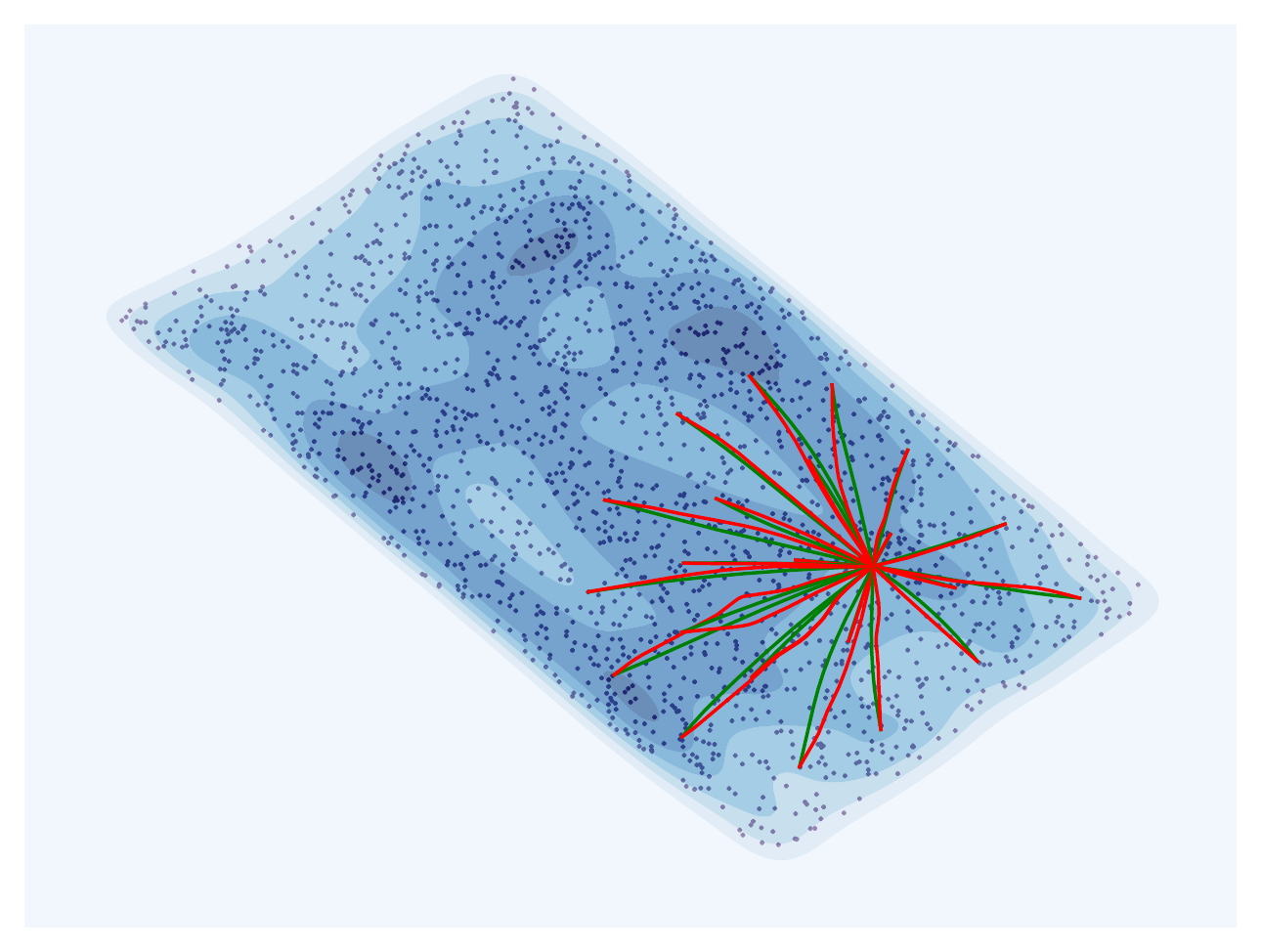}
    \end{subfigure}
    
    \begin{subfigure}{0.31\linewidth}
        \includegraphics[width=1\linewidth]{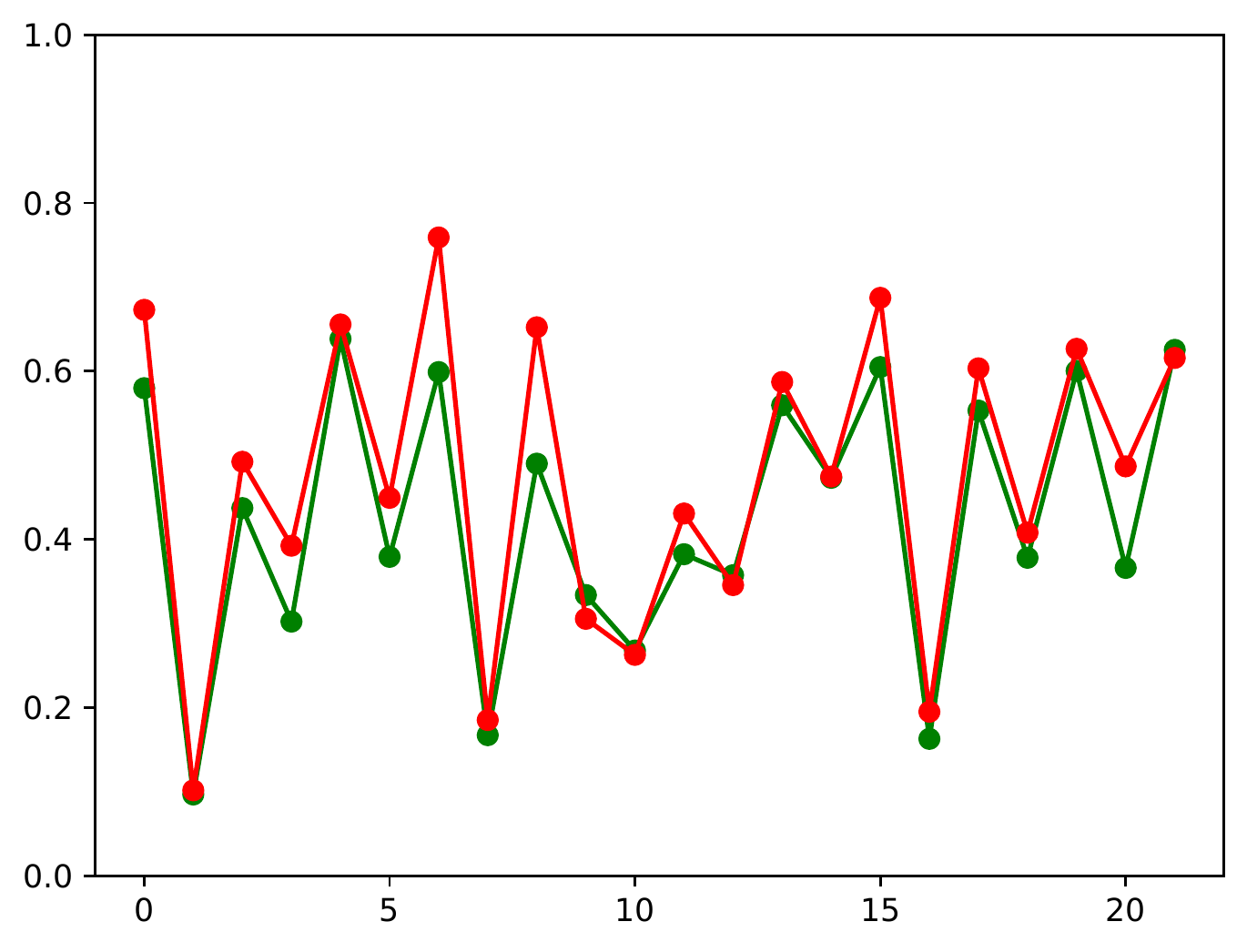}
    \end{subfigure}
    ~
    \begin{subfigure}{0.31\linewidth}
        \includegraphics[width=1\linewidth]{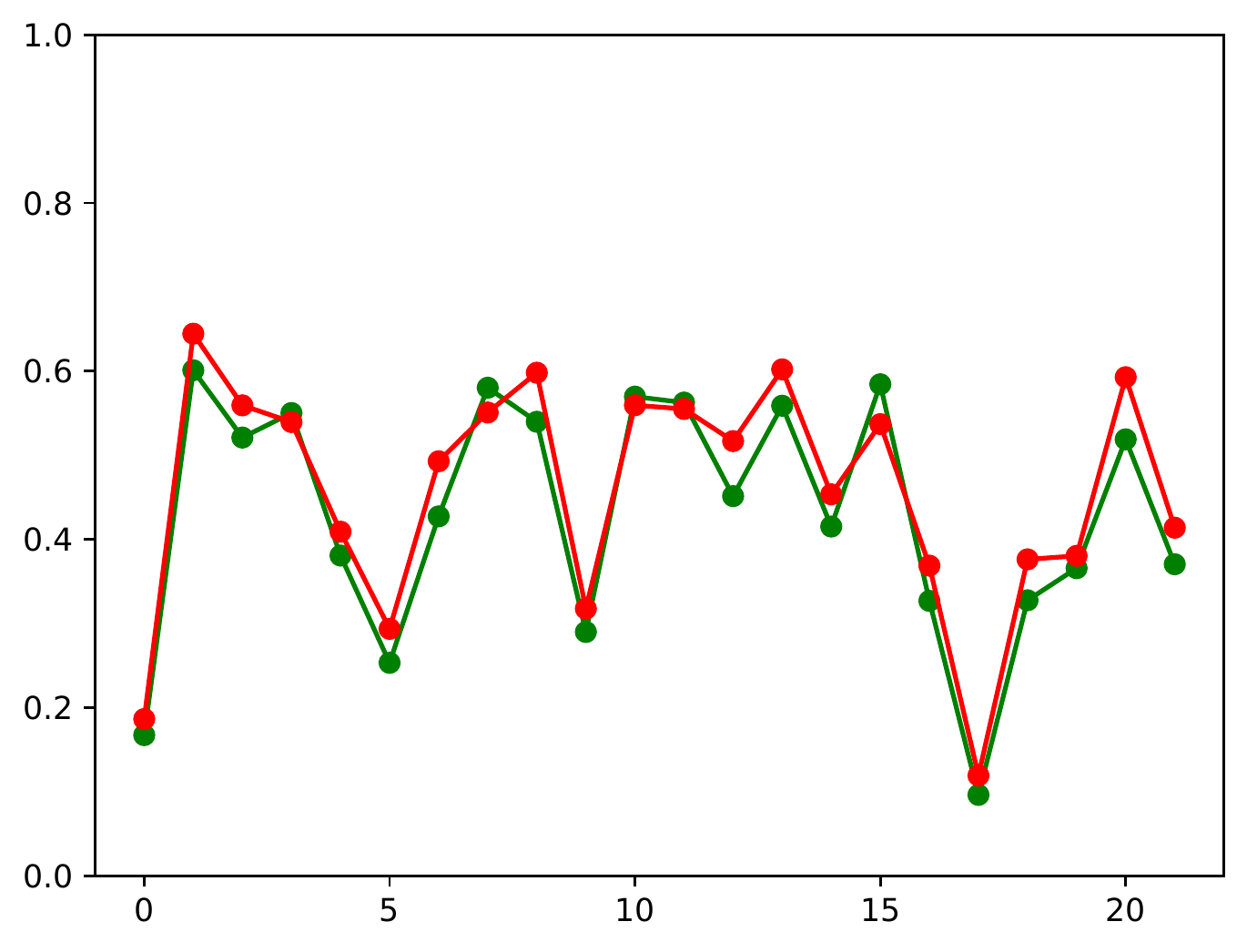}
    \end{subfigure}
    ~
    \begin{subfigure}{0.31\linewidth}
        \includegraphics[width=1\linewidth]{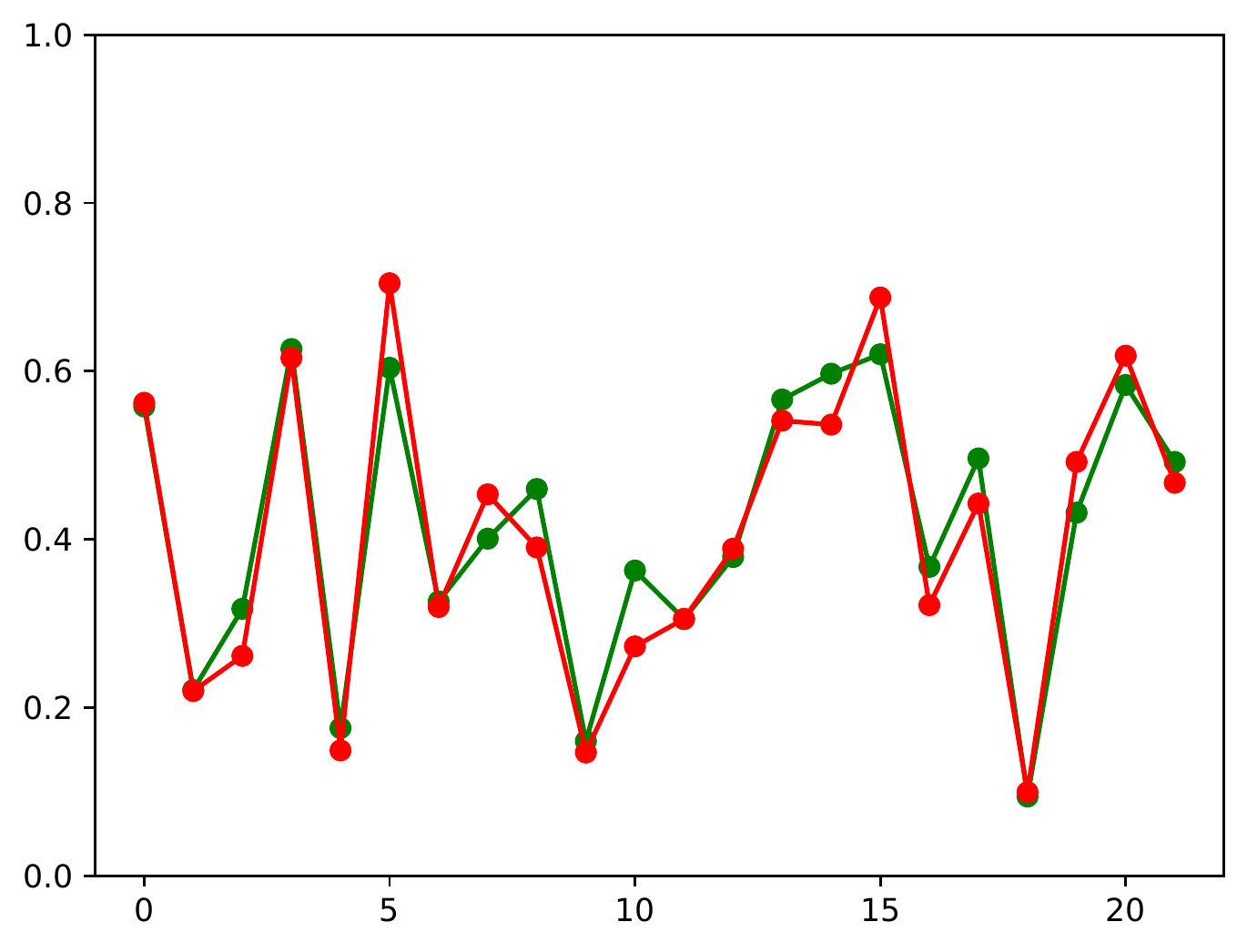}
    \end{subfigure}
    \caption{\emph{Top row}: Comparing the shortest path distance computed under each metric $\bm{M}_\psi(\cdot)$ with green and $\bm{M}_\theta(\cdot)$ with red. When conditions in Prop.~\ref{pror:straight_lines_prop} hold, then shortest paths are approximately straight lines. Since the metric are scaled to be maximum 1 on the training latent codes, the curve lengths are approximately equal. \emph{Bottom row}: Each dot in the graphs correspond to the length of a curve, while the connecting lines only for cleaner illustration.}
    \label{app:constructive_no_curv}
\end{figure}

\begin{table}[b]
    \centering
    \begin{tabular}{c|c c c }
        $\text{encoder}_\phi(\cdot)$ & $3\times3 \times 32$ & $3\times3 \times 32$ & $3\times3 \times 32$ \\
         $\mu_\phi(\cdot)$ & $\text{encoder}_\phi(\cdot)$ & flatten & MLP(512, d) \\
         $\log(\sigma^2_\phi(\cdot))$ & $\text{encoder}_\phi(\cdot)$ & flatten & MLP(512, d)\\
        \midrule
        $\mu_\theta(\cdot)$ & MLP(d, 512) & unflatten & $\text{decoder}_\theta(\cdot)$ \\
        $\text{decoder}_\theta(\cdot)$ & $4\times4 \times 32$ & $4\times4 \times 32$ & $4\times4 \times 32$
    \end{tabular}
    \caption{The convolutional VAE details. Encoder: We used convolutional filters with padding $1$ only in the first two filters, stride $2$ only in the first $2$ filters and $1$ in the final filter. Decoder: We used transposed convolutions with padding $1$ only in the last two filters, stride $2$ only in the last $2$ filters and $1$ in the fist filter. Also, we applied a final convolution filter $3\times3\times1$ with stride $1$ and padding $1$ to provide a smooth output. We used \texttt{tanh} activations both for the convolutional and the MLPs layers.}
    \label{app:tab:cvae_prior}
\end{table}

\begin{table}[b]
    \centering
    \begin{tabular}{c|c c }
        $\text{encoder}_\phi(\cdot)$ & MLP(D, H) & MLP(H, H) \\
         $\mu_\phi(\cdot)$ & $\text{encoder}_\phi(\cdot)$ & MLP(H, d) \\
         $\log(\sigma^2_\phi(\cdot))$ & $\text{encoder}_\phi(\cdot)$ & MLP(H, d)\\
         \midrule
        $\text{decoder}_\theta(\cdot)$ & MLP(d, H) & MLP(H, H) \\
         $\mu_\theta(\cdot)$ & $\text{decoder}_\theta(\cdot)$ & MLP(H, D) \\
         $\log(\sigma^2_\theta(\cdot))$ & $\text{decoder}_\theta(\cdot)$ & MLP(H, D)\\
         
    \end{tabular}
    \caption{The VAE encoder details. We used \texttt{tanh} activations for the MLPs.}
    \label{app:tab:vae_prior}
\end{table}

We trained all the parameters using the Adam optimizer with learning rate $1e^{-3}$. The batch size is 128. The number of epochs 500. For the normalization constant estimation in our proposed prior we used $10\cdot$batch\_size samples from $p(\bm{z})$ and we further include the batch latent codes as samples. This helps to regularize the behavior of $f_\psi(\cdot)$.

\subsection{Details for the constructive examples}
\label{app:constructive_examples}

\begin{figure}[t]
    \begin{subfigure}{0.31\linewidth}
        \includegraphics[width=1\linewidth]{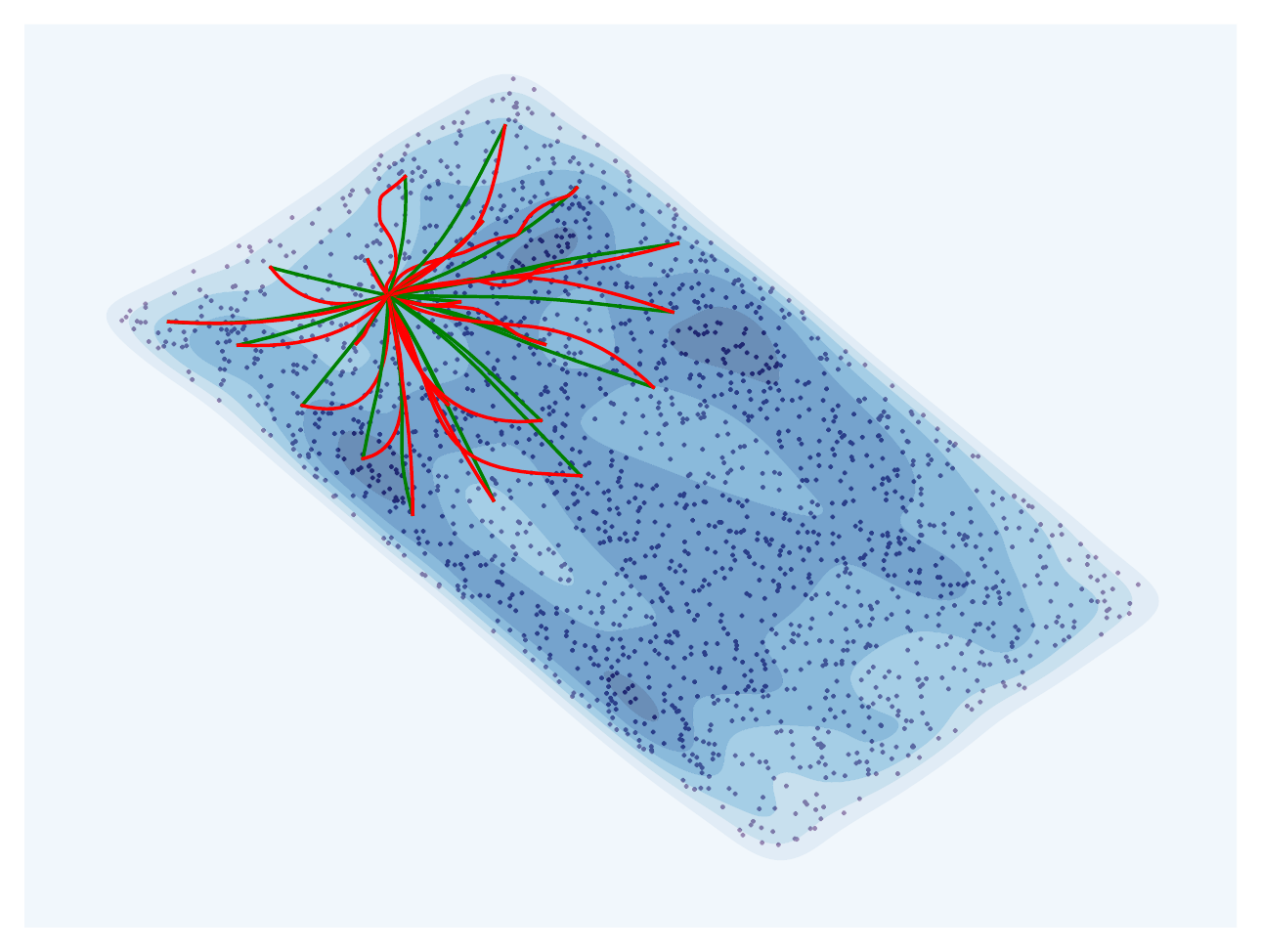}
    \end{subfigure}
    ~
    \begin{subfigure}{0.31\linewidth}
        \includegraphics[width=1\linewidth]{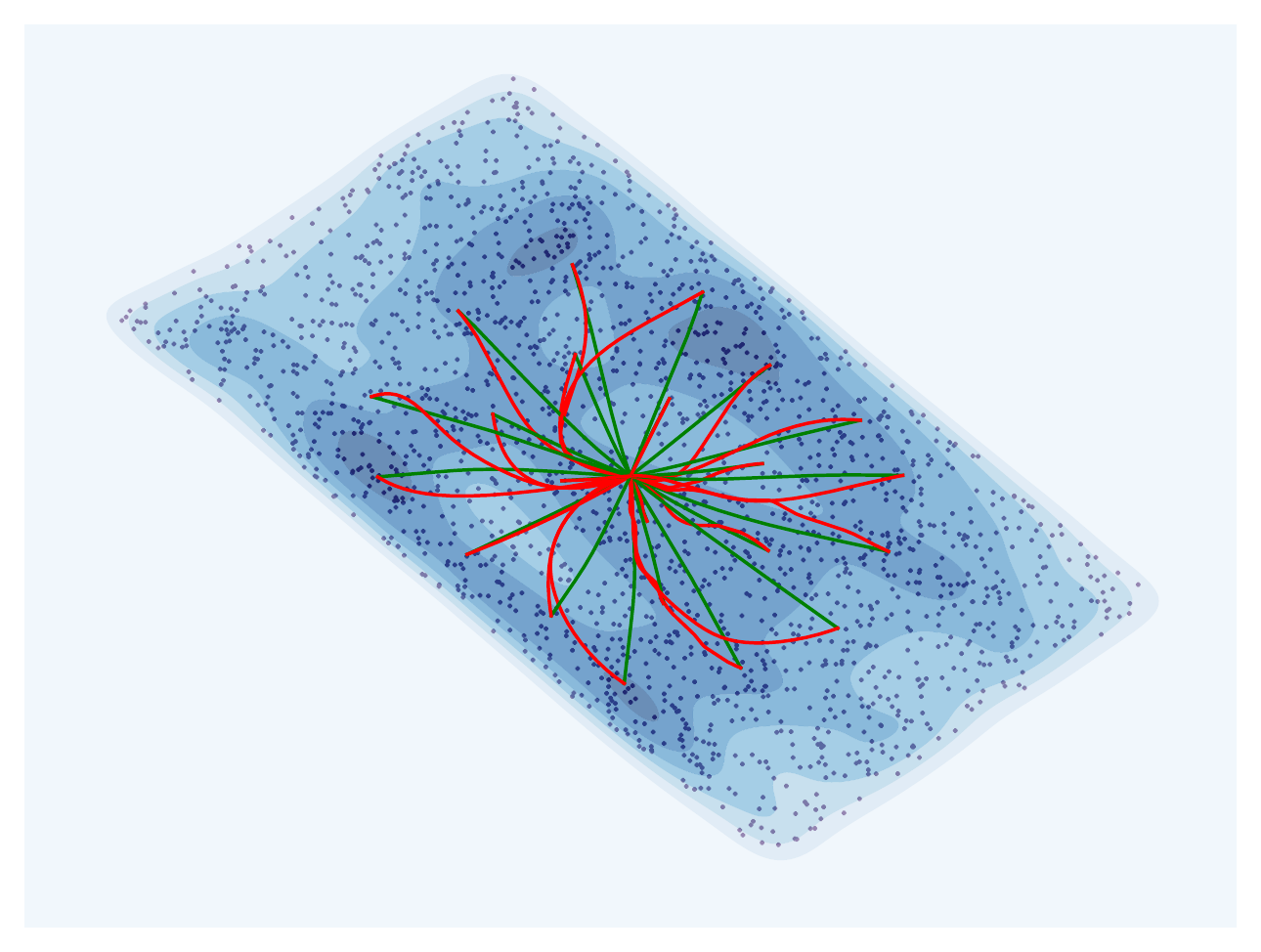}
    \end{subfigure}
    ~
    \begin{subfigure}{0.31\linewidth}
        \includegraphics[width=1\linewidth]{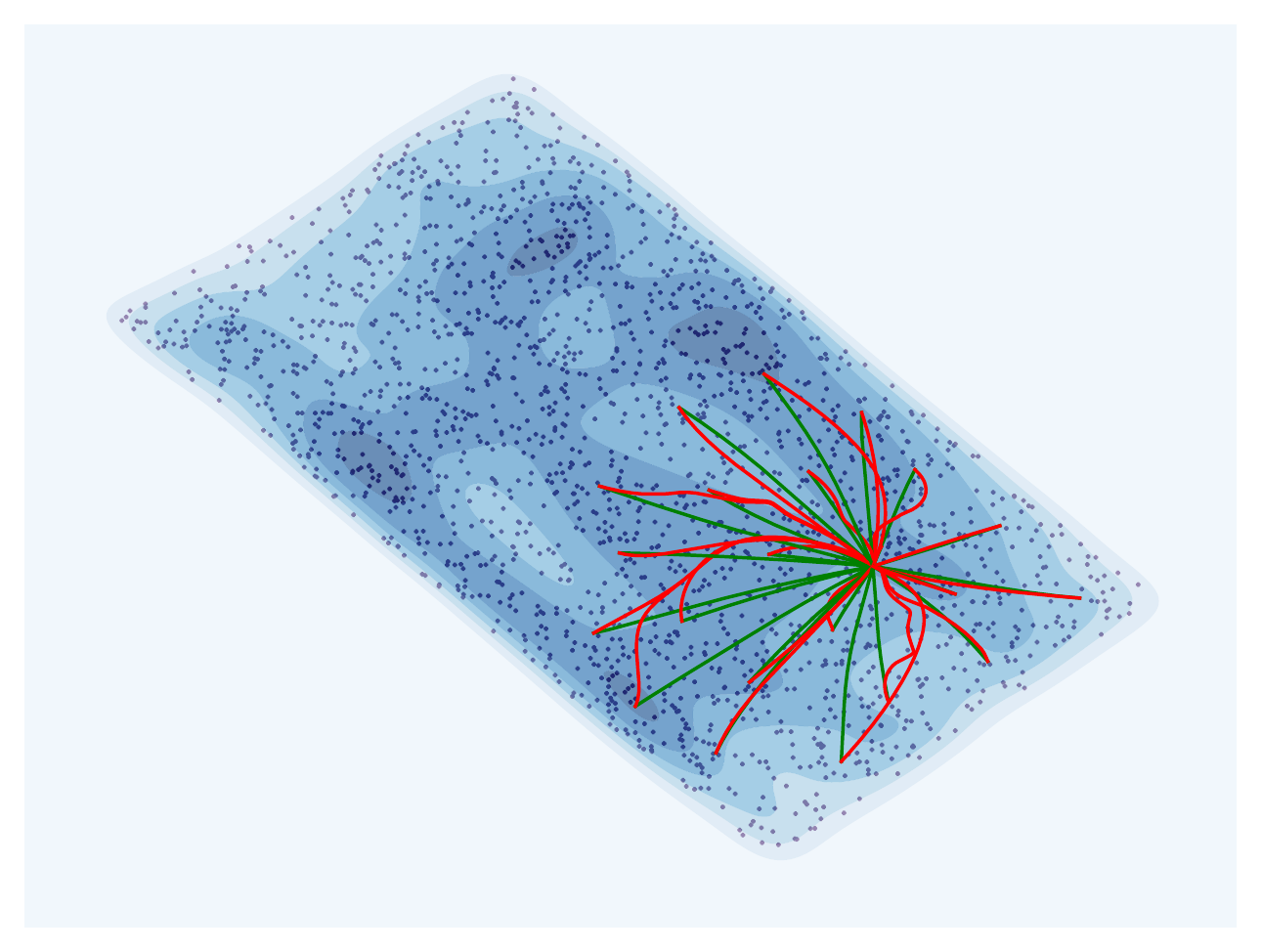}
    \end{subfigure}
    
    \begin{subfigure}{0.32\linewidth}
        \includegraphics[width=1\linewidth]{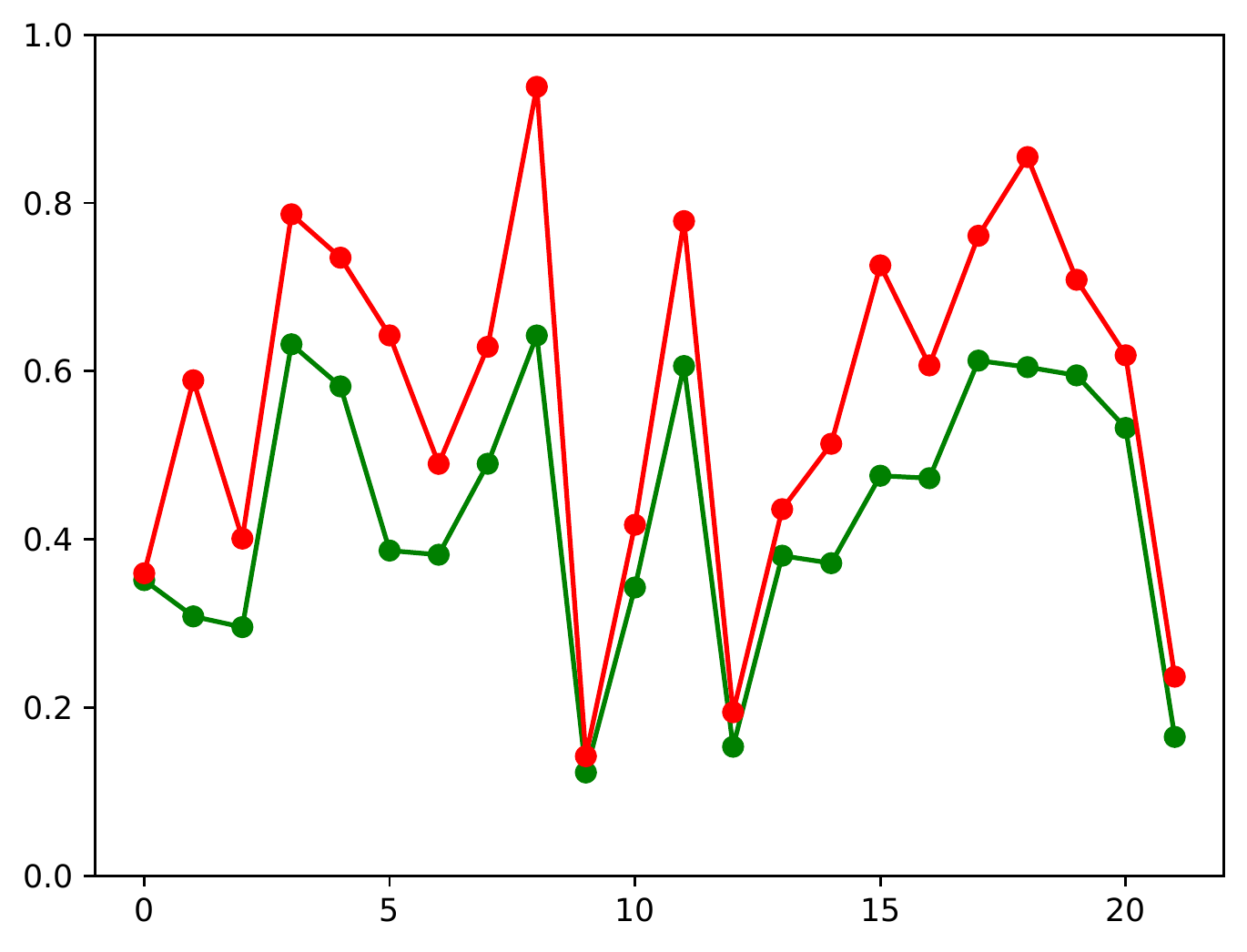}
    \end{subfigure}
    ~
    \begin{subfigure}{0.31\linewidth}
        \includegraphics[width=1\linewidth]{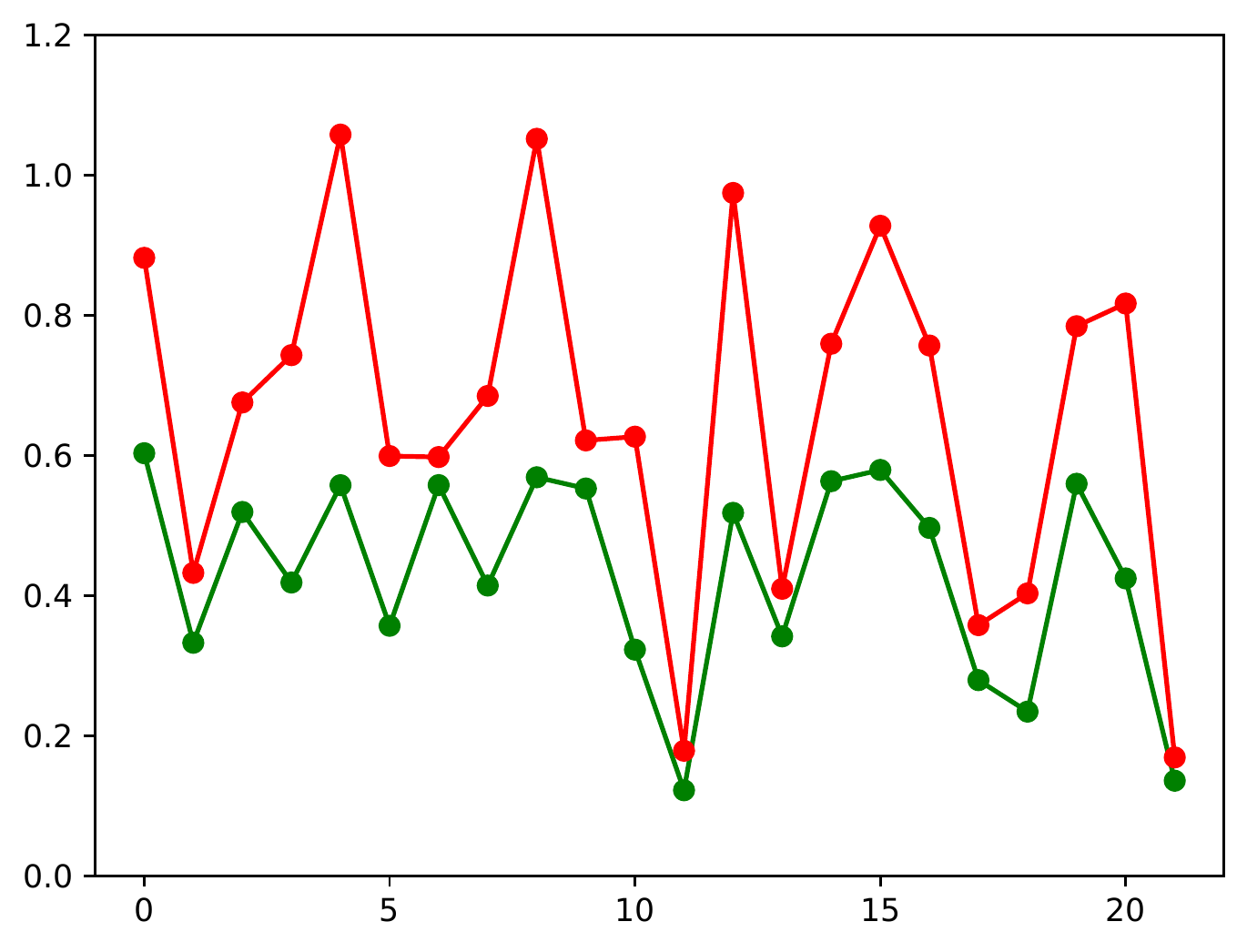}
    \end{subfigure}
    ~
    \begin{subfigure}{0.31\linewidth}
        \includegraphics[width=1\linewidth]{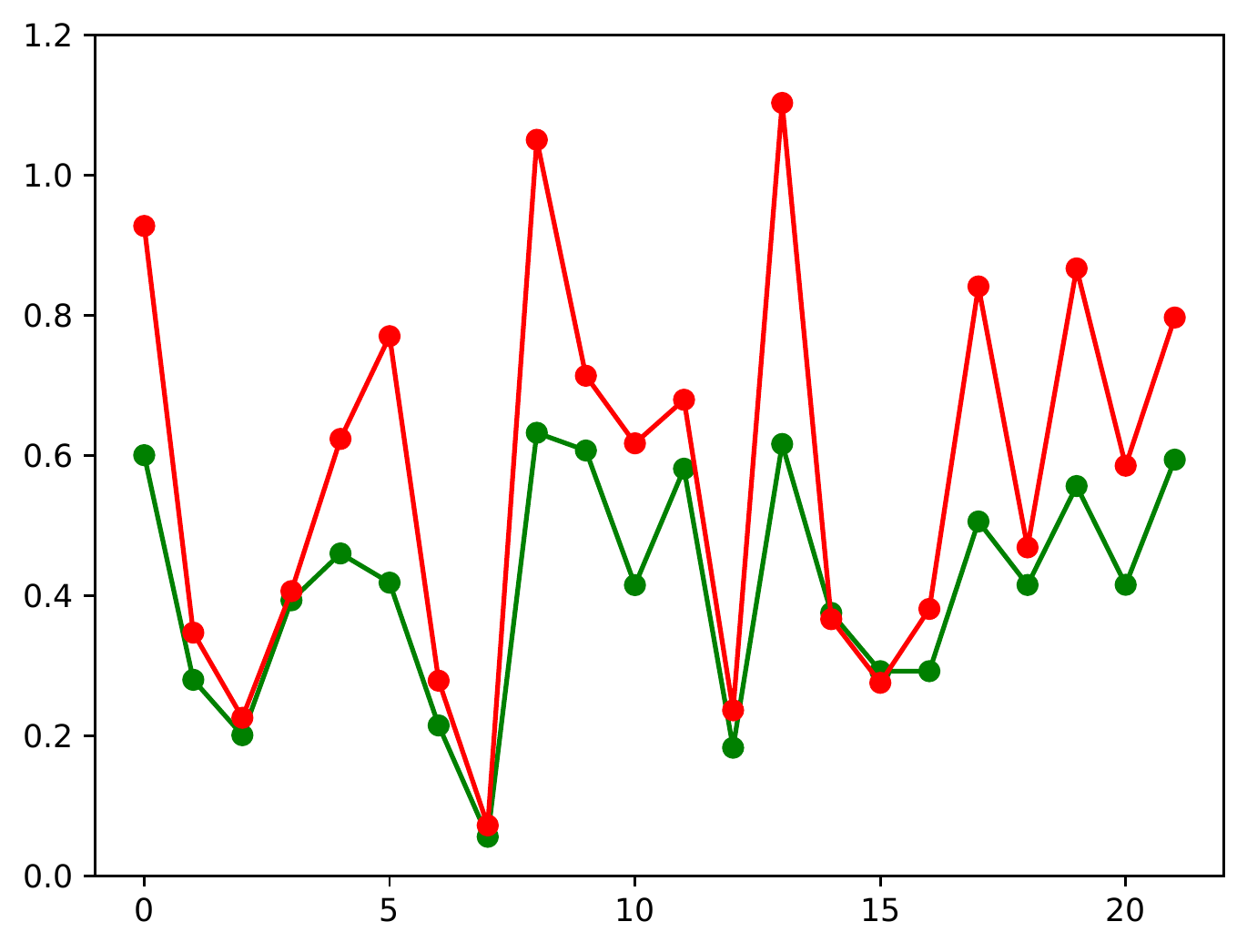}
    \end{subfigure}
    
    \caption{\emph{Top row}: Comparing the shortest path distance computed under each metric $\bm{M}_\psi(\cdot)$ with green and $\bm{M}_\theta(\cdot)$ with red. When conditions in Prop.~\ref{pror:straight_lines_prop} does not hold, then shortest paths are not straight lines. However, when the points are very close then the paths are similar. Note that the metrics are scaled to be maximum 1 on the training latent codes. \emph{Bottom row}: Each dot in the graphs correspond to the length of a curve, while the connecting lines only for cleaner illustration.}
    \label{app:constructive_with_curv}
\end{figure}

\begin{figure}[t]
    \begin{subfigure}{0.31\linewidth}
        \includegraphics[width=1\linewidth]{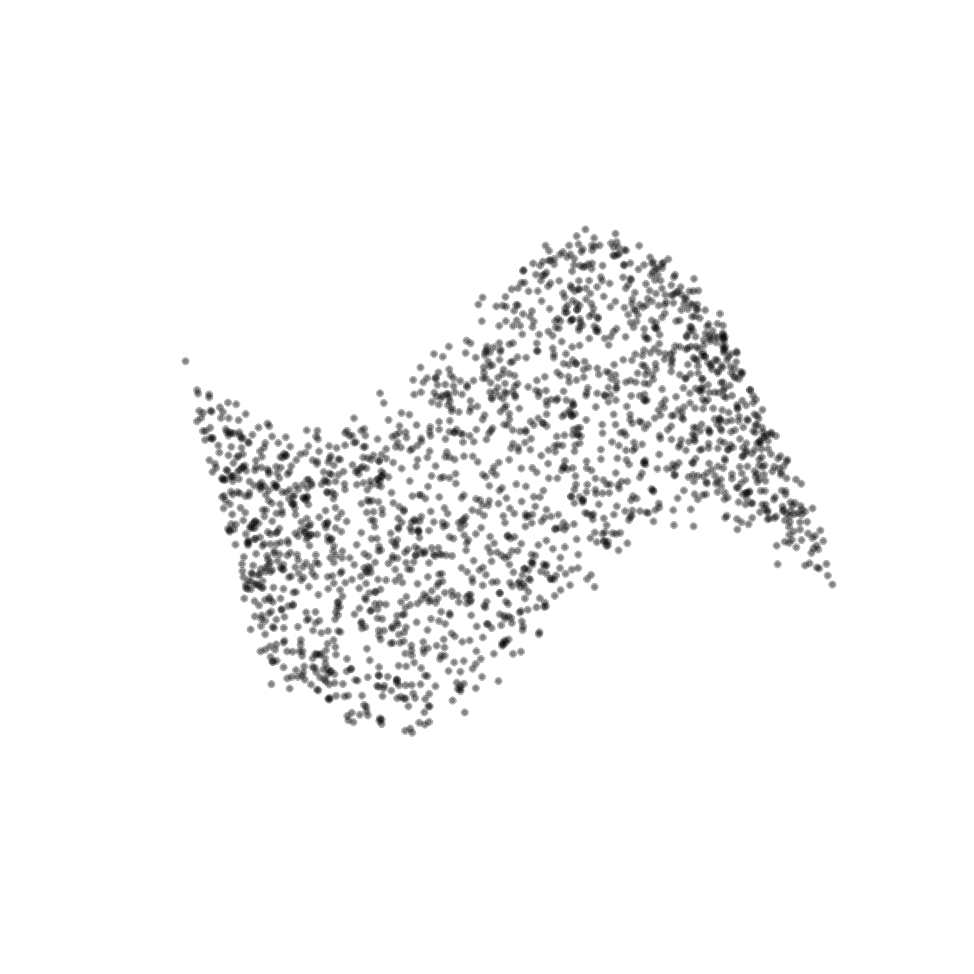}
    \end{subfigure}
    ~
    \begin{subfigure}{0.31\linewidth}
        \includegraphics[width=1\linewidth]{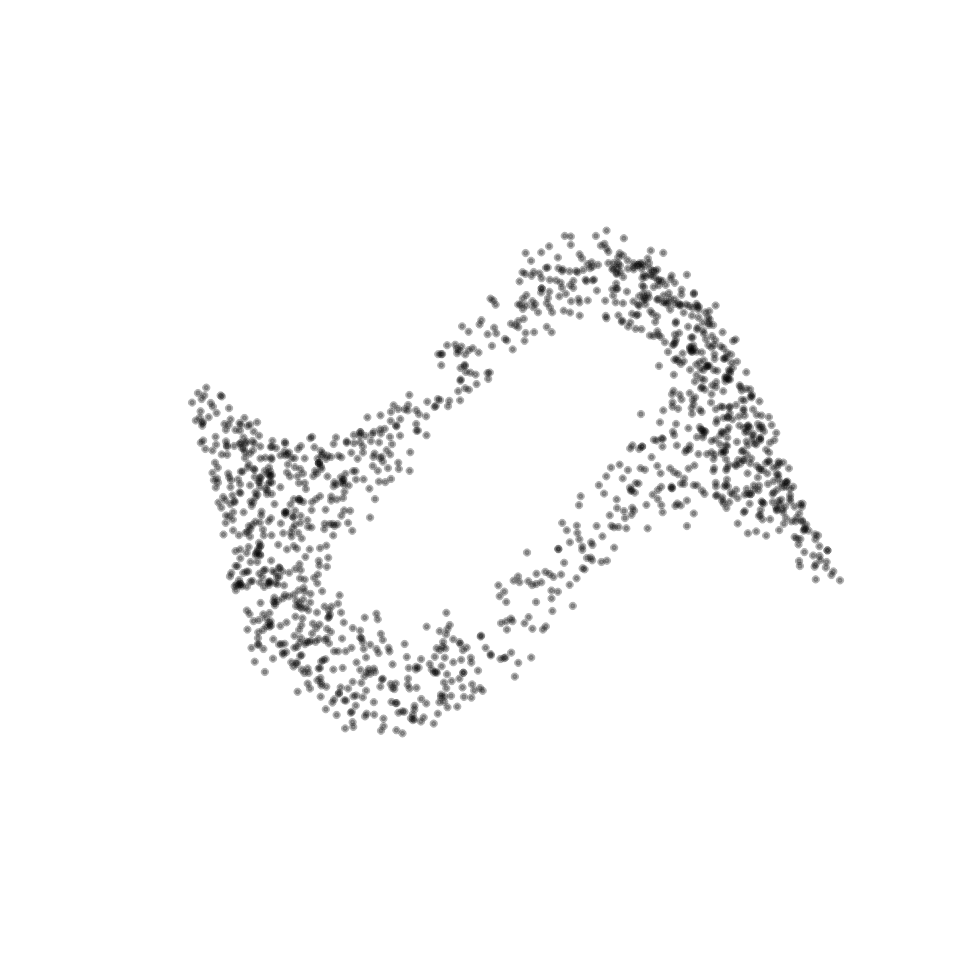}
    \end{subfigure}
    ~
    \begin{subfigure}{0.31\linewidth}
        \includegraphics[width=1\linewidth]{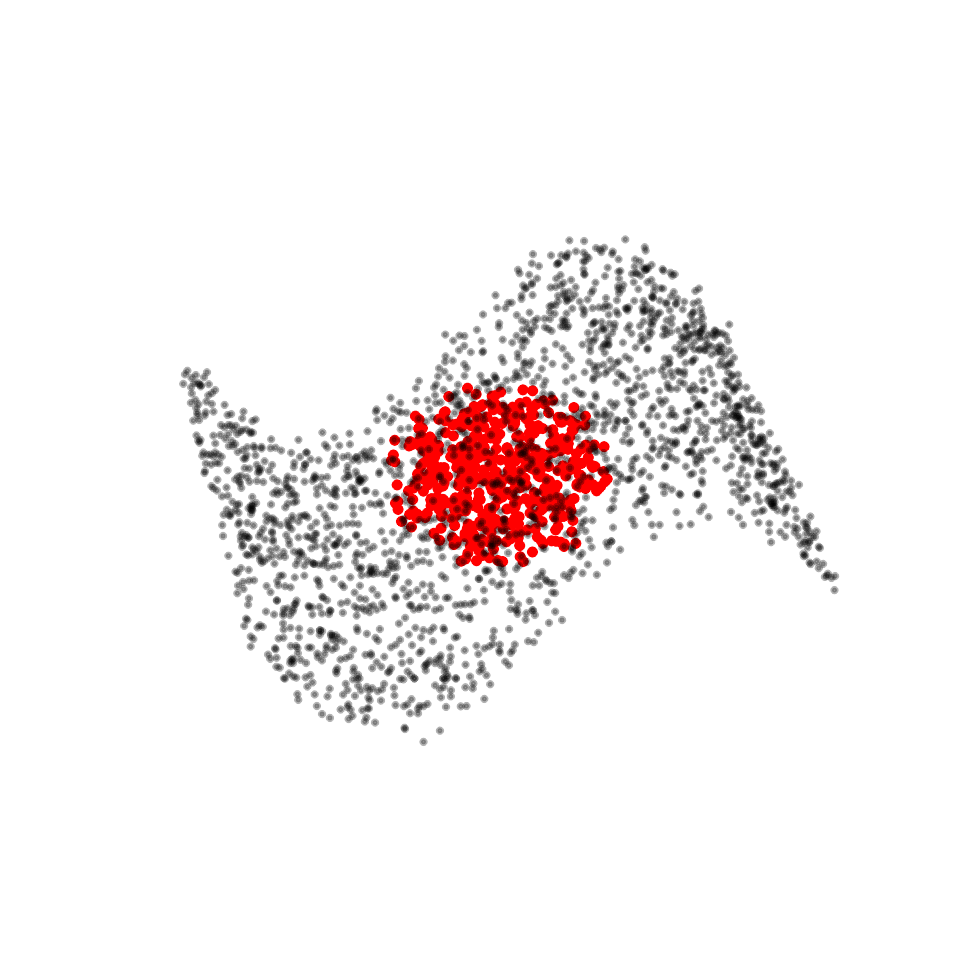}
    \end{subfigure}
    
    \begin{subfigure}{0.32\linewidth}
        \includegraphics[width=1\linewidth]{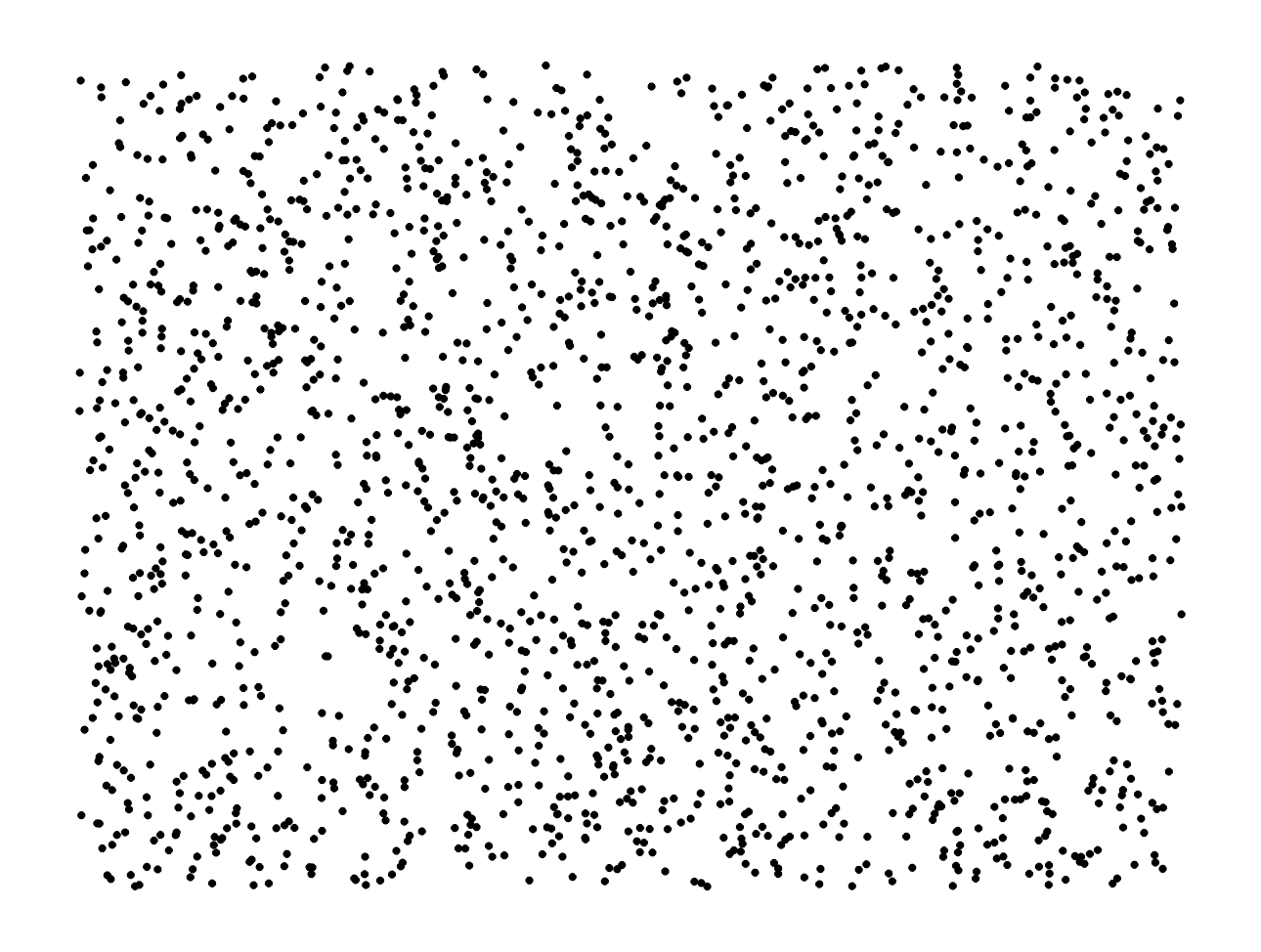}
    \end{subfigure}
    ~
    \begin{subfigure}{0.31\linewidth}
        \includegraphics[width=1\linewidth]{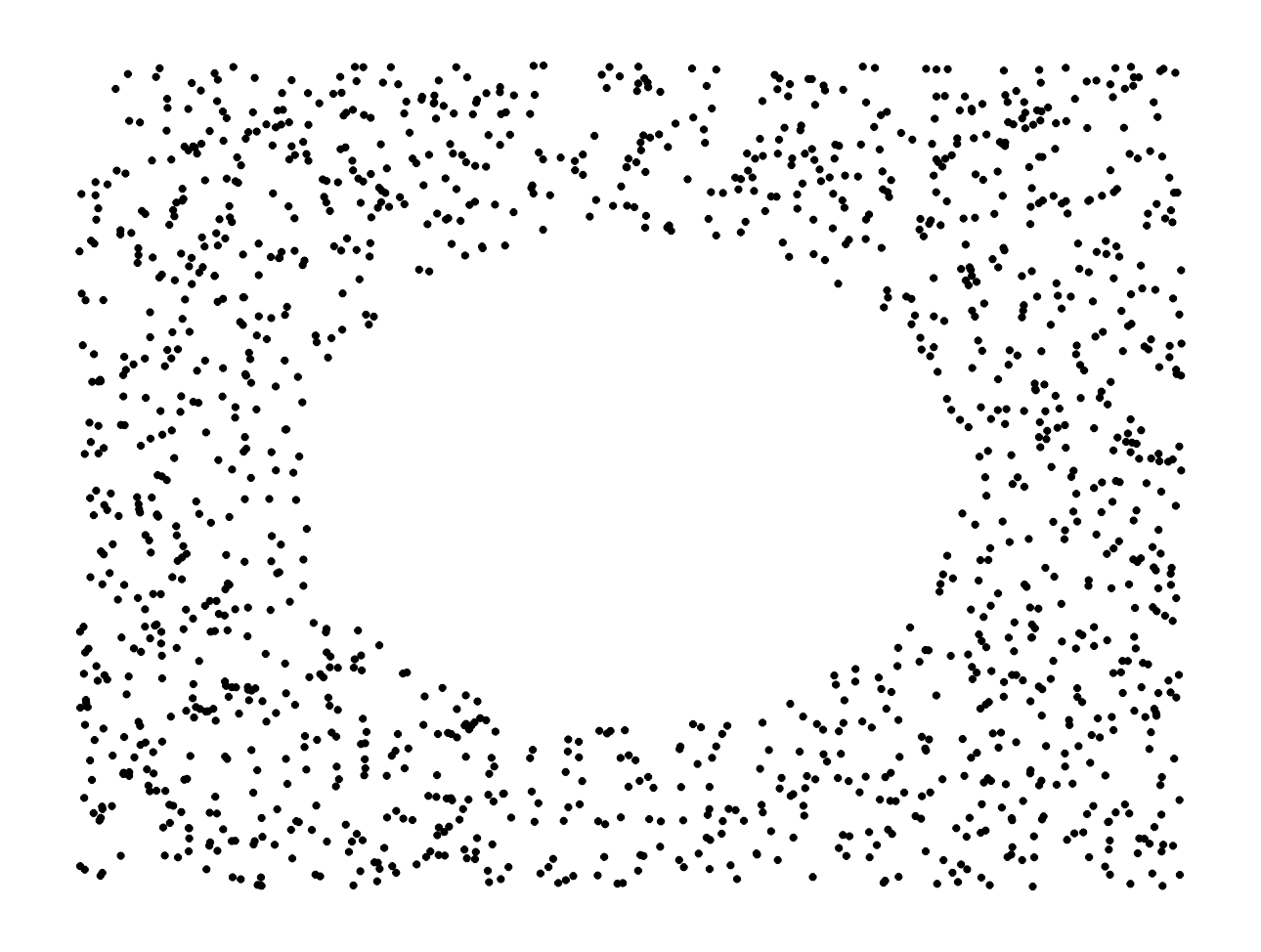}
    \end{subfigure}
    ~
    \begin{subfigure}{0.31\linewidth}
        \includegraphics[width=1\linewidth]{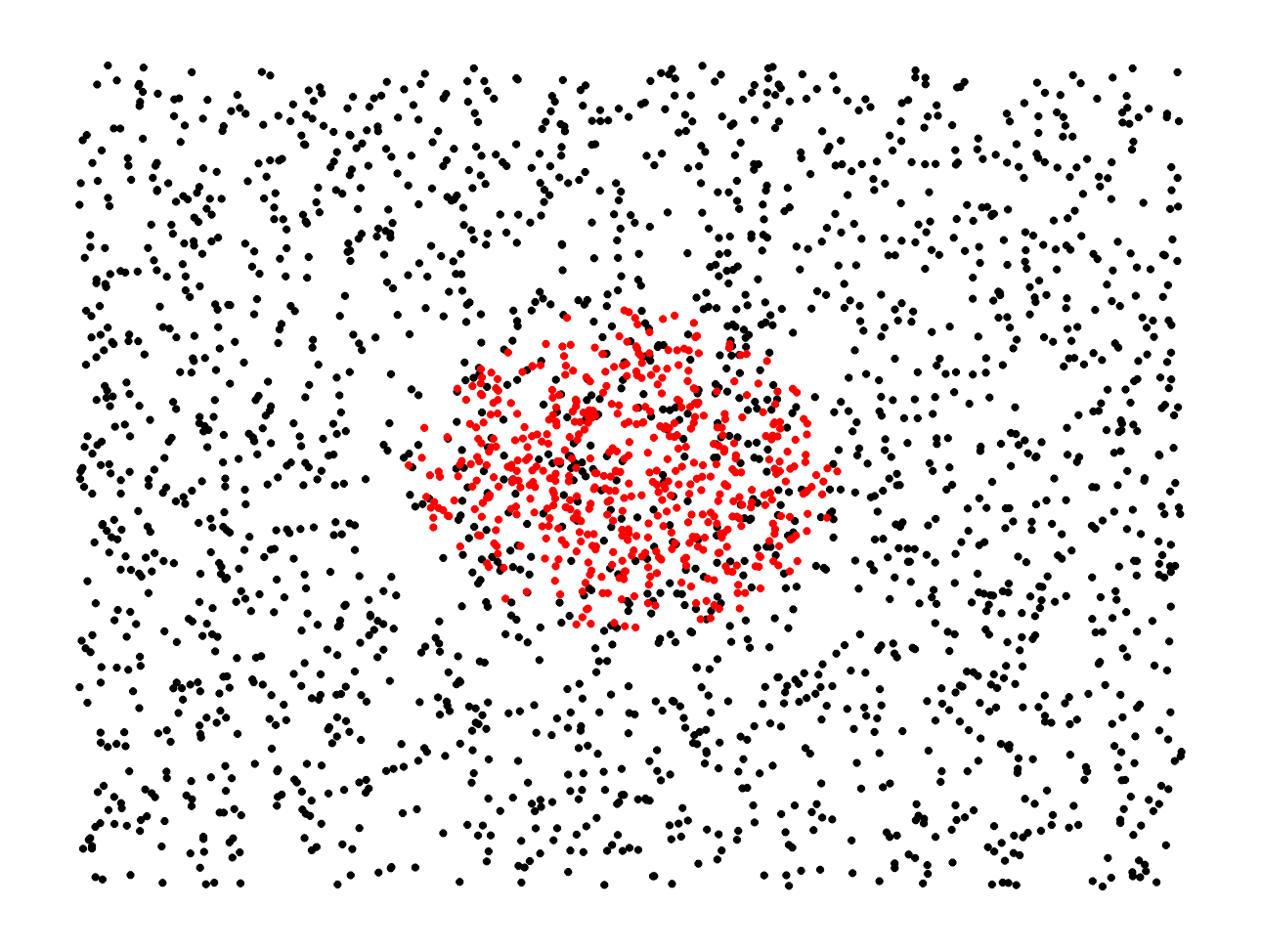}
    \end{subfigure}
    
    \caption{The synthetic data for the constructive examples.}
    \label{app:fig:constructive_data}
\end{figure}

We construct a surface in $\X=\R^3$ as $\bm{x}=[\bm{z}, 0.25 \cdot \sin(z_1)] + \varepsilon$ where ${z}_j\sim\mathcal{U}(0,2\pi),~j=1,2$ and $\varepsilon\sim \N(0, 0.1^2 \cdot \I_3)$. We call this as the normal dataset. We also construct a surface with a hole in the middle by removing the points in the center with radius $||\bm{z}||_2 < 0.3$, before the mapping in $\R^3$. Finally, we construct a uniform ball directly in $\R^3$ with radius $||\bm{x}||_2 < 0.2$ that we place in the center of the normal surface. We present the datasets in $\X$ and the corresponding ``true'' latent codes in Fig.~\ref{app:fig:constructive_data}. These are the three datasets that correspond to the analysis we did in Sec.~\ref{sec:new_metric_discussion}. In fact, the normal surface is the closest one to the Prop.~\ref{pror:straight_lines_prop}, since the manifold has low curvature locally. 

For the deep neural networks and the RBF we use the same setting as in App.~\ref{app:prior_comparisson}. Also, we use for the solution of the ODE system the strategy presented in App.~\ref{app:shortest_path_solver}. Note that in the main paper (see Sec.~\ref{sec:exp:constructive_examples}) we reparametrize the curves with respect to the Euclidean metric. This allows to compare as good as possible the actual curves in $\Z$. In other words, we compare how ``close'' are the two curves in the space.

Additionally, we show in Fig.~\ref{app:constructive_no_curv} and Fig.~\ref{app:constructive_with_curv} a second comparison to demonstrate Prop.~\ref{app:prop:reparametrization_straight_lines}. Basically, we compare the actual lengths of the shortest paths computed under each Riemannian metric. We train a VAE and an RBF using the normal surface data. In the first Fig.~\ref{app:constructive_no_curv} the bandwidth of the RBF kernels is scaled by 1.5, which makes the uncertainty term of the pull-back metric (second term in Eq.~\ref{eq:pullback_vae_metric}) nearly zero. In other words, this implies that the $\sigma_\theta(\cdot)$ is nearly constant. Since this is a simple surface the curvature of $\mu_\theta(\cdot)$ is expected to be low. Additionally, we expect the encoder to provide a nearly uniform distribution in $\Z$ since the data are almost uniformly distributed in $\X$. Note that we rescale the metrics (see App.~\ref{app:riemannian_metrics}) such that the highest mangification factor on the training latent codes to be 1 in the neighborhood of $\Z$ that we consider. Hence, as expected by Prop.~\ref{app:prop:reparametrization_straight_lines} both metrics result to shortest paths that have approximately equal lengths. However, when the bandwidth of the RBF is not scaled, it is very small, so the second term of the pull-back changes fast. Hence, even if $\mu_\theta(\cdot)$ remains the same, the uncertainty term increases the curvature. Consequently, the lengths are not similar anymore. Of course, we could potentially use always larger bandwidths for the kernels to alleviate this issue. The problem with this approach is that we lose the locality of the RBF, which implies that we do not capture precisely the geometry of the data manifold. In other words, we will allow the shortest paths to move in regions of $\Z$ with no latent codes, which does not necessarily correspond to the data manifold in $\X$.

\subsection{Details for efficiency and robustness}
\label{app:efficiency_robustness}

For the experiments that we conducted in Sec.~\ref{sec:exp:efficiency_robustness} we used the same VAE setting as in App.~\ref{app:prior_comparisson} and we projected using PCA in 100 dimensions the MNIST digits 0,1,2.

In order to have comparable magnification factors we scale each metric such that the highest magnification factor on the training latent codes to be 1 (see App.~\ref{app:riemannian_metrics}). For our proposed metric we also set the upper bound to be 100. As regards the pull-back metric we cannot explicitly control the upper bound of the magnification factor. So we set $\zeta$ in the RBF such that the maximum $\sigma^2_\theta(\cdot)$ to be 1000 times the mean variance of the training latent codes (see App.~\ref{app:riemannian_metrics}).

Here we explain in more details the result in Fig.~\ref{fig:robust_efficiency}. In practice, for the RBF kernel in higher dimensions some points can easily get a very high $\bm{J}_{\sigma_\theta}(\cdot)$. These points lie a bit further from the center of a kernel where the $\sigma^2_\theta(\cdot)$ changes extremely fast. More specifically, this occurs in the ``tails'' of the RBF kernel. As we know due to the curse of dimensionality this phenomenon is common for kernels in higher dimensions. Therefore, some of the latent codes may get a very high magnification factor, which means that this will scale down significantly the metric (see Eq.~\ref{app:eq:pull_back_metric_scaling}). This is precisely what we observe in Fig.~\ref{fig:robust_efficiency} (\emph{right}) for $d=10$. Also, in this case the distribution of the magnification factor has two modes. The interpretation is that there are some latent codes near the centers so the magnification factor is small, and some points closer to the tails of the RBFs which results in higher magnification factor values. While there are few points closer to the tails which causes the huge downscaling.

As regards our metric, its behavior is robust, which means that the prior near the latent codes is non-zero and the actual density values across them are comparable. Also, there are no training latent codes that get extreme prior values e.g. nearly zero. In order to show that only near the latent codes the magnification factor is small, we sampled uniformly from the hypercube that surrounds the latent codes and evaluated the metric. The result in Fig.~\ref{fig:convex_hull_sampling} shows that, indeed, only near the latent codes the metric is small. Especially, as the dimension increases, there is more empty space in the hypercube with no latent codes so the magnification factor is large. The interpretation is that the proposed prior adapts well on the training latent codes and does not assign density in regions of $\Z$ with no latent codes. Of course, an extremely flexible prior should not be used, because this could easily overfit the latent codes. This results in a highly curved metric i.e., the metric changes extremely fast.

\begin{figure}[h]
    \centering
    \begin{overpic}[width=0.6\linewidth]{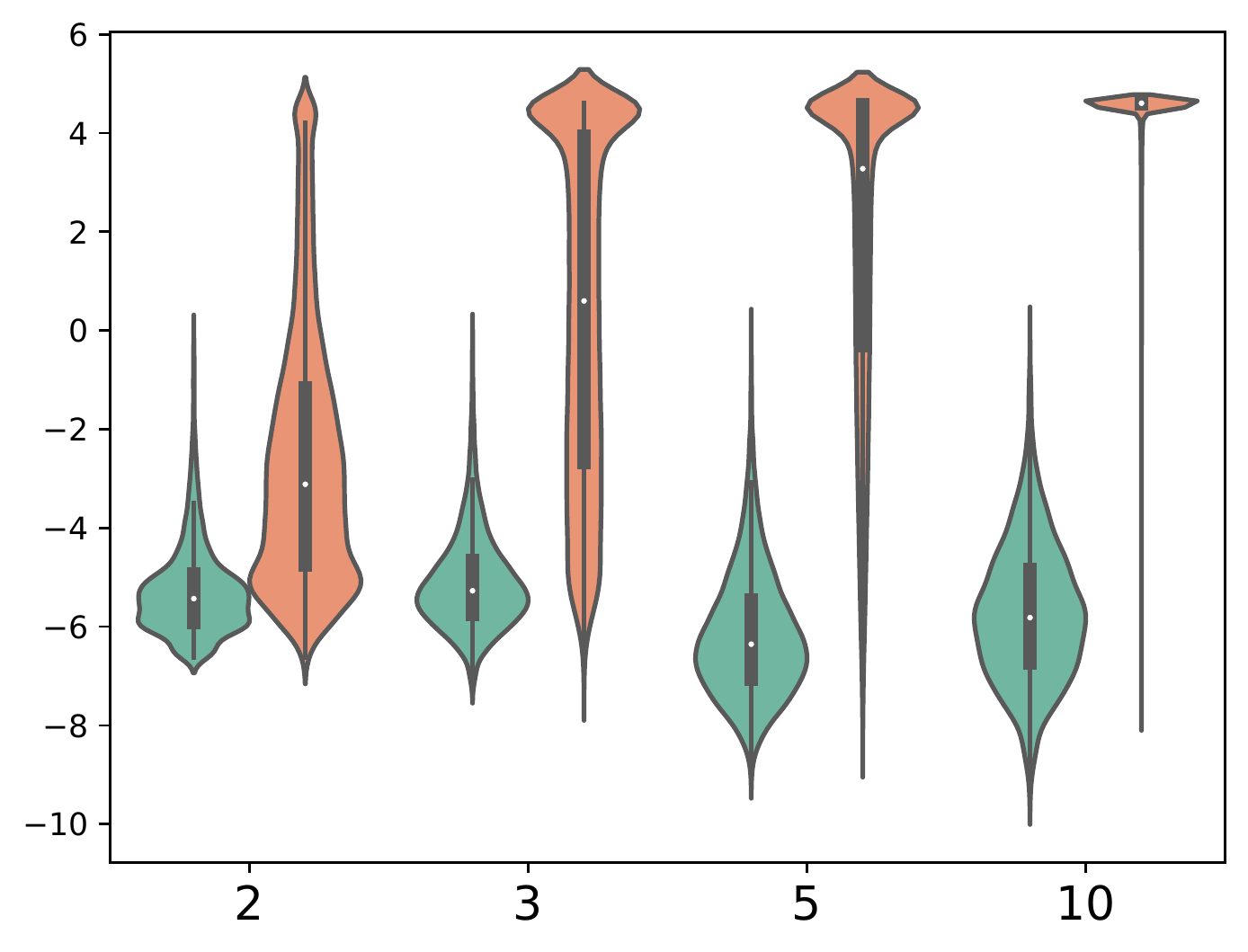}
        \put(-7, 35){\rotatebox{90}{\tiny $\log[\sqrt{\bm{M}(\cdot)}]$}}
        \put(72, -2.5){\tiny dim}
        \put(22,12){{\tiny Training latent codes}}
	    \put(18,14){{\color[RGB]{97,173,151}\circle*{5}}}
	    \put(22,20){{\tiny Uniform samples}}
	    \put(18,22){{\color[RGB]{232,140,109}\circle*{5}}}
    \end{overpic}
    \caption{Comparison of the magnification factor between the training latent codes and uniform samples in their bounding box.}
    \label{fig:convex_hull_sampling}
\end{figure}
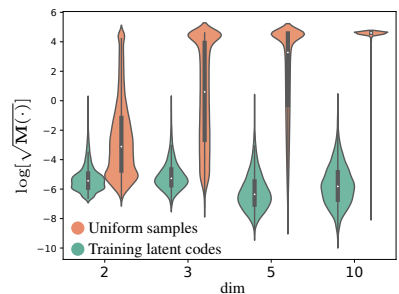

Finally, due to the huge curvature of the pull-back metric, mainly due to the uncertainty term, the shortest path solver has to run for longer and also fails many times. Note that even the evaluation of the metric and its derivative that we need to compute the ODE system (see Eq.~\ref{eq:general_ode}) is significantly less efficient than our proposed metric. Here we used a different strategy to compute shortest paths. First we run the BVP solver with the straight line as the initial solution, and if this fails, we re-run the solver initialized by the graph based solution. The reason for doing that is to show that the ODE system under the proposed metric is easier and it can be solved directly without the graph initialization.

\subsection{Details for the LANDs experiment}
\label{app:ex:land_mnist}

For the LAND experiment we used the same VAE, RBF and data as in App.~\ref{app:efficiency_robustness}. Note that we did not use all the latent codes for training the LANDs, but instead, we quantized them using $k$-means with 120 centers. Even if the proposed metric is much more efficient than the pull-back, still computing one shortest path relies on the solution of an ODE system. This makes the use of all the latent codes prohibited. As reported in the main paper (see Sec.~\ref{sec:ex:land_mnist}) the running times are significantly different, as it is much more efficient to compute shortest paths under our proposed metric.

Moreover, we observe that the pull-back metric underestimates the precision matrices (or overestimates the covariance matrices). The reason is that for some points the shortest path length is large, because the pull-back metric gets large mainly due to the non-robustness of the RBF term. So the corresponding logarithmic map is large as well. This causes the precision to become smaller such that to capture these points that lie ``far'' from the component's center. Obviously, this is not a desirable behavior, since it only occurs due to the poor behavior and non-robustness of the RBF.

\subsection{Details for for life science experiments}
\label{app:experiments_biology}

Here we explain the details for the experiments with the real-world datasets (see Sec.~\ref{sec:ex:life_data}). As we mention in the main paper these experiments should be considered as a proof-of-concept, since specialized generative models have been proposed in the literature. With our experiment we want to show that geometry might be a suitable theory to utilize for exploratory data analysis in life sciences.

\textbf{Mouse cortex cell data}

For the mouse cortex cell data \citep{zeisel:science:2015} we used the scvi-tools\footnote{https://www.scvi-tools.org/en/stable/index.html}. As a reprocessing step we kept the 558 genes with the highest variability, and also, we projected the data into 100 dimensions using PCA for simplicity of the analysis. For the VAE and the RBF we used the setting as App.~\ref{app:prior_comparisson}. For the shortest path solver we used the setting as in App.~\ref{app:shortest_path_solver}. Since for the LAND fit we need the logarithmic maps to be as precise as possible, if the BVP solver fails, then we do not consider this point for the corresponding mixture component. Also, we quantized the latent codes using $k$-means and where $k=200$.

The resulting LAND adapts better to the latent codes, especially when we observe the individual components Fig.~\ref{fig:app:cortex}. Interestingly, the centers between the GMM and the mixture of LANDs differ. One reason is that the Euclidean distance of the latent codes does not correspond to the actual distance on the data manifold. For example, if some points are very sparse on the data manifold in $\X$, the encoder will push everything towards the support of the base distribution $p(\bm{z})$ of our proposed prior. However, the Euclidean mean in $\Z$ is not aware of the data manifold's geometry in $\X$, while the LAND mean potentially corresponds to a better estimate on the actual data manifold. In particular, the geometry aware mean under our proposed metric will be closer to the high density region in $\Z$. For instance, assume that the latent codes exhibit a non-convex distribution as a semi-circle. In this case, the Euclidean mean will be outside of the latent codes support, while our mean will be in the support.

For further analysis, we show that we can use the principal geodesics for each component (see Fig.~\ref{app:fig:kmeans_principal_geodesics}), which can be seen as a form of local \emph{disentaglement}. More specifically, we eigen-decompose the precision matrices of the LANDs mixture and we solve the exponential map with initial velocities the eigen-vectors. Clearly, the resulting paths correspond to the directions with the highest variance on the data manifold in $\X$. In this way we are able to recover locally the intrinsic degrees of freedom of the dataset. This can be seen as a non-linear extension of the PCA. Geometry aware disentanglement seems as a promising direction for future research \citep{pfau:neurips:2020}.

\textbf{Chemical compounds}

We used the ZINC database\footnote{https://zinc.docking.org/} \citep{zinc:2015}. In particular, using the SMILES representation, we sampled randomly 6400 points and for simplicity kept only the first 30 characters of each sequence. Clearly, this is rather simplified setting, however, patterns of chemical structures are present. Each batch has dimension $128\times 30 \times 1$, using one-hot encoding. For the VAE we used an encoder based on 1-dimensional Convolutions and a recurrent decoder, for specific details see Table~\ref{app:tab:recurrent_vae}. Chemical compounds have by definition an inherent natural structure. As we observe in Fig.~\ref{app:fig:chem_paths_planes} the resulting representations are indeed non-linearly structured, which the corresponding shortest path respect. So we are able to find interpretable, more meaningful shortest paths between points, compute mean values and barycenters accordingly, etc. Hence, geometry could potentially uncover some useful properties in the latent space. Even if our setting is rather simplified, we see that indeed nonlinear structures appear in the latent space. Similarly, \citet{detlefsens:arxiv:2020} studied the behavior of latent representations for protein sequences and showed that structure aware paths reveal biological
information that is otherwise obscured. 

\begin{table}[b]
    \centering
    \begin{tabular}{c|c c c}
        $\text{encoder}_\phi(\cdot)$ & $5 \times 32$ & $5 \times 32$ & $5 \times 32$ \\
         $\mu_\phi(\cdot)$ & $\text{encoder}_\phi(\cdot)$ & flatten & MLP(128, d) \\
         $\log(\sigma^2_\phi(\cdot))$ & $\text{encoder}_\phi(\cdot)$ & flatten & MLP(128, d)\\
         \midrule
        $\text{c}_{\theta}(\cdot)$ & MLP(d, H) & & \\
        $\text{hs}_{\theta}(\cdot)$ & MLP(d, H) & & \\
        $\text{decoder}_{\theta}(\cdot)$ &\multicolumn{2}{c}{LSTM(H, $\text{c}_{\theta}(\cdot)$, $\text{hs}_{\theta}(\cdot)$)} & MLP(H, 1) \\
    \end{tabular}
    \caption{Recurrent VAE. Enc: Conv1d with kernel size 5 and 32 filters, stride 1 in the first conv and stride 2 for the other two. Dec: For the output sequence we input 1 at each LSTM step and get the output value pushing the hidden state through an MLP. We used \texttt{tanh} activations and $H=128$.}
    \label{app:tab:recurrent_vae}
\end{table}

\begin{figure*}[h]
	\centering
	\begin{subfigure}{0.24\linewidth}
	    \includegraphics[width=1\linewidth]{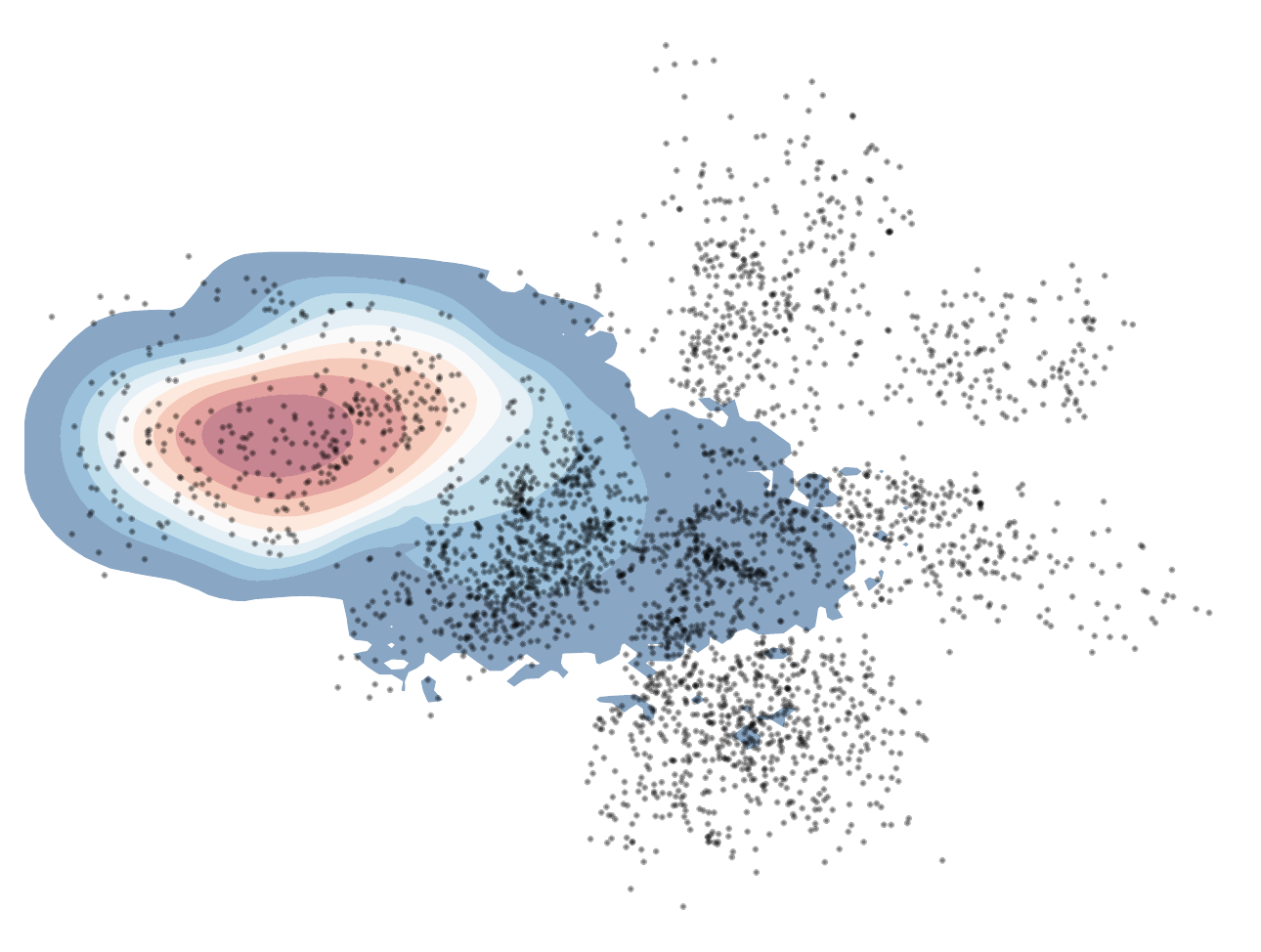}
	\end{subfigure}
	\begin{subfigure}{0.24\linewidth}
	    \includegraphics[width=1\linewidth]{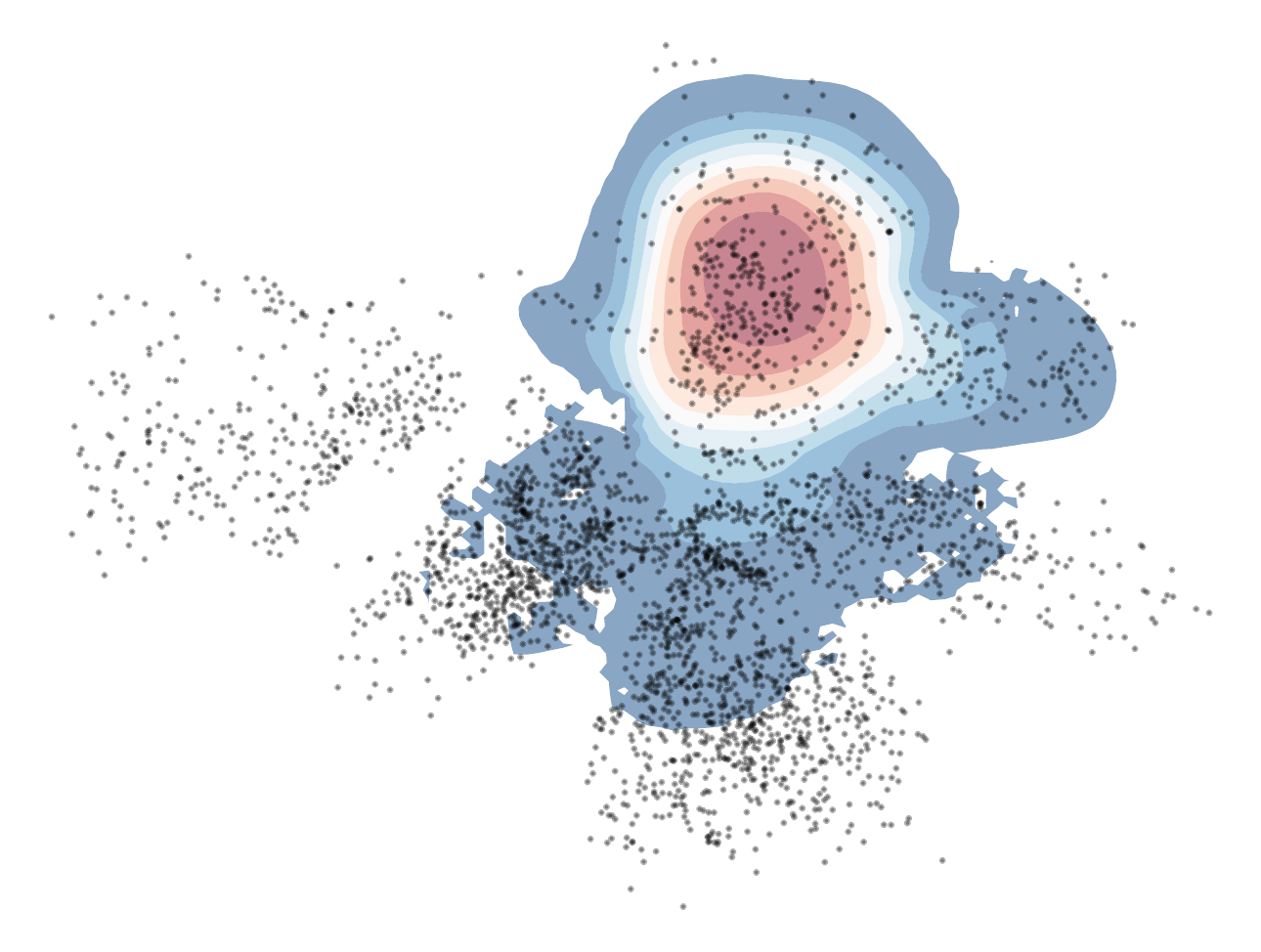}
	\end{subfigure}
	\begin{subfigure}{0.24\linewidth}
	    \includegraphics[width=1\linewidth]{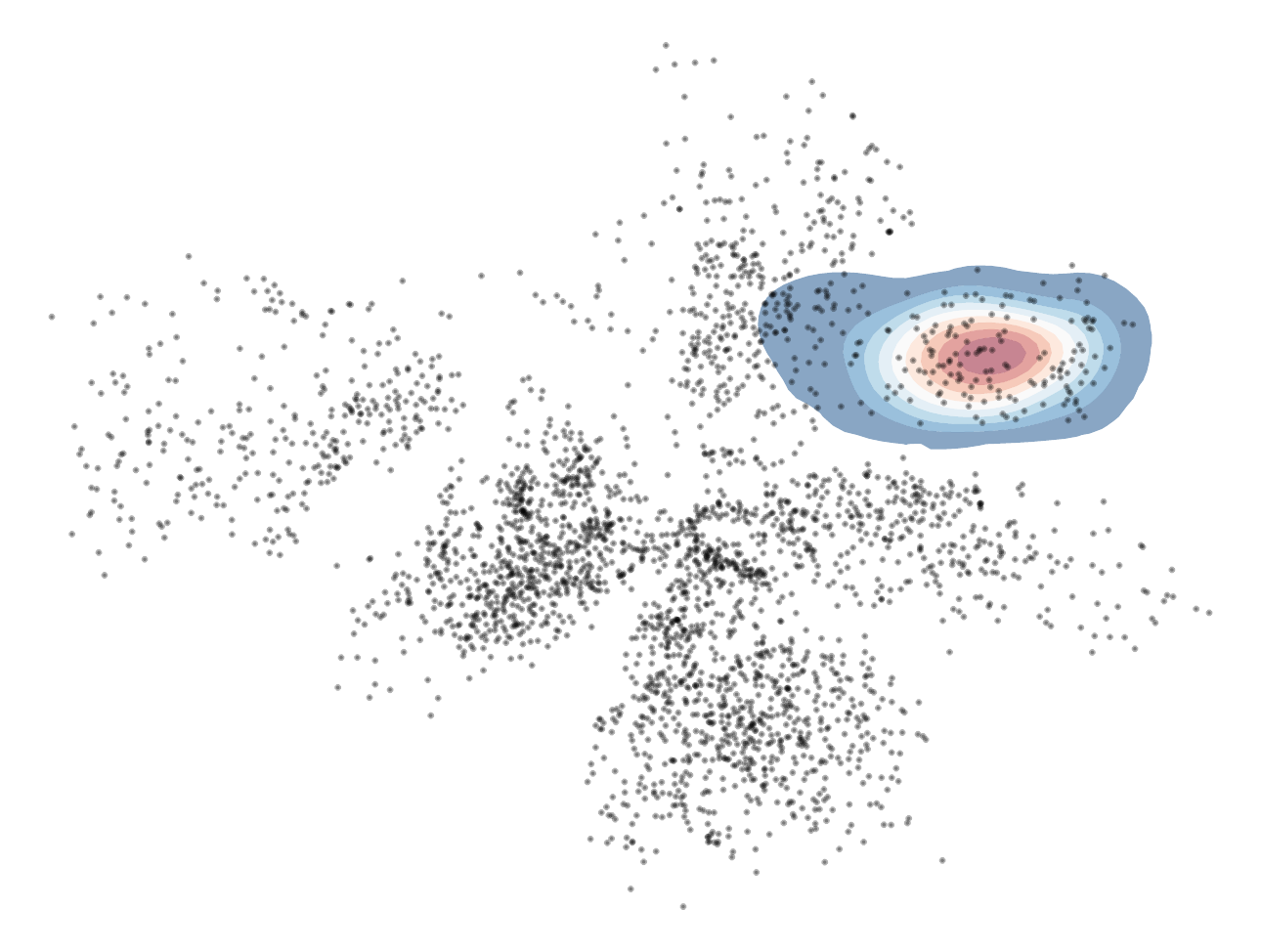}
	\end{subfigure}
	\begin{subfigure}{0.24\linewidth}
	    \includegraphics[width=1\linewidth]{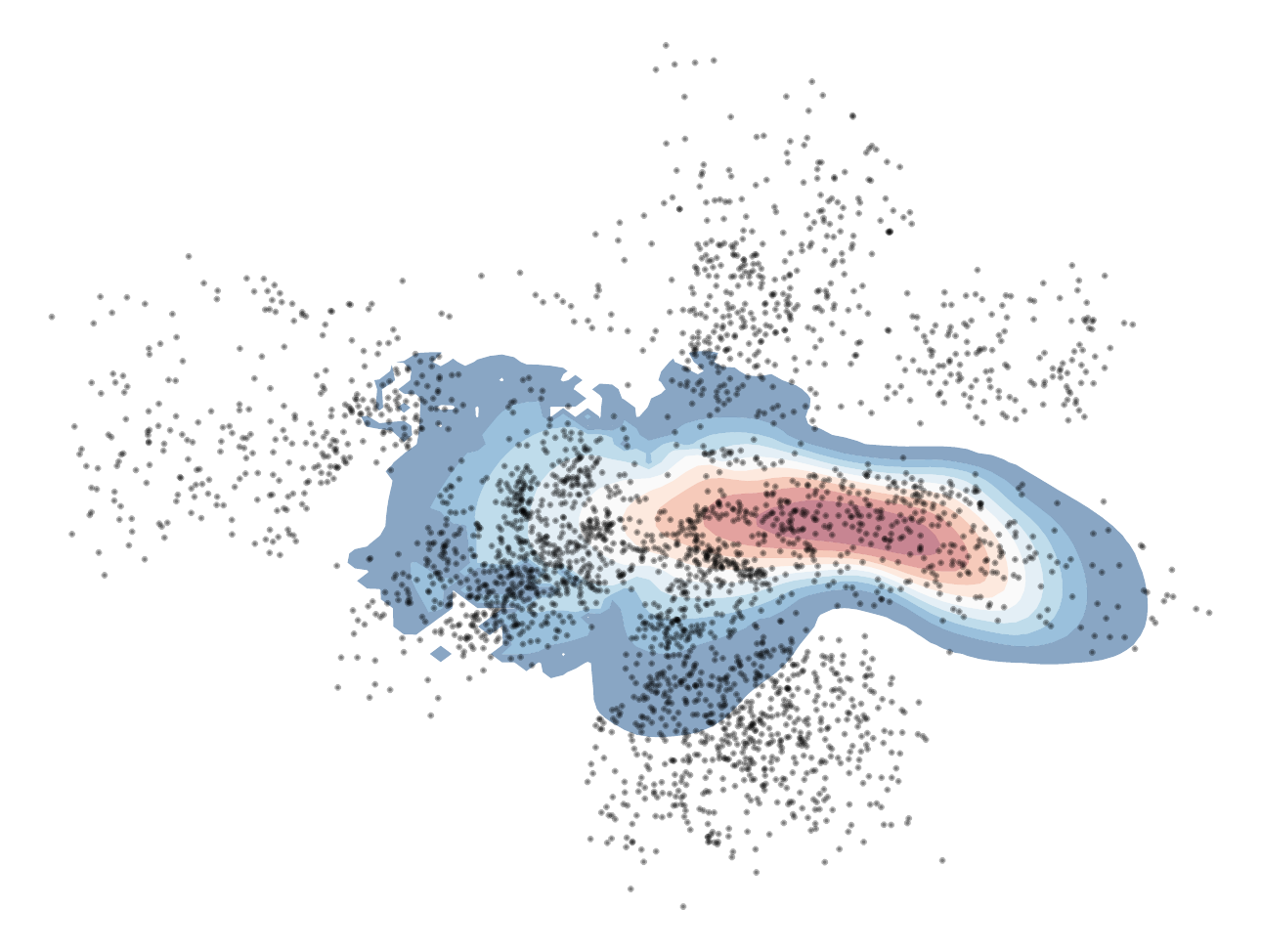}
	\end{subfigure}

	\begin{subfigure}{0.24\linewidth}
	    \includegraphics[width=1\linewidth]{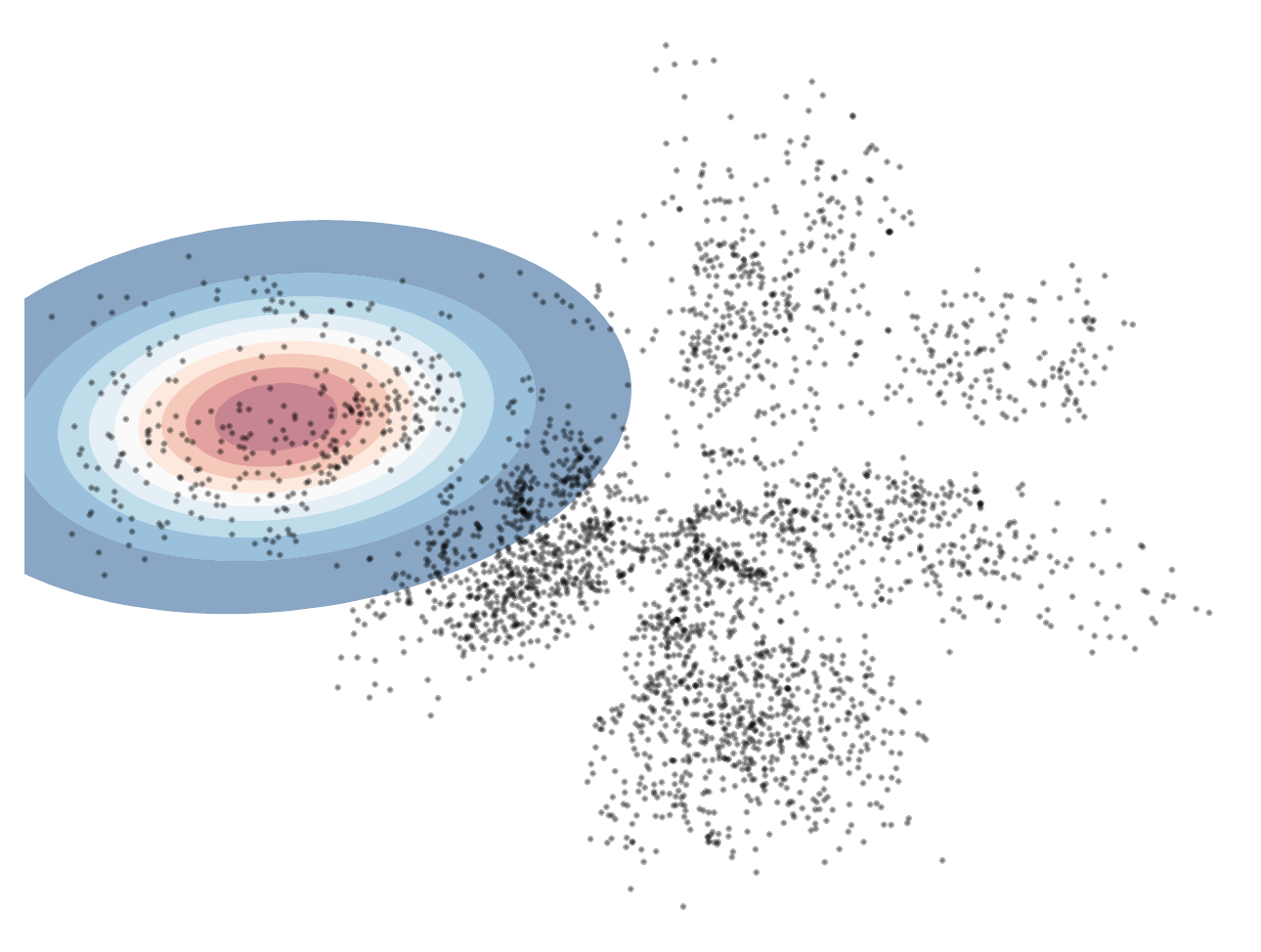}
	    \caption{Cluster 1}
	\end{subfigure}
	\begin{subfigure}{0.24\linewidth}
	    \includegraphics[width=1\linewidth]{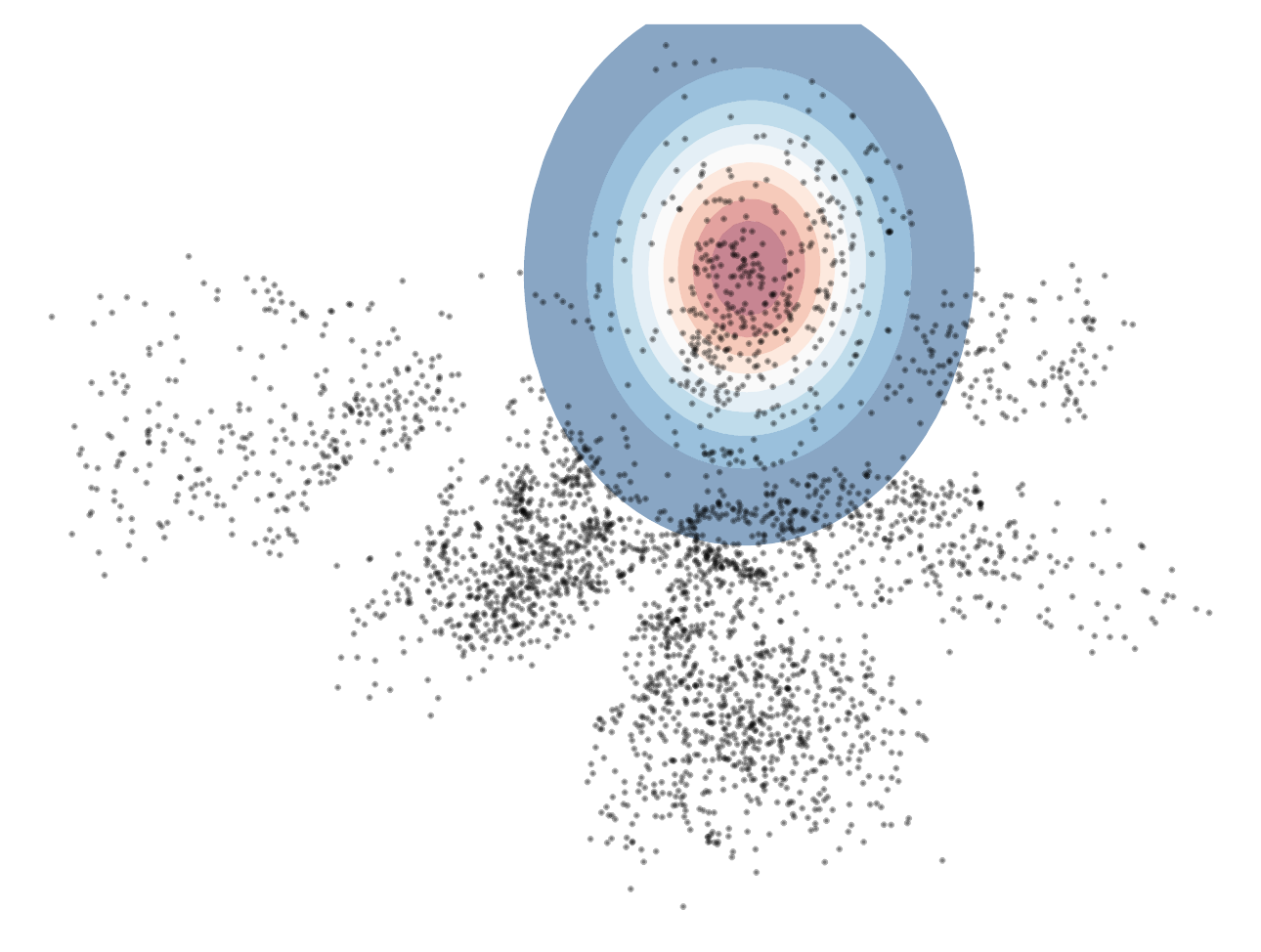}
	    \caption{Cluster 2}
	\end{subfigure}
	\begin{subfigure}{0.24\linewidth}
	    \includegraphics[width=1\linewidth]{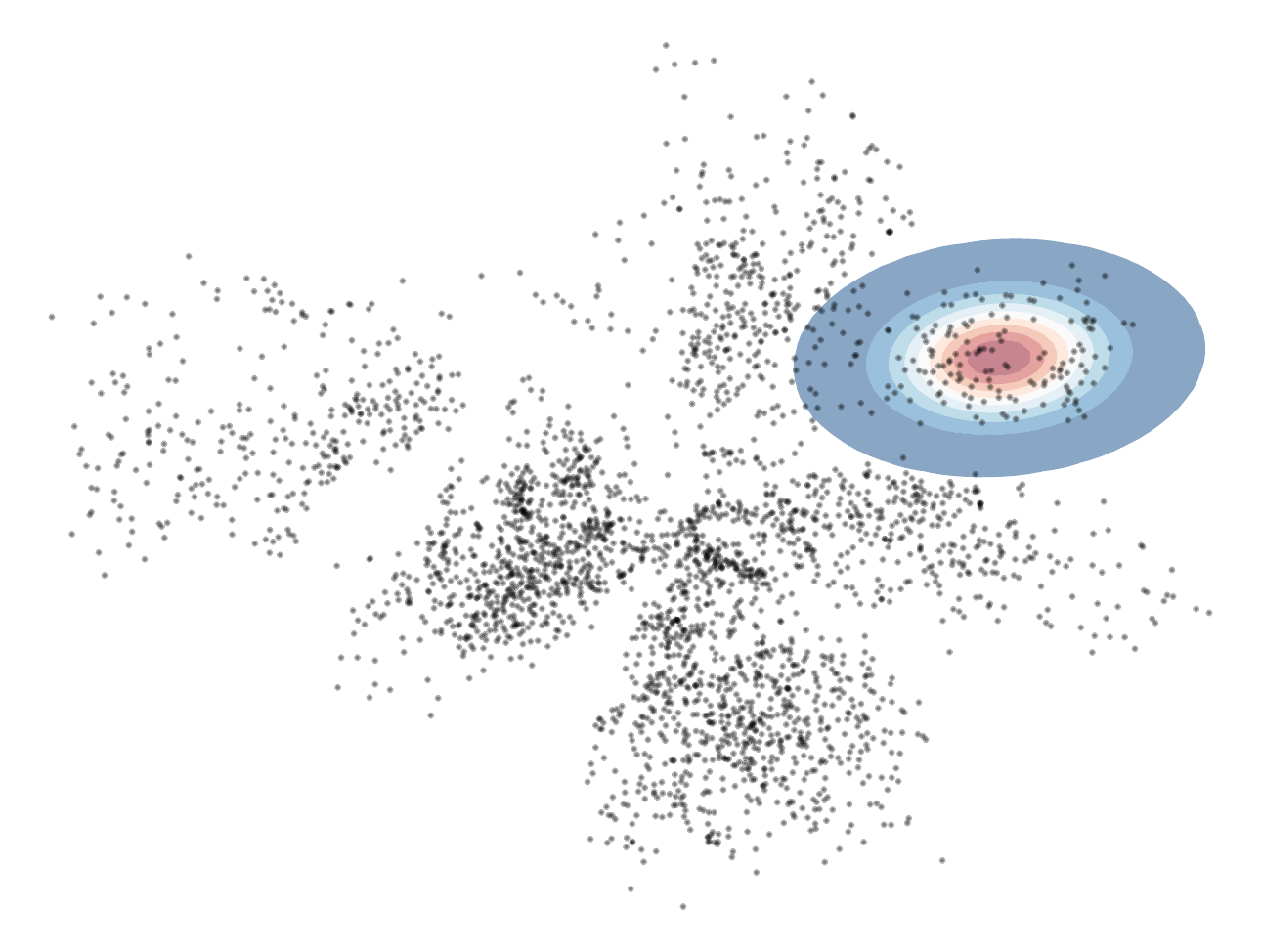}
	    \caption{Cluster 3}
	\end{subfigure}
	\begin{subfigure}{0.24\linewidth}
	    \includegraphics[width=1\linewidth]{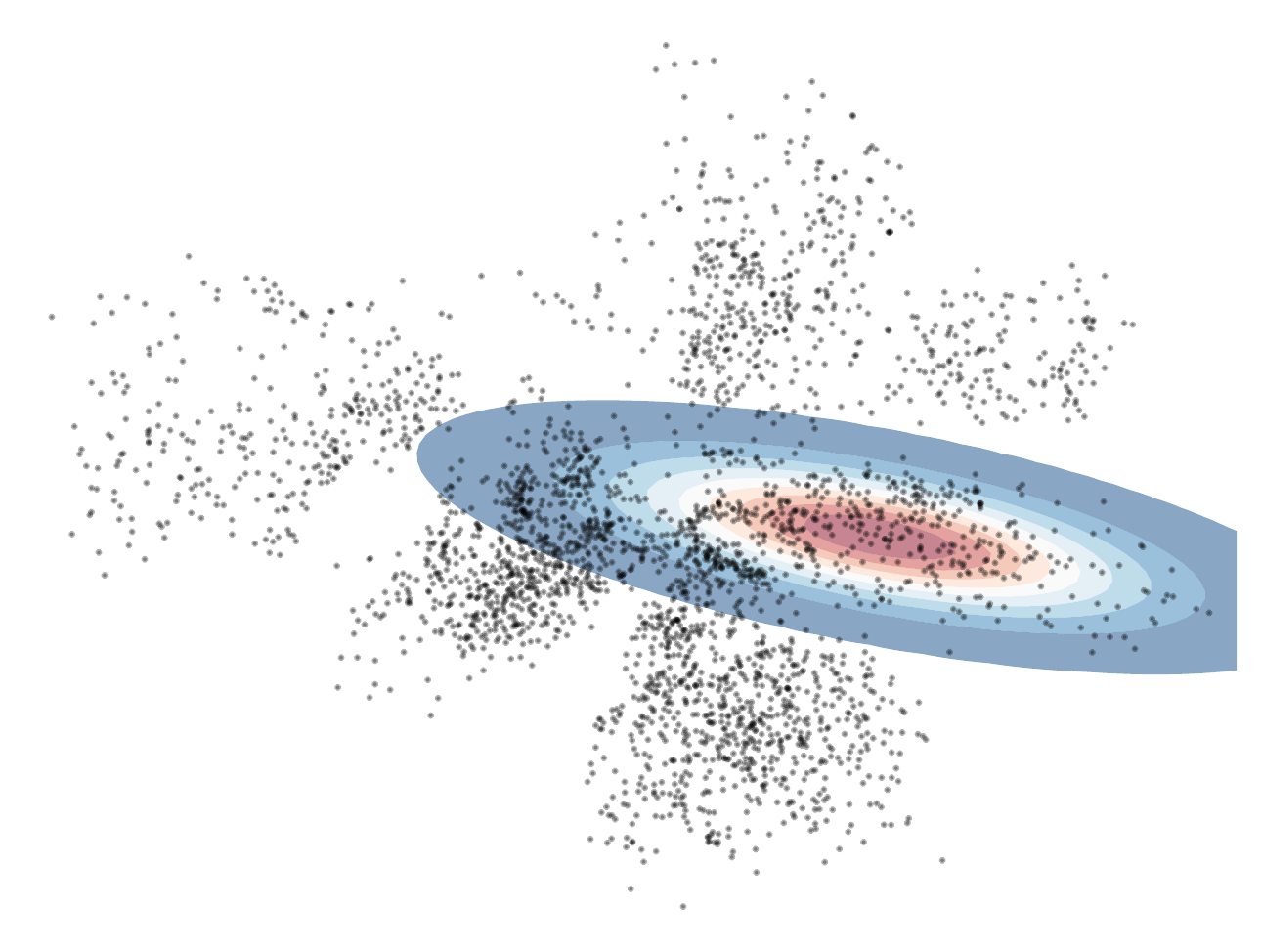}
	    \caption{Cluster 4}
	\end{subfigure}
	
	\raggedright
    \begin{subfigure}{0.24\linewidth}
	    \includegraphics[width=1\linewidth]{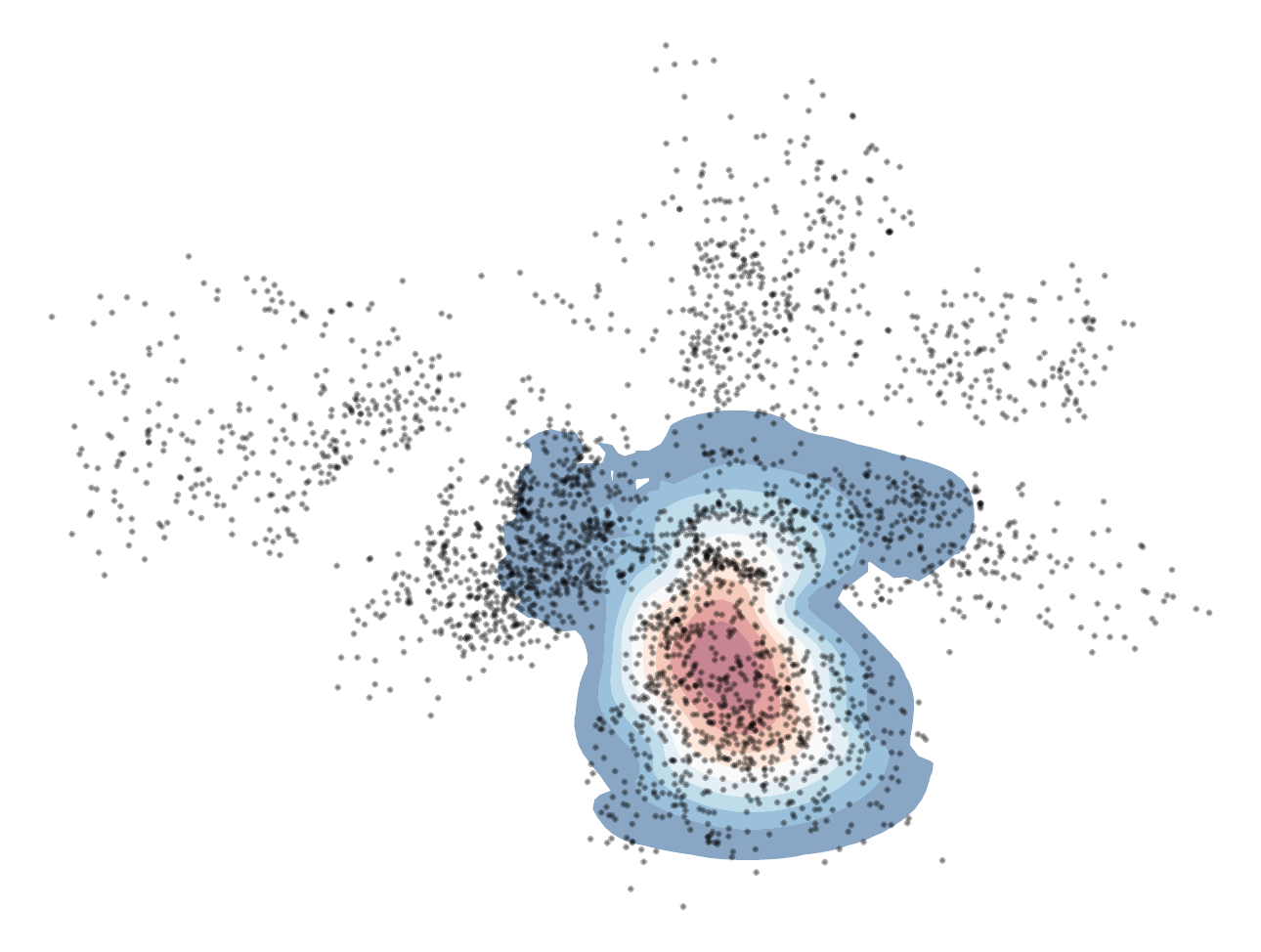}
	\end{subfigure}
	\begin{subfigure}{0.24\linewidth}
	    \includegraphics[width=1\linewidth]{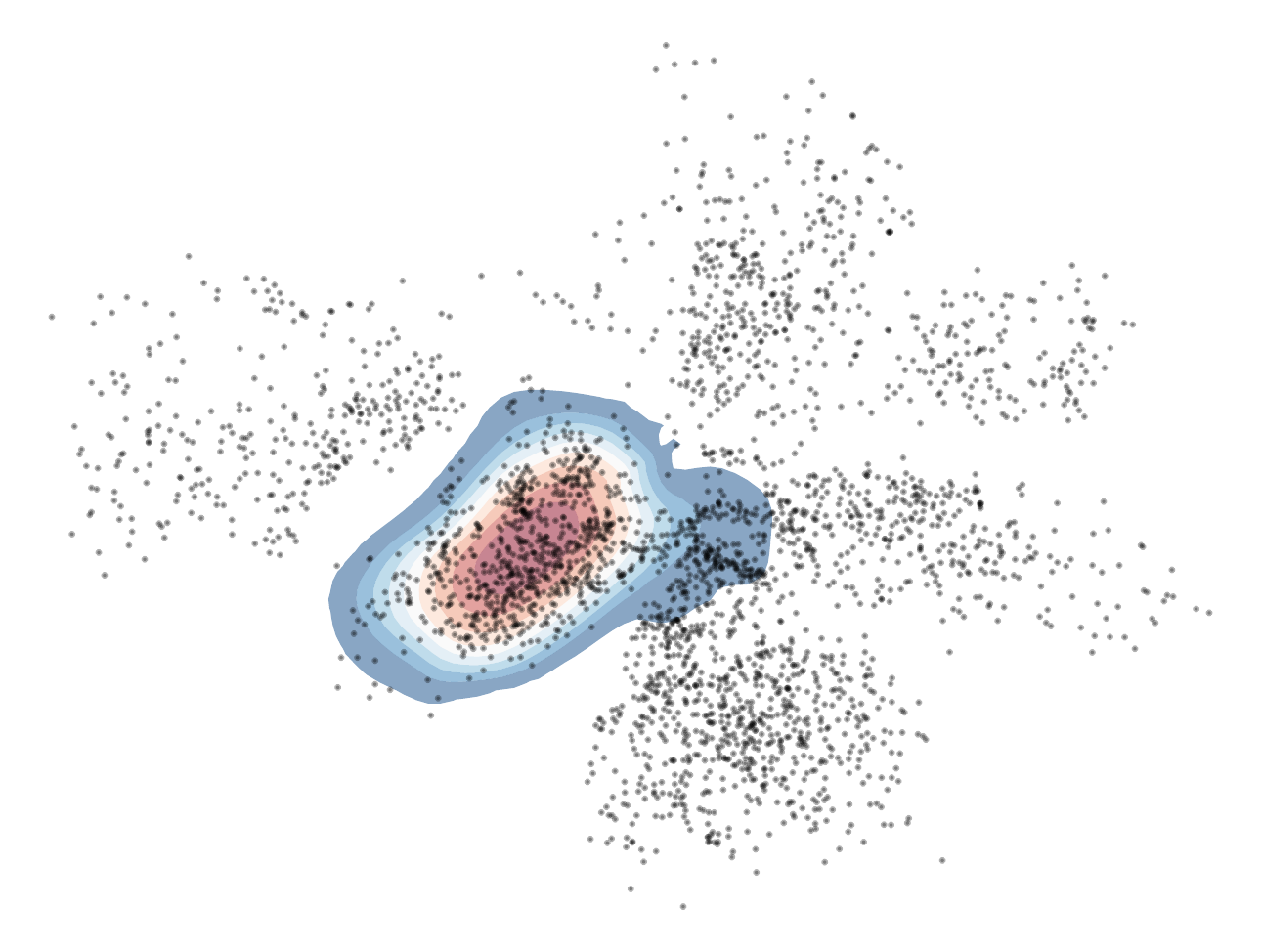}
	\end{subfigure}
	
	\raggedright
	\begin{subfigure}{0.24\linewidth}
	    \includegraphics[width=1\linewidth]{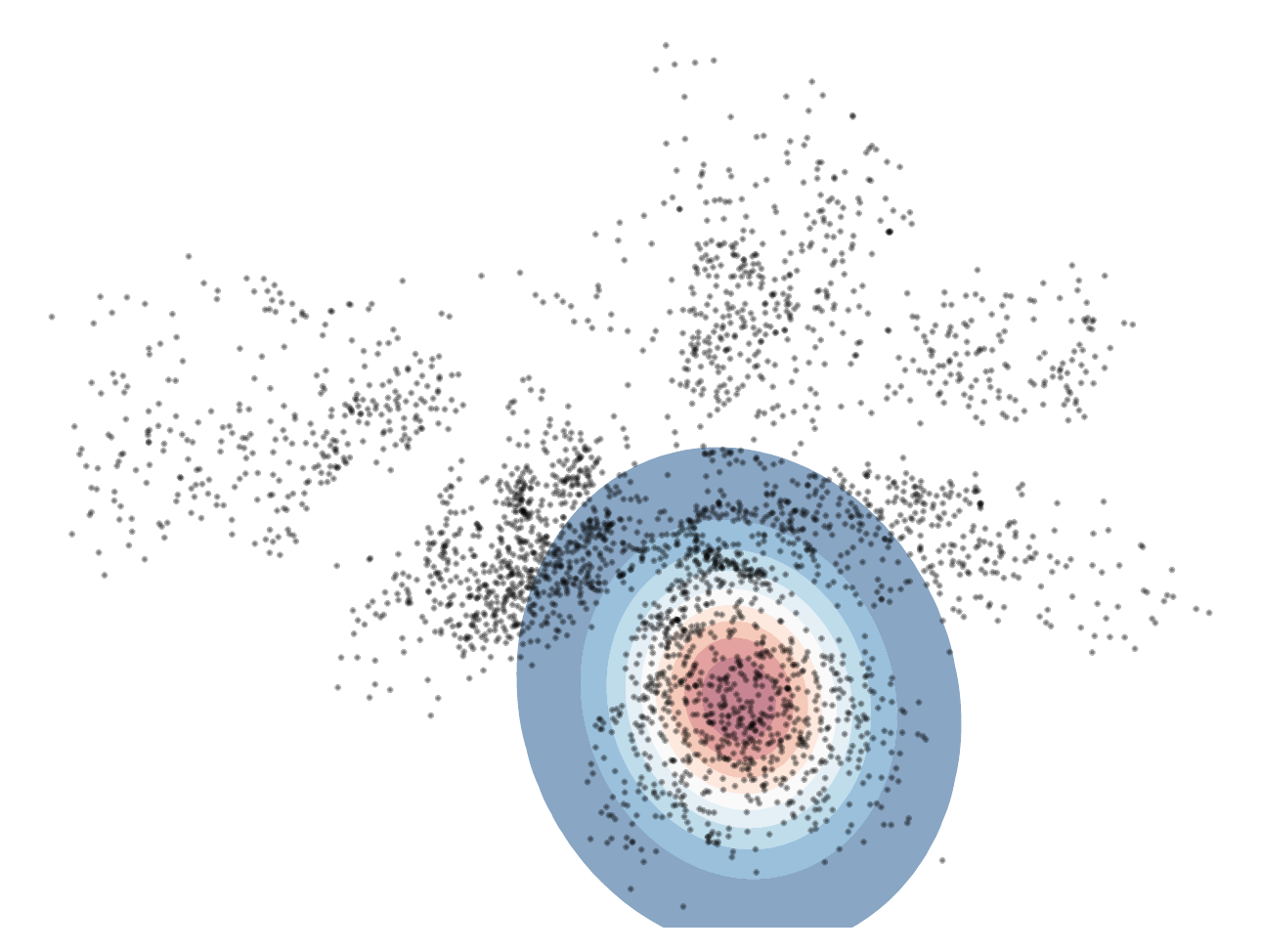}
	    \caption{Cluster 5}
	\end{subfigure}
	\begin{subfigure}{0.24\linewidth}
	    \includegraphics[width=1\linewidth]{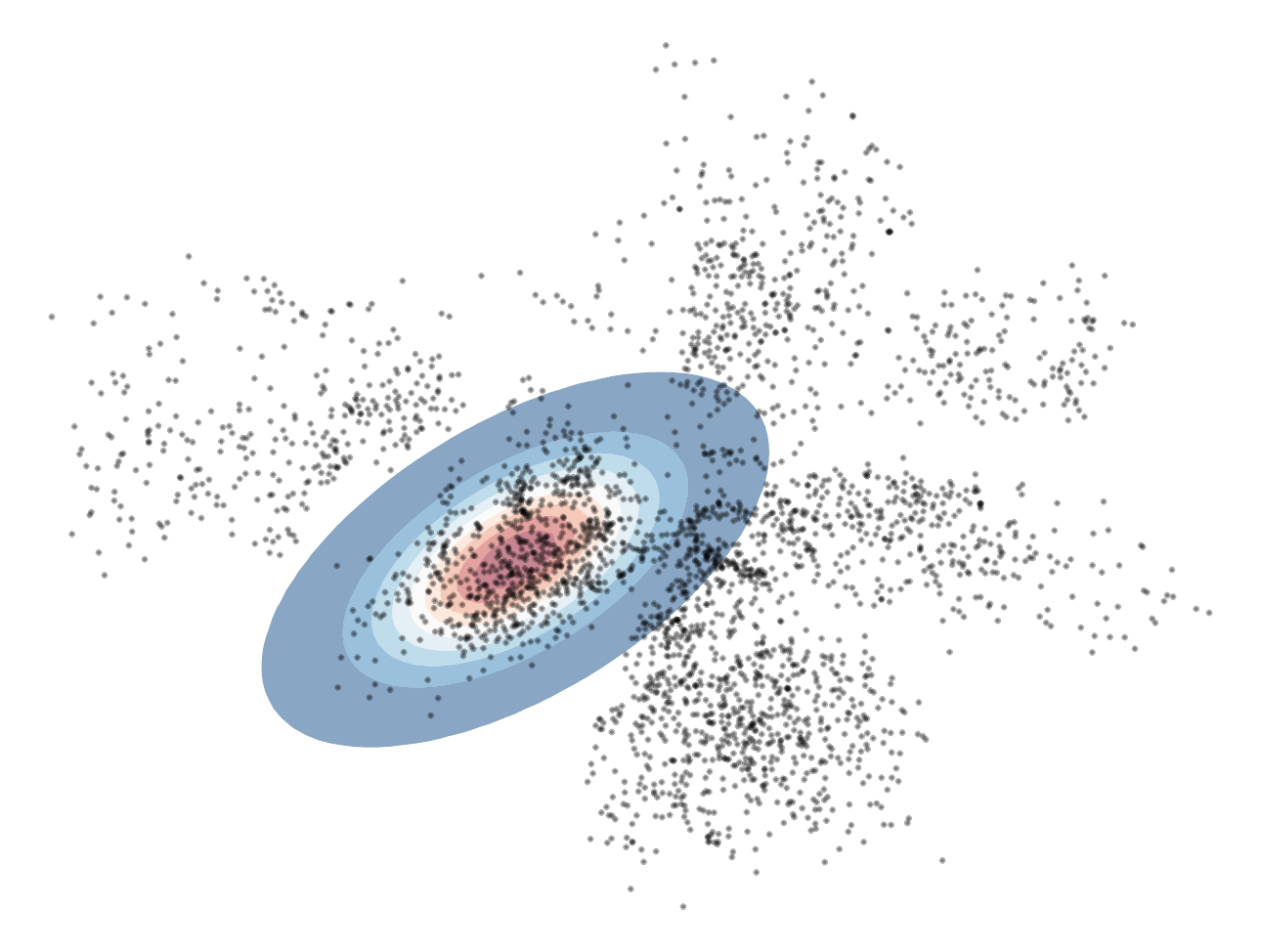}
	    \caption{Cluster 6}
	\end{subfigure}
    \caption{Cortex dataset. The individual components for the mixture of LANDs and the corresponding Gaussian mixture model. \emph{Top row}: The LAND component. \emph{Bottom row}: The corresponding GMM component. We see that the LAND components adapt to the training latent codes, uncovering the structure for each cluster.}
    \label{fig:app:cortex}
\end{figure*}

\begin{figure*}
\centering
    \begin{subfigure}{0.45\linewidth}
	    \includegraphics[width=1\linewidth]{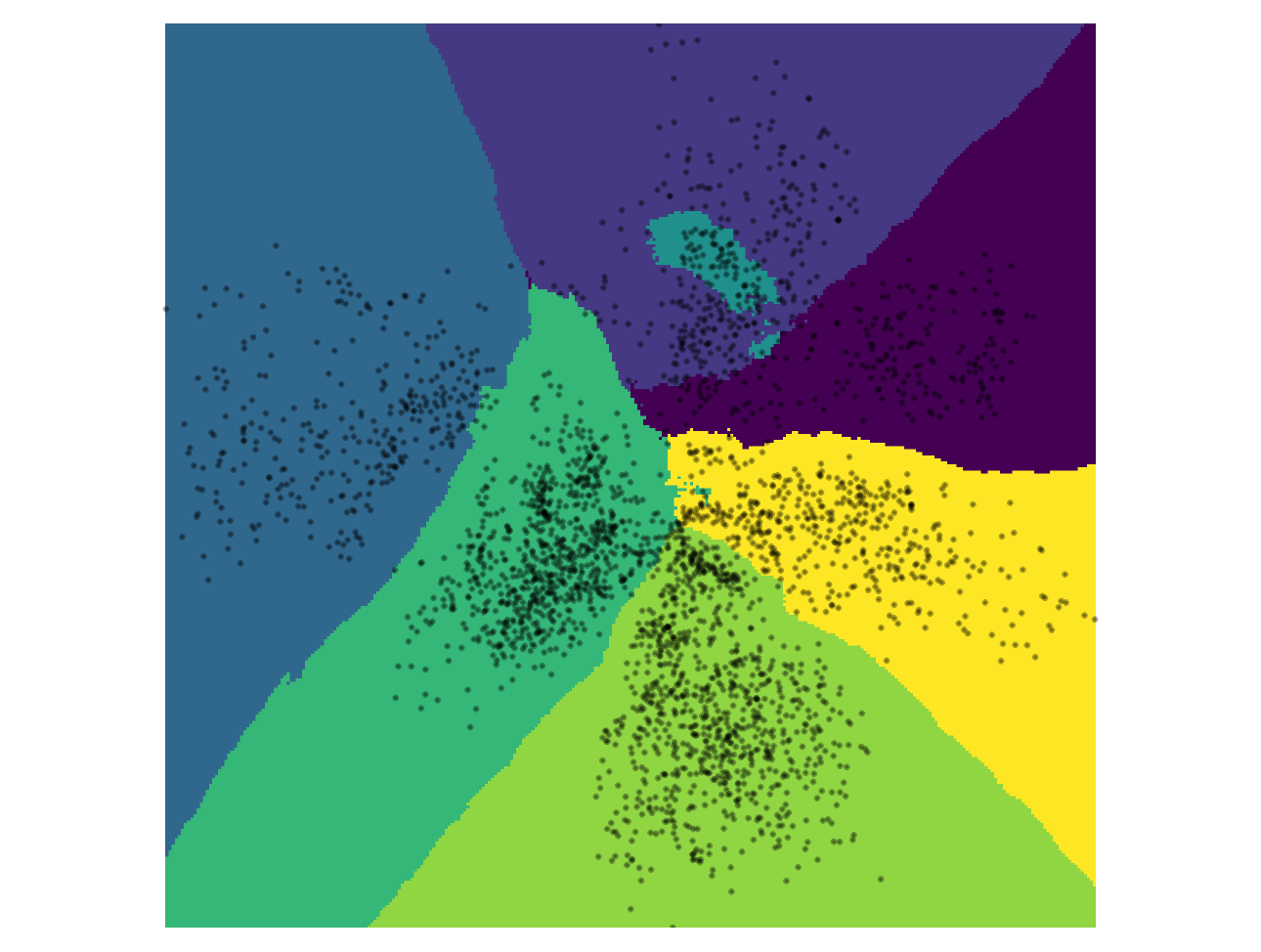}
	\end{subfigure}
	\begin{subfigure}{0.45\linewidth}
	    \includegraphics[width=1\linewidth]{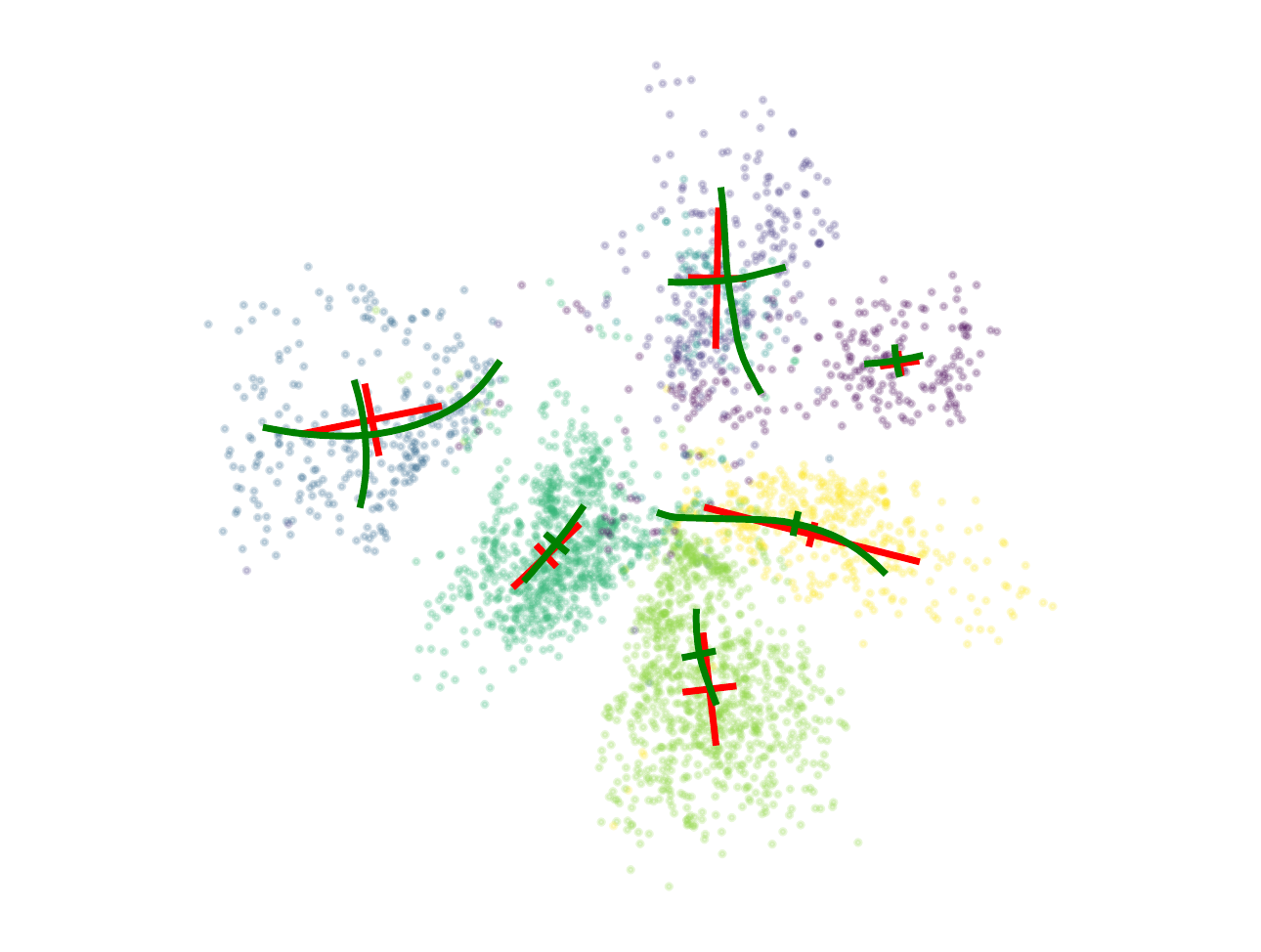}
	\end{subfigure}
	\caption{Cortex dataset. \emph{Left}: The true clusters of the latent codes in $\Z$. We used a $k$-NN classifier to approximate the true clusters with $k=21$. \emph{Right}: Comparing the principal geodesics computed using the mixture of LANDs (\emph{green}) with the linear eigenvectors computed using the GMM (\emph{red}) for each component. The principal geodesic can be seen as a form of \emph{local geometric disentaglement}, since these paths correspond to the highest variance on the data manifold.}
	\label{app:fig:kmeans_principal_geodesics}
\end{figure*}

\begin{figure*}
    \centering
    \begin{subfigure}{0.32\linewidth}
        \includegraphics[width=1\linewidth]{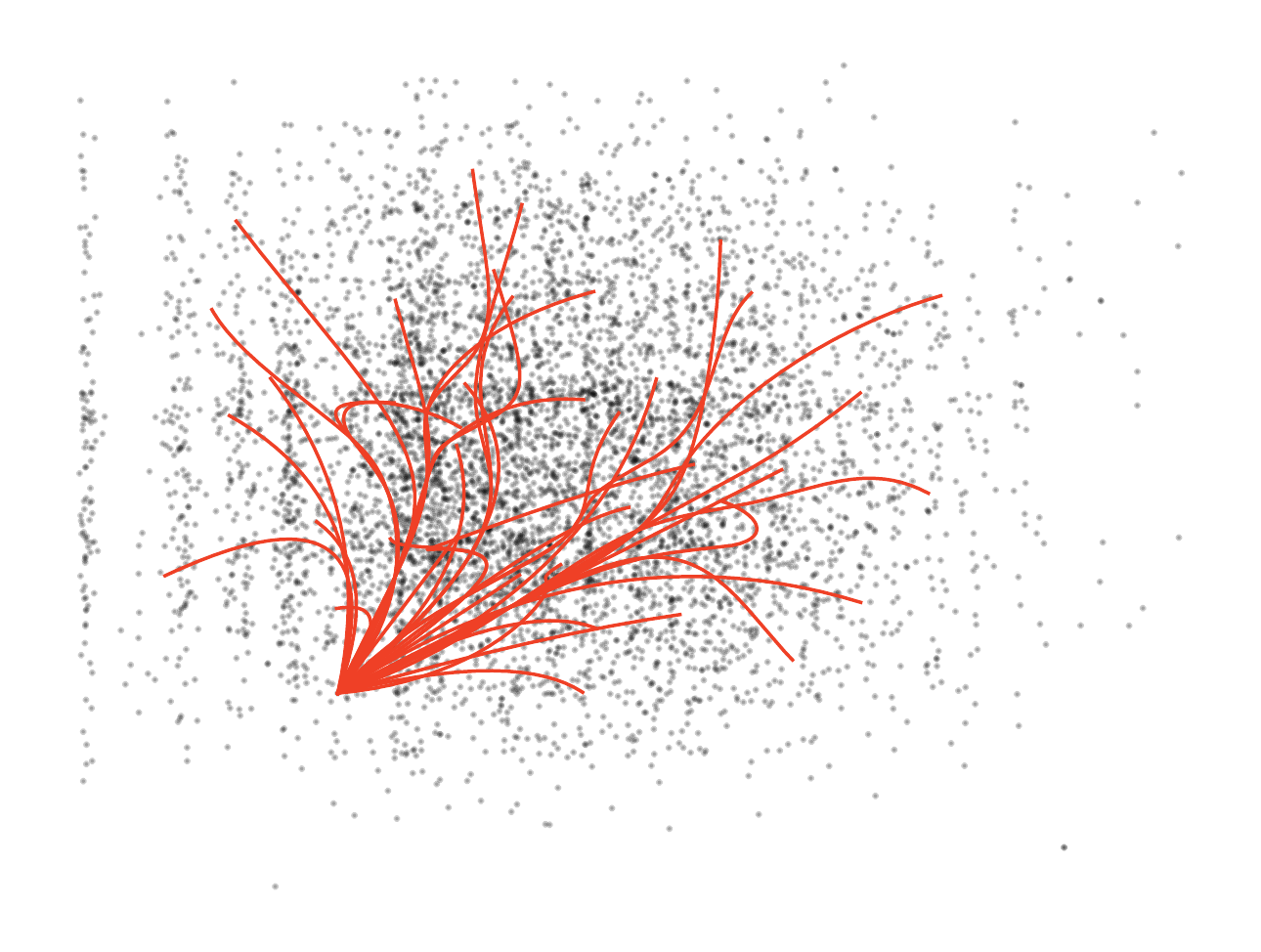}
        \caption{View $xy$-axis}
    \end{subfigure}
    ~
    \begin{subfigure}{0.32\linewidth}
        \includegraphics[width=1\linewidth]{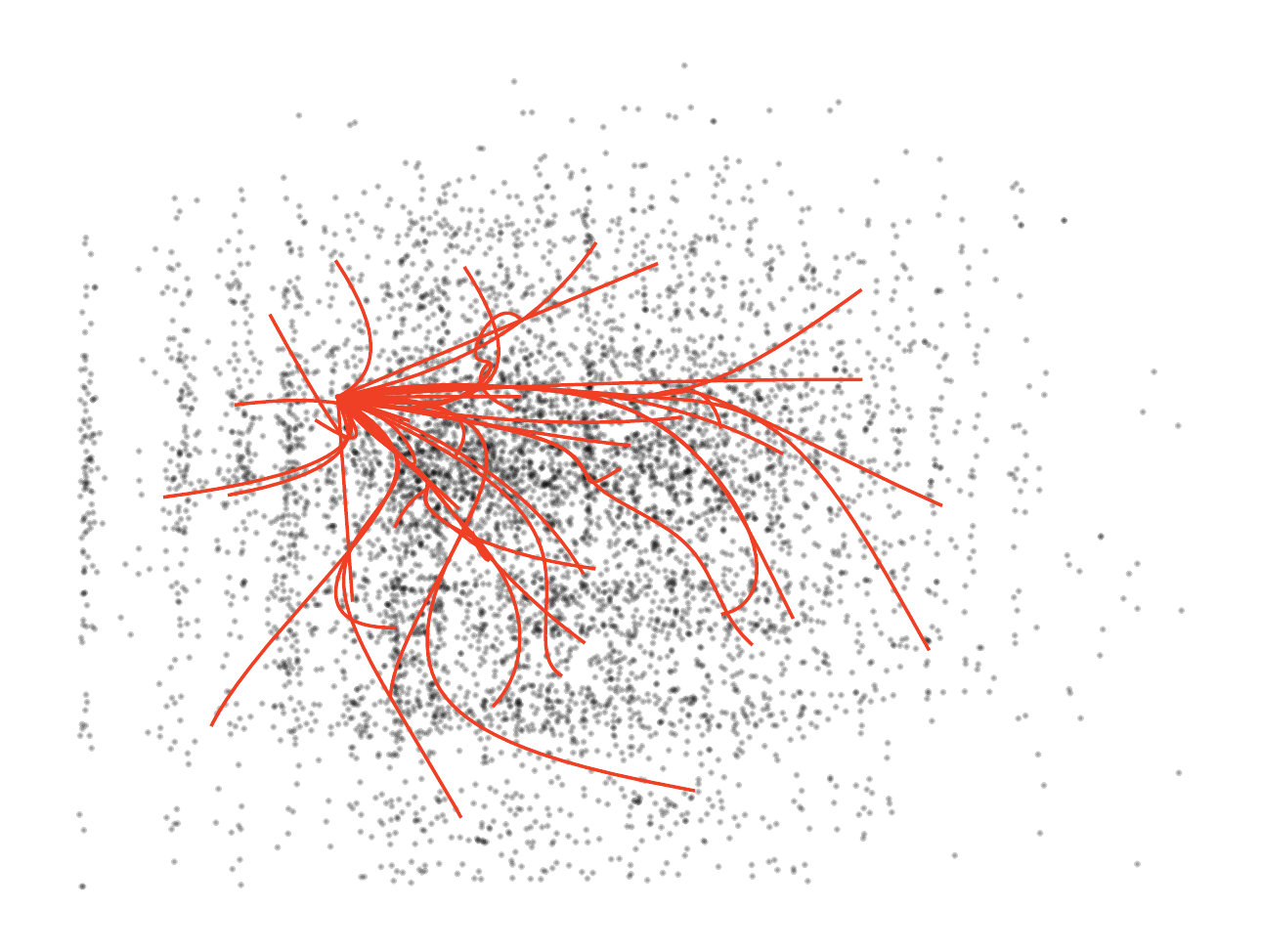}
        \caption{View $xz$-axis}
    \end{subfigure}
    ~
    \begin{subfigure}{0.32\linewidth}
        \includegraphics[width=1\linewidth]{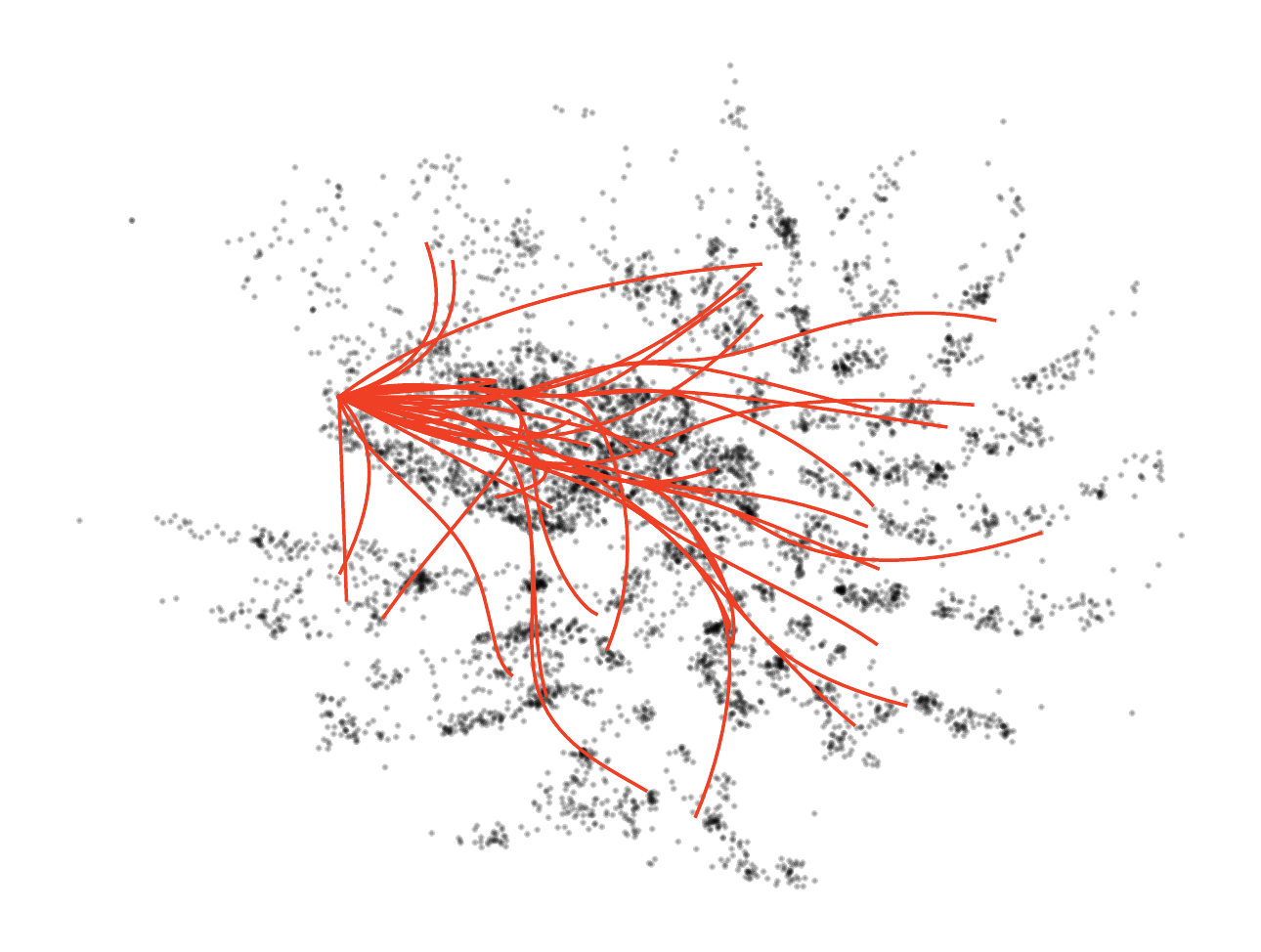}
        \caption{View $yz$-axis}
    \end{subfigure}
    \caption{Here we show the latent space of the molecule experiment from different views together with the corresponding geodesics.}
    \label{app:fig:chem_paths_planes}
\end{figure*}


\end{document}